\documentclass{new_tlp}

\let\proof\relax

\usepackage{microtype}
\usepackage{amsmath}
\usepackage{amssymb}
\usepackage{amsthm}
\usepackage[vlined,linesnumbered]{algorithm2e}
\usepackage{hyperref}
\usepackage[capitalise,noabbrev]{cleveref}
\usepackage{enumitem}
\usepackage{xcolor}
\usepackage{graphicx}
\usepackage{subcaption}

\crefname{equation}{}{}
\Crefname{equation}{}{}

\makeatletter
\newcommand{\leqnomode}{\tagsleft@true}
\newcommand{\reqnomode}{\tagsleft@false}
\makeatother

\newlist{alphaenum}{enumerate}{1}
\setlist[alphaenum]{label=\upshape(\alph*)}
\crefname{alphaenumi}{Property}{Properties}
\Crefname{alphaenumi}{Property}{Properties}

\definecolor{darkgreen}{rgb}{0.00,0.40,0.13}
\newcommand\DG{\color{darkgreen}}

\newtheoremstyle{empty}{}{}{}{}{\itshape}{.}{ }{\thmnote{{\itshape#3\/}}}

\newtheoremstyle{property}{}{}{}{}{\itshape}{.}{ }{\thmname{#1} \thmnote{{\itshape#3\/}}}

\theoremstyle{plain}
\newtheorem{theorem}{Theorem}
\newtheorem{proposition}[theorem]{Proposition}
\newtheorem{lemma}[theorem]{Lemma}
\newtheorem{corollary}[theorem]{Corollary}
\newtheorem{observation}[theorem]{Observation}

\theoremstyle{definition}
\newtheorem{definition}{Definition}
\newtheorem*{definition*}{Definition}
\newtheorem{example}{Example}

\theoremstyle{remark}
\newtheorem{remark}{Remark}

\theoremstyle{empty}
\newtheorem*{proofpart}{Part}

\theoremstyle{property}
\newtheorem*{proofcase}{Case}
\newtheorem*{proofprop}{Property}

\SetKwProg{Fn}{\KwSty{function}}{}{end}
\SetKwFunction{KwGroundProgram}{Ground}
\SetKwFunction{KwGroundRule}{GroundRule}
\SetKwFunction{KwGroundComponent}{GroundComponent}
\SetKwFunction{KwGroundComponentShort}{GC}
\SetKwFunction{KwPropagate}{Propagate}
\SetKwFunction{KwAssemble}{Assemble}
\SetKwFunction{KwRewrite}{Rewrite}
\SetKwFunction{KwSelect}{Select}
\SetKwFunction{KwMatch}{match}
\SetKw{Let}{let}
\SetKw{And}{and}
\SetKw{Or}{or}
\SetKw{True}{t}
\SetKw{False}{f}
\SetKwFunction{KwMatches}{Matches}
\newcommand{\Hb}{\ensuremath{\mathcal{A}}}
\newcommand{\Matches}[3]{\ensuremath{\KwMatches^{#2}_{#1}(#3)}}
\newcommand{\Select}[2]{\ensuremath{\KwSelect_{#1}(#2)}}
\newcommand{\GroundRule}[7]{\ensuremath{\KwGroundRule^{#2, #3}_{#1, #7, #4}(#5, #6)}}
\newcommand{\GroundComponent}[3]{\ensuremath{\KwGroundComponent(#1, #2, #3)}}
 \newcommand{\GroundComponentAbbrev}[3]{\ensuremath{\KwGroundComponentShort(#1,#2,#3)}}
\newcommand{\GroundProgram}[1]{\ensuremath{\KwGroundProgram(#1)}}

\newcommand{\Assemble}[2]{\ensuremath{\KwAssemble(#1,#2)}}
\newcommand{\Propagate}[5]{\ensuremath{\KwPropagate_{#1}^{#4,#5}(#2,#3)}}
\newcommand{\Naf}{{\neg}}
\newcommand\Match[2]{\ensuremath{\KwMatch(#1,#2)}}
\newcommand{\Signature}{\ensuremath{\Sigma}}
\newcommand{\SignatureF}{\ensuremath{\mathcal{F}}}
\newcommand{\SignatureP}{\ensuremath{\mathcal{P}}}
\newcommand{\SignatureV}{\ensuremath{\mathcal{V}}}
\newcommand{\nref}[2]{\cref{#1}~\ref{#1:#2}}
\newcommand{\Nref}[2]{\Cref{#1}~\ref{#1:#2}}
\newcommand\Index{\mathbb{I}}
\newcommand\IndexAlt{\mathbb{J}}
\newcommand\Pred[1]{\ensuremath{\mathit{#1}}}
\newcommand\PredE{\ensuremath{\mathcal{E}}}

\newcommand\Reduct[2]{\ensuremath{#1^{#2}}}
\newcommand\ReductC[2]{\Reduct{{(#1)}}{#2}}
\newcommand\IDReductP[2]{\ensuremath{#1_{#2}}}
\newcommand\IDReductPC[2]{\IDReductP{{(#1)}}{#2}}
\newcommand\IDReductN[2]{\IDReductP{#1}{\overline{#2}}}
\newcommand\IDReductNC[2]{\IDReductN{{(#1)}}{#2}}
\newcommand\IDEvalP[2]{\ensuremath{\mathit{pe}_{#2}({#1})}}
\newcommand\IDEvalPC[2]{\IDEvalP{#1}{#2}}
\newcommand\IDEvalN[2]{\IDEvalP{#1}{\overline{#2}}}
\newcommand\IDEvalNC[2]{\IDEvalN{#1}{#2}}
\newcommand\Head[1]{\ensuremath{H(#1)}}
\newcommand\Body[1]{\ensuremath{B(#1)}}
\newcommand\Tuple[1]{\ensuremath{H(#1)}}
\newcommand\Cond[1]{\ensuremath{B(#1)}}
\newcommand\WeightSym{w}
\newcommand\Weight[1]{\ensuremath{\WeightSym(#1)}}
\newcommand\WeightP[1]{\ensuremath{\WeightSym^+(#1)}}
\newcommand\WeightM[1]{\ensuremath{\WeightSym^-(#1)}}
\newcommand\IDPos[1]{\ensuremath{{#1}^+}}
\newcommand\IDNeg[1]{\ensuremath{{#1}^-}}
\newcommand\IDPosNeg[1]{\ensuremath{{#1}^\pm}}
\newcommand\Predicate[1]{\ensuremath{\mathrm{pred}(#1)}}

\newcommand\ImpliesT{\text{\ implies\ }}
\newcommand\AndT{\text{\ and\ }}
\newcommand\CandT{\text{, and\ }}
\newcommand\CT{\text{,\ }}
\newcommand\ForT{\text{\ for\ }}
\newcommand\ForAllT{\text{\ for all\ }}
\newcommand\ForSomeT{\text{\ for some\ }}
\newcommand\ByT[1]{\text{\ by\ #1}}
\newcommand\BecauseT[1]{\text{\ because\ #1}}

\newcommand\ImpliesAndT{\rlap\AndT\phantom\ImpliesT}
\newcommand\StepSym{\ensuremath{T}}
\newcommand\StepO[1]{\ensuremath{\StepSym_{#1}}}
\newcommand\Step[2]{\ensuremath{\StepO{#1}(#2)}}
\newcommand\StepRO[2]{\ensuremath{\StepSym_{#1}^{#2}}}
\newcommand\StepR[3]{\ensuremath{\StepRO{#1}{#2}(#3)}}
\newcommand\StableSym{\ensuremath{S}}
\newcommand\StableO[1]{\ensuremath{\StableSym_{#1}}}
\newcommand\Stable[2]{\ensuremath{\StableO{#1}(#2)}}
\newcommand\StableRO[2]{\ensuremath{\StableSym_{#1}^{#2}}}
\newcommand\StableR[3]{\ensuremath{\StableRO{#1}{#2}(#3)}}
\newcommand\LeastM[1]{\ensuremath{\mathit{LM}(#1)}}
\newcommand\FormulasF{\ensuremath{\mathcal{F}}}
\newcommand\FormulasH{\ensuremath{\mathcal{H}}}

\newcommand\FormulasN{\ensuremath{\mathcal{N}}}
\newcommand\FormulasR{\ensuremath{\mathcal{R}}}
\newcommand\FOID{\textup{\textsc{id}}}
\newcommand\WellSym{\ensuremath{W}}
\newcommand\WellO[1]{\ensuremath{\WellSym_{#1}}}
\newcommand\Well[2]{\ensuremath{\WellO{#1}(#2)}}
\newcommand\WellM[1]{\ensuremath{\mathit{WM}(#1)}}
\newcommand\WellRO[2]{\ensuremath{\WellSym_{#1}^{#2}}}
\newcommand\WellR[3]{\ensuremath{\WellRO{#1}{#2}(#3)}}
\newcommand\WellRM[2]{\ensuremath{\mathit{WM}^{#2}(#1)}}
\newcommand\ApproxM[1]{\ensuremath{\mathit{AM}(#1)}}
\newcommand\ApproxRM[3]{\ensuremath{\mathit{AM}^{#3}_{#2}(#1)}}
\newcommand\Simp[2]{\ensuremath{#1^{#2}}}
\newcommand\Restrict[2]{\ensuremath{#1|_{#2}}}
\newcommand\TransSym{\ensuremath{\tau}}
\newcommand\TransASym{\ensuremath{\pi}} \newcommand\Trans[1]{\ensuremath{\TransSym(#1)}}
\newcommand\TransD[2]{\ensuremath{\TransSym_{#1}(#2)}}
\newcommand\TransA[1]{\ensuremath{\TransASym(#1)}}
\newcommand\TransAD[2]{\ensuremath{{\TransASym_{#1}}(#2)}}
\newcommand\GrdSym{\ensuremath{\mathrm{Inst}}} \newcommand\Grd[1]{\ensuremath{\GrdSym(#1)}}
\newcommand\GrdSimp[2]{\ensuremath{\Simp{\GrdSym}{#2}(#1)}}
\newcommand\ElemF{E}
\newcommand\ElemG{G}

\newcommand\ElemJ{D}
\newcommand\ElemC{C}
\newcommand{\justifies}{\triangleright} \newcommand{\notjustifies}{\mathrel{\not\kern-1pt\triangleright}}
\newcommand\TransRSym[1]{\ensuremath{\TransASym_{#1}}}
\newcommand\TransR[2]{\ensuremath{\TransRSym{#2}(#1)}}
\newcommand\Count{\mathrm{\#count}}
\newcommand\Sum{\mathrm{\#sum}}
\newcommand\SumP{\ensuremath{\mathrm{\#sum}^+}}
\newcommand\SumM{\ensuremath{\mathrm{\#sum}^-}}

\newcommand\Aggregate[4]{\ensuremath{#1 \{ #2 \} #3 #4}}
\newcommand\Elem[2]{\ensuremath{#1 : #2}}
\newcommand\CountA[3]{\Aggregate{\Count}{#1}{#2}{#3}}
\newcommand\SumA[3]{\Aggregate{\Sum}{#1}{#2}{#3}}
\newcommand\SumPA[3]{\Aggregate{\SumP}{#1}{#2}{#3}}
\newcommand\SumMA[3]{\Aggregate{\SumM}{#1}{#2}{#3}}

\newcommand\AggregateAtomP[3]{\ensuremath{\alpha_{#1,#2}/{#3}}}
\newcommand\AggregateAtom[3]{\ensuremath{\alpha_{#1,#2}(#3)}}
\newcommand\EmptyAtomP[3]{\ensuremath{\epsilon_{#1,#2}/{#3}}}
\newcommand\EmptyAtom[3]{\ensuremath{\epsilon_{#1,#2}(#3)}}
\newcommand\ElementAtomP[4]{\ensuremath{\eta_{#1,#2,#3}/{#4}}}
\newcommand\ElementAtom[4]{\ensuremath{\eta_{#1,#2,#3}(#4)}}
\newcommand\AggrEmpty[4]{\ensuremath{\epsilon_{#1,#2}(#3,#4)}}
\newcommand\AggrElem[4]{\ensuremath{\eta_{#1,#2}(#3,#4)}}
\newcommand\Relative{\ensuremath{C}}
\newcommand\External{\ensuremath{E}}
\newcommand\Botsym{\ensuremath{B}}
\newcommand\Topsym{\ensuremath{T}}
\newcommand\Ir[1]{\ensuremath{\mathit{#1\!\Relative}}}
\newcommand\Ia[1]{\ensuremath{\mathit{#1\!A}}}
\newcommand\Iai[2]{\ensuremath{\Ia{#1}_{#2}}}
\newcommand\Ib[1]{\ensuremath{\mathit{#1\!\Botsym}}}
\newcommand\Ie[1]{\ensuremath{\mathit{#1\!\External}}}
\newcommand\Iei[2]{\ensuremath{\Ie{#1}_{#2}}}
\newcommand\Ii[2]{\ensuremath{#1_{#2}}}
\newcommand\Ipi[2]{\ensuremath{{#1}_{#2}'}}
\newcommand\Ipr[1]{\ensuremath{\Ir{#1}'}}
\newcommand\Ipa[1]{\ensuremath{\Ia{#1}'}}
\newcommand\Iri[2]{\ensuremath{\Ir{#1}_{#2}}}
\newcommand\Irpi[2]{\ensuremath{\Ir{#1}_{#2}'}}
\newcommand\It[1]{\ensuremath{\mathit{#1\!\Topsym}}}
\newcommand\Iwhb[1]{\ensuremath{\widehat{\Ib{#1}}}}

\newcommand\Iwht[1]{\ensuremath{\widehat{\It{#1}}}}
\newcommand\Iwt[1]{\ensuremath{\widetilde{#1}}}
\newcommand\Iwtb[1]{\ensuremath{\widetilde{\Ib{#1}}}}
\newcommand\Iwte[1]{\ensuremath{\widetilde{\Ie{#1}}}}
\newcommand\Iwtt[1]{\ensuremath{\widetilde{\It{#1}}}}
\newcommand\Fa[1]{\ensuremath{{#1}^{\alpha}}}
\newcommand\Fai[2]{\ensuremath{{#1}^{\alpha}_{#2}}}
\newcommand\Fb[1]{\ensuremath{\mathit{#1\!\Botsym}}}
\newcommand\Fe[1]{\ensuremath{{#1}^{\eta}}}
\newcommand\Fei[2]{\ensuremath{{#1}^{\eta}_{#2}}}

\newcommand\Fi[2]{\ensuremath{{#1}_{#2}}}
\newcommand\Fpi[2]{\ensuremath{{#1}_{#2}'}}
\newcommand\Fr[1]{\ensuremath{\mathit{#1\!\Relative}}}

\newcommand\Fri[2]{\ensuremath{\Fr{#1}_{#2}}}

\newcommand\Ft[1]{\ensuremath{\mathit{#1\!\Topsym}}}
\newcommand\Fv[1]{\ensuremath{{#1}^{\epsilon}}}
\newcommand\Fvi[2]{\ensuremath{{#1}^{\epsilon}_{#2}}}

\newcommand\Fwhb[1]{\ensuremath{\widehat{\Fb{#1}}}}
\newcommand\Fwht[1]{\ensuremath{\widehat{\Ft{#1}}}}
\newcommand\Fwtb[1]{\ensuremath{\widetilde{\Fb{#1}}}}
\newcommand\Fwtt[1]{\ensuremath{\widetilde{\Ft{#1}}}}
\newcommand\Fsq[2]{\ensuremath{{(#1)}_{#2}}}
\newcommand\Fsqi[3]{\ensuremath{\Fsq{\Fi{#1}{#3}}{#3 \in #2}}}
\newcommand\Fsqic[3]{\ensuremath{\Fsq{\Fi{#1}{#3}}{(#3) \in #2}}}
\newcommand\Fsqti[3]{\ensuremath{\Fsq{\Trans{\Fi{#1}{#3}}}{#3 \in #2}}}
\newcommand\Fsqtic[3]{\ensuremath{\Fsq{\Trans{\Fi{#1}{#3}}}{(#3) \in #2}}}
\newcommand\OpSym{\ensuremath{O}}
\newcommand\OppSym{\ensuremath{O'}}
\newcommand\OpApp[1]{\ensuremath{\OpSym(#1)}}
\newcommand\OppApp[1]{\ensuremath{\OppSym(#1)}}

 \newcommand{\sysfont}{\textit}

\newcommand{\Gringo}{\sysfont{Gringo}}

\newcommand{\assat}{\sysfont{assat}}

\newcommand{\clasp}{\sysfont{clasp}}

\newcommand{\clingo}{\sysfont{clingo}}
\newcommand{\cmodels}{\sysfont{cmodels}}

\newcommand{\dlv}{\sysfont{dlv}}

\newcommand{\gringo}{\sysfont{gringo}}

\newcommand{\idlv}{\sysfont{idlv}}
\newcommand{\idp}{\sysfont{idp}}

\newcommand{\lparse}{\sysfont{lparse}}

\newcommand{\smodels}{\sysfont{smodels}}

\newcommand{\wasp}{\sysfont{wasp}}

\title[Foundations of Grounding in ASP]{On the Foundations of Grounding in\\Answer Set Programming}

\author[R. Kaminski and T. Schaub]{ROLAND KAMINSKI and TORSTEN SCHAUB\\
University of Potsdam\\
\email{\{kaminski,torsten\}@cs.uni-potsdam.de}}

\makeatletter
\renewenvironment{keywords}
  {\noindent\normalfont\small\rmfamily{\em \keywordsname}:}{\vspace{6.5\p@ \@plus 3\p@ \@minus 1\p@}\endtrivlist
   \vbox{\hrule \@width \hsize}}
\newcounter{Diff}
\newcounter{Start}
\def\ps@appendixheadings{\let\@mkboth\@gobbletwo
  \def\@oddhead{\setcounter{Diff}{\numexpr\value{page}-\value{Start}+2}\hfil{\itshape\@shorttitle}\hfil \llap{A-\theDiff}}\def\@evenhead{\setcounter{Diff}{\numexpr\value{page}-\value{Start}+2}\rlap{A-\theDiff}\hfil\itshape\@shortauthor\hfil}\let\@oddfoot\@empty
  \let\@evenfoot\@oddfoot
  \def\sectionmark##1{\markboth{##1}{}}\def\subsectionmark##1{\markright{##1}}}
\makeatother

\begin{document}

\maketitle

\begin{abstract}
We provide a comprehensive elaboration of the theoretical foundations of variable instantiation,
or grounding, in Answer Set Programming (ASP).
Building on the semantics of ASP's modeling language,
we introduce a formal characterization of grounding algorithms
in terms of (fixed point) operators.
A major role is played by dedicated well-founded operators whose associated models provide semantic guidance
for delineating the result of grounding along with on-the-fly simplifications.
We address an expressive class of logic programs that incorporates recursive aggregates and
thus amounts to the scope of existing ASP modeling languages.
This is accompanied with a plain algorithmic framework detailing the grounding of recursive aggregates.
The given algorithms correspond essentially to the ones used in the ASP grounder \gringo.
\end{abstract}


\section{Introduction}\label{sec:introduction}

Answer Set Programming (ASP; \citeNP{lifschitz02a}) allows us to address knowledge-intense search and optimization problems in a greatly declarative way
due to its integrated modeling, grounding, and solving workflow~\cite{gebsch16a,kalepesc16a}.
Problems are modeled in a rule-based logical language featuring variables, function symbols, recursion, and aggregates, among others.
Moreover, the underlying semantics allows us to express defaults and reachability in an easy way.
A corresponding logic program is then turned into a propositional format by systematically replacing all variables by variable-free terms.
This process is called \emph{grounding}.
Finally, the actual ASP solver takes the resulting propositional version of the original program and computes its answer sets.

Given that both grounding and solving constitute the computational cornerstones of ASP,
it is surprising that the importance of grounding has somehow been eclipsed by that of solving.
This is nicely reflected by the unbalanced number of implementations.
With \lparse~\cite{syrjanen01a}, (the grounder in) \dlv~\cite{falepe12a}, and \gringo~\cite{gekakosc11a},
three grounder implementations face dozens of solver implementations, among them
\smodels~\cite{siniso02a},
(the solver in) \dlv~\cite{dlv03a},
\assat~\cite{linzha04a},
\cmodels~\cite{gilima06a},
\clasp~\cite{gekasc09c},
\wasp~\cite{aldoleri15a} just to name the major ones.
What caused this imbalance?
One reason may consist in the high expressiveness of ASP's modeling language and the resulting algorithmic intricacy~\cite{gekakosc11a}.
Another may lie in the popular viewpoint that grounding amounts to database materialization,
and thus that most fundamental research questions have been settled.
And finally the semantic foundations of full-featured ASP modeling languages have been established only recently~\cite{haliya14a,gehakalisc15a},
revealing the semantic gap to the just mentioned idealized understanding of grounding.
In view of this, research on grounding focused on algorithm and system design~\cite{falepe12a,gekakosc11a} and the characterization of language fragments
guaranteeing finite propositional representations~\cite{syrjanen01a,gescth07a,lielif09a,cacoiale08a}.

As a consequence, the theoretical foundations of grounding are much less explored than those of solving.
While there are several alternative ways to characterize the answer sets of a logic program~\cite{lifschitz08a}, and thus the behavior of a solver,
we still lack indepth formal characterizations of the input-output behavior of ASP grounders.
Although we can describe the resulting propositional program up to semantic equivalence,
we have no formal means to delineate the actual set of rules.

To this end, grounding involves some challenging intricacies.
First of all,
the entire set of systematically instantiated rules is infinite in the worst --- yet not uncommon --- case.
For a simple example, consider the program:
\begin{align*}
p(a) & \\
p(X) & \leftarrow p(f(X))
\end{align*}
This program induces an infinite set of variable-free terms, viz.\ $a$, $f(a)$, $f(f(a)),\dots$, that leads to an infinite propositional program by
systematically replacing variable $X$ by all these terms in the second rule, viz.\
\begin{align*}
  p(a),\quad
  p(a) \leftarrow p(f(a)),\quad
  p(f(a)) \leftarrow p(f(f(a))),\quad
  p(f(f(a))) \leftarrow p(f(f(f(a)))),\quad
  \dots
\end{align*}
On the other hand, modern grounders only produce the fact \(p(a)\) and no instances of the second rule,
which is semantically equivalent to the infinite program.
As well,
ASP's modeling language comprises (possibly recursive) aggregates,
whose systematic grounding may be infinite in itself.
To illustrate this, let us extend the above program with the rule
\begin{align}
q & \leftarrow \#\mathit{count}\{X\mathrel{:}p(X)\}=1\label{eq:count}
\end{align}
deriving \(q\) when the number of satisfied instances of \(p\) is one.
Analogous to above, the systematic instantiation of the aggregate's element results in an infinite set, viz.\
\begin{align*}
  \{
  a\mathrel{:}p(a),\
  f(a)\mathrel{:}p(f(a)),\
  f(f(a))\mathrel{:}p(f(f(a))),\
  \dots
  \}.
\end{align*}
Again, a grounder is able to reduce the rule in \cref{eq:count} to the fact \(q\) since only \(p(a)\) is obtained in our example.
That is, it detects that the set amounts to the singleton \(\{a\mathrel{:}p(a)\}\), which satisfies the aggregate.
After removing the rule's (satisfied) antecedent, it produces the fact \(q\).
In fact, a solver expects a finite set of propositional rules including aggregates over finitely many objects only.
Hence, in practice, the characterization of the grounding result boils down to identifying a finite yet semantically equivalent set of rules
(whenever possible).
Finally, in practice, grounding involves simplifications whose application depends on the ordering of rules in the input.
In fact,
shuffling a list of propositional rules only affects the order in which a solver enumerates answer sets,
whereas shuffling a logic program before grounding may lead to different though semantically equivalent sets of rules.
To see this, consider the program:
\begin{align*}
p(X) & \leftarrow \neg q(X) \wedge u(X) & u(1) \quad & u(2)            \\
q(X) & \leftarrow \neg p(X) \wedge v(X) &            & v(2) \quad v(3)
\end{align*}
This program has two answer sets; both contain $p(1)$ and $q(3)$,
while one contains $q(2)$ and the other $p(2)$.
Systematically grounding the program yields the obvious four rules.
However, depending upon the order, in which the rules are passed to a grounder,
it already produces either the fact $p(1)$ or $q(3)$ via simplification.
Clearly, all three programs are distinct but semantically equivalent in sharing the above two answer sets.

Our elaboration of the foundations of ASP grounding rests upon
the semantics of ASP's modeling language~\cite{haliya14a,gehakalisc15a},
which captures the two aforementioned sources of
infinity by associating non-ground logic programs with infinitary propositional formulas~\cite{truszczynski12a}.
Our main result shows that the stable models of a non-ground input program coincide with
the ones of the ground output program returned by our grounding algorithm upon termination.
In formal terms, this means that the stable models of the infinitary formula associated with the input program coincide
with the ones of the resulting ground program.
Clearly, the resulting program must be finite and consist of finitary subformulas only.
A major part of our work is thus dedicated to equivalence preserving transformations between ground programs.
In more detail,
we introduce a formal characterization of grounding algorithms in terms of (fixed point) operators.
A major role is played by specific well-founded operators whose associated models provide semantic guidance
for delineating the result of grounding. More precisely, we show how to obtain a finitary propositional formula capturing a logic program whenever the
corresponding well-founded model is finite,
and notably how this transfers to building a finite propositional program from an input program during grounding.
The two key instruments accomplishing this are dedicated forms of \emph{program simplification} and \emph{aggregate translation},
each addressing one of the two sources of infinity in the above example.
In practice, however, all these concepts are subject to \emph{approximation}, which leads to the order-dependence observed in the last example.

We address an expressive class of logic programs that incorporates recursive aggregates and
thus amounts to the scope of existing ASP modeling languages~\cite{PotasscoUserGuide}.
This is accompanied with an algorithmic framework detailing the grounding of recursive aggregates.
The given grounding algorithms correspond essentially to the ones used in the ASP grounder \gringo~\cite{gekakosc11a}.
In this way, our framework provides a formal characterization of one of the most widespread grounding systems.
In fact,
modern grounders like (the one in) \dlv~\cite{falepe12a} or \gringo~\cite{gekakosc11a} are based on database evaluation techniques~\cite{ullman88a,abhuvi95a}.
The instantiation of a program is seen as an iterative bottom-up process starting from the program's facts
while being guided by the accumulation of variable-free atoms possibly derivable from the rules seen so far.
During this process,
a ground rule is produced if its positive body atoms belong to the accumulated atoms,
in which case its head atom is added as well.
This process is repeated until no further such atoms can be added.
From an algorithmic perspective,
we show how a grounding framework (relying upon database evaluation techniques) can be extended to incorporate recursive aggregates.

Our paper is organized as follows.

\Cref{sec:background} lays the basic foundations of our approach.
We start in \cref{sec:operators} by recalling definitions of (monotonic) operators on lattices;
they constitute the basic building blocks of our characterization of grounding algorithms.
We then review infinitary formulas along with their stable and well-founded semantics
in \cref{sec:formulas,,sec:stable,,sec:well-founded}, respectively.
In this context, we explore several operators and define a class of infinitary logic programs that
allows us to capture full-featured ASP languages with (recursive) aggregates.
Interestingly, we have to resort to concepts borrowed from \FOID-logic~\cite{brdetr16a,truszczynski12a}
to obtain monotonic operators that are indispensable for capturing iterative algorithms.
Notably, the \FOID-well-founded model can be used for approximating regular stable models.
Finally, we define in \cref{sec:simplification} our concept of program simplification and elaborate upon its semantic properties.
The importance of program simplification can be read off two salient properties.
First, it results in a finite program whenever the interpretation used for simplification is finite.
And second, it preserves all stable models when simplified with the \FOID-well-founded model of the program.

\Cref{sec:splitting} is dedicated to the formal foundations of component-wise grounding.
As mentioned, each rule is instantiated in the context of all atoms being possibly derivable up to this point.
In addition, grounding has to take subsequent atom definitions into account.
To this end, we extend well-known operators and resulting semantic concepts with contextual information,
usually captured by two- and four-valued interpretations, respectively,
and elaborate upon their formal properties that are relevant to grounding.
In turn,
we generalize the contextual operators and semantic concepts to sequences of programs in order to
reflect component-wise grounding.
The major emerging concept is essentially a well-founded model for program sequences
that takes backward and forward contextual information into account.
We can then iteratively compute this model to approximate the well-founded model of the entire program.
This model-theoretic concept can be used for governing an ideal grounding process.

\Cref{sec:aggregates} turns to logic programs with variables and aggregates.
We align the semantics of such aggregate programs with the one of~\citeN{ferraris11a}
but consider infinitary formulas~\cite{haliya14a}.
In view of grounding aggregates, however,
we introduce an alternative translation of aggregates that is strongly equivalent to that of Ferraris
but provides more precise well-founded models.
In turn, we refine this translation to be bound by an interpretation so that it produces finitary formulas
whenever this interpretation is finite.
Together, the program simplification introduced in \cref{sec:simplification} and aggregate translation provide the basis for
turning programs with aggregates into semantically equivalent finite programs with finitary subformulas.

\Cref{sec:dependency} further refines our semantic approach to reflect actual grounding processes.
To this end, we define the concept of an instantiation sequence based on rule dependencies.
We then use the contextual operators of \cref{sec:splitting} to define approximate models of instantiation sequences.
While approximate models are in general less precise than well-founded ones, they are better suited for on-the-fly
grounding along an instantiation sequence.
Nonetheless, they are strong enough to allow for completely evaluating stratified programs.

\Cref{sec:algorithms} lays out the basic algorithms for grounding rules, components, and entire programs
and characterizes their output in terms of the semantic concepts developed in the previous sections.
Of particular interest is the treatment of aggregates, which are decomposed into dedicated normal rules before grounding,
and reassembled afterward.
This allows us to ground rules with aggregates by means of grounding algorithms for normal rules.
Finally,
we show that our grounding algorithm guarantees that
an obtained finite ground program is equivalent to the original non-ground program.

The previous sections focus on the theoretical and algorithmic cornerstones of grounding.
\Cref{sec:refinements} refines these concepts by further detailing aggregate propagation, algorithm specifics,
and the treatment of language constructs from \gringo's input language.

We relate our contributions to the state of the art in \cref{sec:related} and summarize it in \cref{sec:conclusion}.

Although the developed approach is implemented in \gringo\ series~4 and~5,
their high degree of sophistication make it hard to retrace the algorithms from \cref{sec:algorithms}.
Hence, to ease comprehensibility,
we have moreover implemented the presented approach in $\mu$-\gringo\footnote{The $\mu$-\gringo\ system is available at \url{https://github.com/potassco/mu-gringo}.\label{fn:mu-gringo}}
in a transparent way and equipped it with means for retracing the developed concepts during grounding.
This can thus be seen as the practical counterpart to the formal elaboration given below.
Also, this system may enable some readers to construct and to experiment with own grounder extensions.

This paper draws on material presented during an invited talk at the third workshop on grounding, transforming, and
modularizing theories with variables~\cite{gekasc15a}.

 \section{Foundations}\label{sec:background}

\subsection{Operators on lattices}\label{sec:operators}

This section recalls basic concepts on operators on complete lattices.

A \emph{complete lattice} is a partially ordered set \((L,\leq)\)
in which every subset \(S \subseteq L\) has a greatest lower bound
and a least upper bound in \((L,\leq)\).

An \emph{operator}~\(\OpSym\) on lattice~\((L,\leq)\) is a function from \(L\) to \(L\).
It is \emph{monotone} if \(x \leq y\) implies \(\OpApp{x} \leq \OpApp{y}\) for each \(x,y \in L\);
and it is \emph{antimonotone} if \(x \leq y\) implies \(\OpApp{y} \leq \OpApp{x}\) for each \(x,y \in L\).

Let \(\OpSym\) be an operator on lattice~\((L,\leq)\).
A \emph{prefixed point} of \(\OpSym\) is an \(x \in L\) such that \(\OpApp{x} \leq x\).
A \emph{postfixed point} of \(\OpSym\) is an \(x \in L\) such that \(x \leq \OpApp{x}\).
A \emph{fixed point of} \(\OpSym\) is an \(x \in L\) such that \(x = \OpApp{x}\), i.e., it is both a prefixed and a postfixed point.

\begin{theorem}[Knaster-Tarski; \protect\citeNP{tarski55}]\label{thm:knaster-tarski}
Let \(\OpSym\) be a monotone operator on complete lattice \((L,\leq)\).
Then, we have the following properties:
\begin{alphaenum}
\item Operator~\(\OpSym\) has a least fixed and prefixed point which are identical.\label{thm:knaster-tarski:a}
\item Operator~\(\OpSym\) has a greatest fixed and postfixed point which are identical.\label{thm:knaster-tarski:b}
\item The fixed points of~\(\OpSym\) form a complete lattice.\label{thm:knaster-tarski:c}
\end{alphaenum}
\end{theorem}

\subsection{Formulas and interpretations}\label{sec:formulas}

We begin with a propositional signature~\Signature\ consisting of a set of atoms.
Following~\citeN{truszczynski12a},
we define the sets \(\FormulasF_0,\FormulasF_1, \ldots\) of formulas as follows:
\begin{itemize}
\item \(\FormulasF_0\) is the set of all propositional atoms in~\(\Signature\),
\item
\(\FormulasF_{i+1}\) is the set of all elements of \(\FormulasF_i\),
all expressions \(\FormulasH^\wedge\) and \(\FormulasH^\vee\) with \(\FormulasH \subseteq \FormulasF_i\), and
all expressions \(F \rightarrow G\) with \(F,G\in\FormulasF_i\).
\end{itemize}
The set
\(
\FormulasF = \bigcup_{i=0}^\infty\FormulasF_i
\)
contains all \emph{(infinitary propositional) formulas} over \(\Signature\).

In the following, we use the shortcuts
\begin{itemize}
\item \(\top = \emptyset^\wedge\) and \(\bot = \emptyset^\vee\),
\item \(\Naf F = F \rightarrow \bot\) where \(F\) is a formula, and
\item \(F \wedge G = {\{ F, G \}}^\wedge\) and \(F \vee G = {\{ F, G \}}^\vee\) where \(F\) and \(G\) are formulas.
\end{itemize}

We say that a formula is \emph{finitary}, if it has a finite number of subformulas.

An occurrence of a subformula in a formula is called \emph{positive},
if the number of implications containing that occurrence in the antecedent is even, and \emph{strictly positive} if that number is zero;
if that number is odd the occurrence is \emph{negative}.
The sets \(\IDPos{F}\) and \(\IDNeg{F}\) gather all atoms occurring positively or negatively in formula \(F\), respectively;
if applied to a set of formulas, both expressions stand for the union of the respective atoms in the formulas.
Also, we define \(\IDPosNeg{F}=\IDPos{F}\cup\IDNeg{F}\) as the set of all atoms occurring in \(F\).

A \emph{two-valued interpretation} over signature~\Signature\ is a set~\(I\) of propositional atoms
such that \(I \subseteq \Signature\).
Atoms in an interpretation \(I\) are considered \emph{true} and
atoms in \(\Signature \setminus I\) as \emph{false}.
The set of all interpretations together with the \(\subseteq\) relation forms a complete lattice.

The satisfaction relation between interpretations and formulas is defined as follows:
\begin{itemize}
\item \(I \models a\) for atoms \(a\) if \(a \in I\),
\item \(I \models \FormulasH^\wedge\) if \(I \models F\) for all \(F \in \FormulasH\),
\item \(I \models \FormulasH^\vee\) if \(I \models F\) for some \(F \in \FormulasH\), and
\item \(I \models F \rightarrow G\) if \(I \not\models F\) or \(I \models G\).
\end{itemize}
An interpretation~\(I\) is a \emph{model} of a set~\(\FormulasH\) of formulas,
written \(I \models \FormulasH\),
if it satisfies each formula in the set.

In the following, all atoms, formulas, and interpretations operate on the same (implicit) signature,
unless mentioned otherwise.

\subsection{Logic programs and stable models}\label{sec:stable}

Our terminology in this section keeps following the one of~\citeN{truszczynski12a}.

The \emph{reduct} \(\Reduct{F}{I}\) of a formula \(F\) w.r.t.\ an interpretation \(I\) is defined as:
\begin{itemize}
\item \(\bot\) if \(I \not\models F\),
\item \(a\) if \(I \models F\) and \(F = a\in\FormulasF_0\),
\item \({\{ \Reduct{G}{I} \mid G \in \FormulasH\}}^\wedge\) if \(I \models F\) and \(F = \FormulasH^\wedge\),
\item \({\{ \Reduct{G}{I} \mid G \in \FormulasH\}}^\vee\) if \(I \models F\) and \(F = \FormulasH^\vee\), and
\item \(\Reduct{G}{I} \rightarrow \Reduct{H}{I}\) if \(I \models F\) and \(F = G \rightarrow H\).
\end{itemize}
An interpretation \(I\) is a \emph{stable model} of a formula \(F\) if it is among the (set inclusion) minimal models of \(\Reduct{F}{I}\).

Note that the reduct removes (among other unsatisfied subformulas) all occurrences of atoms that are false in \(I\).
Thus, the satisfiability of the reduct does not depend on such atoms,
and
all minimal models of \(\Reduct{F}{I}\) are subsets of~\(I\).
Hence, if \(I\) is a stable model of \(F\), then it is the only minimal model of~\(\Reduct{F}{I}\).

Sets \(\mathcal{H}_1\) and \(\mathcal{H}_2\) of infinitary formulas are \emph{equivalent}
if they have the same stable models
and \emph{classically equivalent} if they have the same models;
they are \emph{strongly equivalent} if, for any set \(\mathcal{H}\) of infinitary formulas,
\(\mathcal{H}_1\cup\mathcal{H}\) and \(\mathcal{H}_2\cup\mathcal{H}\) are equivalent.
As shown by~\citeN{halipeva17a},
this also allows for replacing a part of any formula with a strongly equivalent formula
without changing the set of stable models.

In the following, we consider implications with atoms as consequent and formulas as antecedent.
As common in logic programming, they are referred to as rules, heads, and bodies, respectively,
and denoted by reversing the implication symbol.
More precisely,
an \emph{\(\FormulasF\)-program} is set of \emph{rules} of form
\(
h \leftarrow F
\)
where \(h \in \FormulasF_0\) and \(F\in\FormulasF\).
We use \(\Head{h \leftarrow F} = h\) to refer to rule \emph{heads}
and \(\Body{h \leftarrow F} = F\) to refer to rule \emph{bodies}.
We extend this by letting
\(\Head{P}\ = \{ \Head{r} \mid r \in P \}\) and \(\Body{P}\ = \{ \Body{r} \mid r \in P \}\)
for any program \(P\).

An interpretation \(I\) is a model of an \(\FormulasF\)-program \(P\), written \(I \models P\),
if \(I \models \Body{r} \rightarrow \Head{r}\) for all \(r \in P\).
The latter is also written as \(I \models r\).
We define the reduct of \(P\) w.r.t.\ \(I\) as
\(
\Reduct{P}{I} = \{ \Reduct{r}{I} \mid r \in P \}
\)
where \(\Reduct{r}{I} = \Head{r} \leftarrow \Reduct{\Body{r}}{I}\).
As above,
an interpretation \(I\) is a \emph{stable model} of \(P\) if \(I\) is among the minimal models of \(\Reduct{P}{I}\).
Just like the original definition of~\citeN{gellif88b},
the reduct of such programs leaves rule heads intact and only reduces rule bodies.
(This feature fits well with the various operators defined in the sequel.)

This program-oriented reduct yields the same stable models as obtained
by applying the full reduct to the corresponding infinitary formula.

\begin{proposition}\label{prp:formula-program}
Let \(P\) be an \(\FormulasF\)-program.

Then, the stable models of formula~\({\{ \Body{r} \rightarrow \Head{r} \mid r \in P \}}^\wedge\) are the same as the stable models of program~\(P\).
\end{proposition}

For programs,
\citeN{truszczynski12a} introduces in an alternative reduct,
replacing each negatively occurring atom with \(\bot\), if it is falsified,
and with \(\top\), otherwise.
More precisely,
the so-called \emph{\FOID-reduct} \IDReductP{F}{I} of a formula \(F\) w.r.t.\ an interpretation \(I\) is defined as
\begin{align*}
\IDReductP{a}{I}                     & = a                                                       &
\IDReductN{a}{I}                     & = \top \text{\ if } a \in I                               \\
                                     &                                                           &
\IDReductN{a}{I}                     & = \bot \text{\ if } a \notin I                            \\
\IDReductP{\FormulasH^{\wedge}}{I}   & = {\{ \IDReductP{F}{I} \mid F \in \FormulasH \}}^\wedge   &
\IDReductN{\FormulasH^{\wedge}}{I}   & = {\{ \IDReductN{F}{I} \mid F \in \FormulasH \}}^\wedge   \\
\IDReductP{\FormulasH^{\vee}}{I}     & = {\{ \IDReductP{F}{I} \mid F \in \FormulasH \}}^\vee     &
\IDReductN{\FormulasH^{\vee}}{I}     & = {\{ \IDReductN{F}{I} \mid F \in \FormulasH \}}^\vee     \\
\IDReductPC{F \rightarrow G}{I}      & = \IDReductN{F}{I} \rightarrow \IDReductP{G}{I}           &
\IDReductNC{F \rightarrow G}{I}      & = \IDReductP{F}{I} \rightarrow \IDReductN{G}{I}
\end{align*}
where \(a\) is an atom, \(\FormulasH\) a set of formulas, and \(F\) and \(G\) are formulas.

The \FOID-reduct of an \(\FormulasF\)-program~\(P\) w.r.t.\ an interpretation~\(I\) is
\(\IDReductP{P}{I} = \{ \IDReductP{r}{I} \mid r \in P \}\) where
\(\IDReductP{r}{I} = \Head{r} \leftarrow \IDReductP{\Body{r}}{I}\).
As with \(\Reduct{r}{I}\),
the transformation of \(r\) into \IDReductP{r}{I} leaves the head of \(r\) unaffected.

\begin{example}\label{ex:reduct}
Consider the program containing the single rule
\begin{align*}
p & \leftarrow \Naf \Naf p.
\end{align*}
We get the following reduced programs w.r.t.\ interpretations \(\emptyset\) and \(\{p\}\):
\begin{align*}
  \Reduct{\{p  \leftarrow \Naf \Naf p \}}{\emptyset}&=\{p  \leftarrow \bot\}                 &
  \Reduct{\{p  \leftarrow \Naf \Naf p \}}{\{p\}}&=\{p  \leftarrow \Naf \bot\}\\
  \IDReductP{\{p  \leftarrow \Naf \Naf p \}}{\emptyset}&=\{p  \leftarrow \Naf \Naf p\}\qquad=&
  \IDReductP{\{p  \leftarrow \Naf \Naf p \}}{\{p\}}&=\{p  \leftarrow \Naf \Naf p\}
\end{align*}
Note that both reducts leave the rule's head intact.
\end{example}

Extending the definition of positive occurrences,
we define
a formula as \emph{(strictly) positive} if all its atoms occur (strictly) positively in the formula.
We define an \(\FormulasF\)-program as (strictly) positive if all its rule bodies are (strictly) positive.

For example, the program in \cref{ex:reduct} is positive but not strictly positive
because the only body atom \(p\) appears in the scope of two antecedents within the rule body \(\Naf \Naf p\).

As put forward by~\citeN{emdkow76a},
we may associate with each program \(P\) its \emph{one-step provability operator} \(\StepO{P}\),
defined for any interpretation \(X\) as
\begin{align*}
  \Step{P}{X} &= \{ \Head{r} \mid r \in P, X \models \Body{r} \}.
\end{align*}

\begin{proposition}[\protect\citeNP{truszczynski12a}]\label{prop:step-monotone}
Let \(P\) be a positive \(\FormulasF\)-program.

Then, the operator \(\StepO{P}\) is monotone.
\end{proposition}

Fixed points of \(\StepO{P}\) are models of \(P\) guaranteeing that each contained atom is supported by some rule in \(P\);
prefixed points of \(\StepO{P}\) correspond to the models of \(P\).
According to~\Nref{thm:knaster-tarski}{a}, the \(\StepO{P}\) operator has a least fixed point
for positive \(\FormulasF\)-programs.
We refer to this fixed point as the \emph{least model} of \(P\),
and write it as \(\LeastM{P}\).

Observing that the \FOID-reduct replaces all negative occurrences of atoms,
any \FOID-reduct \IDReductP{P}{I} of a program w.r.t.\ an interpretation \(I\)
is positive and thus possesses a least model \LeastM{\IDReductP{P}{I}}.
This gives rise to the following definition of a stable operator~\cite{truszczynski12a}:
Given an \(\FormulasF\)-program \(P\),
its \emph{\FOID-stable operator} is defined for any interpretation \(I\) as
\[
  \Stable{P}{I} = \LeastM{\IDReductP{P}{I}}.
\]
The fixed points of \(\StableO{P}\) are the \emph{\FOID-stable models} of \(P\).

Note that neither the program reduct \Reduct{P}{I} nor the formula reduct \Reduct{F}{I} guarantee (least) models.
Also, stable models and \FOID-stable models do not coincide in general.
\begin{example}\label{ex:stable}
Reconsider the program from \Cref{ex:reduct}, comprising rule
\begin{align*}
p & \leftarrow \Naf \Naf p.
\end{align*}
This program has the two stable models \(\emptyset\) and \(\{p\}\),
but the empty model is the only \FOID-stable model.
\end{example}

\begin{proposition}[\protect\citeNP{truszczynski12a}]\label{lem:stable:operator:antimonotone}
Let \(P\) be an \FormulasF-program.

Then, the \FOID-stable operator~\(\StableO{P}\) is antimonotone.
\end{proposition}

No analogous antimonotone operator is obtainable for \FormulasF-programs by using the program reduct \Reduct{P}{I}
(and for general theories with the formula reduct \Reduct{F}{I}).
To see this, reconsider Example~\ref{ex:stable} along with its two stable models \(\emptyset\) and \(\{p\}\).
Given that both had to be fixed points of such an operator, it would behave monotonically on \(\emptyset\) and \(\{p\}\).

In view of this, we henceforth consider exclusively \FOID-stable operators and drop the prefix `\FOID'.
However, we keep the distinction between stable and \FOID-stable models.

\citeN{truszczynski12a} identifies in a class of programs for which stable models and \FOID-stable models coincide.
The set \(\FormulasN\) consists of all formulas \(F\) such that any implication in \(F\) has \(\bot\) as consequent and
no occurrences of implications in its antecedent.
An \(\FormulasN\)-program consists of rules of form \( h \leftarrow F \) where \(h \in \FormulasF_0\) and \(F \in \FormulasN\).

\begin{proposition}[\protect\citeNP{truszczynski12a}]
Let \(P\) be an \(\FormulasN\)-program.

Then, the stable and \FOID-stable models of \(P\) coincide.
\end{proposition}

Note that a positive \(\FormulasN\)-program is also strictly positive.

\subsection{Well-founded models}\label{sec:well-founded}

Our terminology in this section follows the one of~\citeN{truszczynski18a} and traces back to the early work of~\citeN{belnap77a} and~\citeN{fitting02a}.\footnote{The interested reader is referred to the tutorial by~\citeN{truszczynski18a} for further details.}

We deal with pairs of sets and extend the basic set relations and operations accordingly.
Given sets \(I'\), \(I\), \(J'\), \(J\), and \(X\), we define:
\begin{itemize}
\item \((I',J') \mathrel{\bar{\prec}} (I,J)\) if \(I' \prec I\) and \(J' \prec J\) \hfill for \((\bar{\prec}, \prec) \in \{ (\sqsubset, \subset), (\sqsubseteq, \subseteq) \}\)
\item \((I',J') \mathrel{\bar{\circ}} (I,J) = (I' \circ I, J' \circ J)\)           \hfill for \((\bar{\circ}, \circ) \in \{ (\sqcup, \cup), (\sqcap, \cap), (\smallsetminus,\setminus) \}\)
\item \((I,J) \mathrel{\bar{\circ}} X = (I,J) \mathrel{\bar{\circ}} (X,X)\)        \hfill for \(\bar{\circ} \in \{ \sqcup, \sqcap, \smallsetminus \}\)
\end{itemize}

A \emph{four-valued interpretation} over signature~\Signature\ is represented by a pair \((I,J) \sqsubseteq (\Signature, \Signature)\)
where \(I\) stands for \emph{certain} and \(J\) for \emph{possible} atoms.
Intuitively, an atom that is
\begin{itemize}
\item certain and possible is \emph{true},
\item certain but not possible is \emph{inconsistent},
\item not certain but possible is \emph{unknown}, and
\item not certain and not possible is \emph{false}.
\end{itemize}

A four-valued interpretation \((I',J')\) is more precise than a four-valued interpretation \((I,J)\), written \((I,J) \leq_p (I',J')\), if \(I \subseteq I'\) and \(J' \subseteq J\).
The precision ordering also has an intuitive reading: the more atoms are certain or the fewer atoms are possible, the more precise is an interpretation.
The least precise four-valued interpretation over~\Signature\ is \((\emptyset,\Signature)\).
As with two-valued interpretations,
the set of all four-valued interpretations over a signature~\Signature\ together with the relation \(\leq_p\) forms a complete lattice.
A four-valued interpretation is called \emph{inconsistent} if it contains an inconsistent atom; otherwise, it is called \emph{consistent}.
It is \emph{total} whenever it makes all atoms either true or false.
Finally, \((I,J)\) is called \emph{finite} whenever both \(I\) and \(J\) are finite.

Following~\citeN{truszczynski18a},
we define the \emph{\FOID-well-founded operator} of
an \(\FormulasF\)-program \(P\)
for any four-valued interpretation \((I,J)\) as
\[
  \Well{P}{I,J} = (\Stable{P}{J}, \Stable{P}{I}).
\]
This operator is monotone w.r.t.\ the precision ordering \(\leq_p\).
Hence, by~\Nref{thm:knaster-tarski}{a}, \(\WellO{P}\) has a least fixed point,
which defines the \emph{\FOID-well-founded model} of \(P\), also written as \(\WellM{P}\).
In what follows,
we drop the prefix `\FOID' and simply refer to the \FOID-well-founded model of a program as its well-founded model.
(We keep the distinction between stable and \FOID-stable models.)

Any well-founded model \((I,J)\) of an \(\FormulasF\)-program~\(P\) satisfies \(I \subseteq J\).

\begin{lemma}\label{lem:well-founded:consistent}
Let \(P\) be an \(\FormulasF\)-Program.

Then, the well-founded model \(\WellM{P}\) of \(P\) is consistent.
\end{lemma}

\begin{example}\label{ex:well-founded}
Consider program~\(\Fi{P}{\ref{ex:well-founded}}\) consisting of the following rules:
\begin{align*}
a &            \\
b & \leftarrow a \\
c & \leftarrow \Naf b \\
d & \leftarrow c \\
e & \leftarrow \Naf d
\end{align*}

We compute the well-founded model of \(\Fi{P}{\ref{ex:well-founded}}\) in four iterations starting from \((\emptyset,\Signature)\):
{\leqnomode
\begin{align*}
\Stable{P}{\Signature}        & = \{a, b\}    & \Stable{P}{\emptyset}   & = \{a, b, c, d, e\} \tag*{\hspace{1.2em}1.}\\
\Stable{P}{\{a, b, c, d, e\}} & = \{a, b\}    & \Stable{P}{\{a, b\}}    & = \{a, b, e\}       \tag*{\hspace{1.2em}2.}\\
\Stable{P}{\{a, b, e\}}       & = \{a, b, e\} & \Stable{P}{\{a, b\}}    & = \{a, b, e\}       \tag*{\hspace{1.2em}3.}\\
\Stable{P}{\{a, b, e\}}       & = \{a, b, e\} & \Stable{P}{\{a, b, e\}} & = \{a, b, e\}       \tag*{\hspace{1.2em}4.}
\end{align*}}The left and right column reflect the certain and possible atoms computed at each iteration, respectively.
We reach a fixed point at Step~4.
Accordingly,
the well-founded model of \(\Fi{P}{\ref{ex:well-founded}}\) is \((\{a, b, e\}, \{a, b, e\})\).
\end{example}

Unlike general \(\FormulasF\)-programs,
the class of \(\FormulasN\)-programs warrants the same stable and \FOID-stable models for each of its programs.
Unfortunately, \(\FormulasN\)-programs are too restricted for our purpose
(for instance, for capturing aggregates in rule bodies\footnote{Ferraris' semantics~\cite{ferraris11a} of aggregates introduces implications,
  which results in rules beyond the class of \(\FormulasN\)-programs.}).
To this end, we define a more general class of programs and refer to them as \(\FormulasR\)-programs.
Although \FOID-stable models of \(\FormulasR\)-programs may differ from their stable models (see below),
their well-founded models encompass both stable and \FOID-stable models.
Thus, well-founded models can be used for characterizing stable model-preserving program transformations.
In fact, we see in Section~\ref{sec:simplification} that the restriction of \(\FormulasF\)- to \(\FormulasR\)-programs allows us to provide tighter
semantic characterizations of program simplifications.

We define \(\FormulasR\) to be the set of all formulas \(F\) such that
implications in \(F\) have no further occurrences of implications in their antecedents.
Then, an \emph{\(\FormulasR\)-program} consists of rules of form
\(
h \leftarrow F
\)
where \(h \in \FormulasF_0\) and \(F \in \FormulasR\).
As with \(\FormulasN\)-programs, a positive \(\FormulasR\)-program is also strictly positive.

Our next result shows that
(\FOID-)well-founded models can be used for approximating (regular) stable models of \(\FormulasR\)-programs.

\begin{theorem}\label{thm:stable:well-founded}
Let \(P\) be an \(\FormulasR\)-program and \((I,J)\) be the well-founded model of \(P\).

If \(X\) is a stable model of \(P\), then \(I \subseteq X \subseteq J\).
\end{theorem}

\begin{example}\label{ex:r-program}
Consider the \(\FormulasR\)-program \(\Fi{P}{\ref{ex:r-program}}\):\footnote{The choice of the body \(b \rightarrow a\) is not arbitrary since it can be seen as representing the aggregate \(\SumA{\Elem{1}{a}, \Elem{-1}{b}}{\geq}{0}\).}
\begin{align*}
c & \leftarrow (b \rightarrow a) &
a & \leftarrow b \\
a & \leftarrow c &
b & \leftarrow a
\end{align*}
Observe that \(\{a, b, c\}\) is the only stable model of \(\Fi{P}{\ref{ex:r-program}}\),
the program does not have any \FOID-stable models, and
the well-founded model of \(\Fi{P}{\ref{ex:r-program}}\) is \((\emptyset,\{a,b,c\})\).
In accordance with \cref{thm:stable:well-founded},
the stable model of \(\Fi{P}{\ref{ex:r-program}}\) is enclosed in the well-founded model.

Note that the \FOID-reduct handles \(b \rightarrow a\) the same way as \(\Naf b \vee a\).
In fact, the program obtained by replacing
\begin{align*}
c & \leftarrow (b \rightarrow a)
\end{align*}
with
\begin{align*}
c & \leftarrow \Naf b \vee a
\end{align*}
is an \FormulasN-program and has neither stable nor \FOID-stable models.
\end{example}

Further, note that the program in Example~\ref{ex:stable} is not an \FormulasR-program,
whereas the one in Example~\ref{ex:well-founded} is an \FormulasR-program.

\subsection{Program simplification}\label{sec:simplification}

In this section,
we define a concept of program simplification relative to a four-valued interpretation and
show how its result can be characterized by the semantic means from above.
This concept has two important properties.
First, it results in a finite program whenever the interpretation used for simplification is finite.
And second, it preserves all (regular) stable models of \(\FormulasR\)-programs when simplified with their well-founded models.

\begin{definition}\label{def:simplification}
Let \(P\) be an \(\FormulasF\)-program, and \((I,J)\) be a four-valued interpretation.

We define the simplification of \(P\) w.r.t.\ \((I,J)\) as
\begin{align*}
\Simp{P}{(I,J)} & = \{ r \in P \mid J \models \IDReductP{\Body{r}}{I} \}.
\end{align*}
\end{definition}

For simplicity, we drop parentheses and we write \(\Simp{P}{I,J}\) instead of \(\Simp{P}{(I,J)}\)
whenever clear from context.

The program simplification
\(\Simp{P}{I,J}\) acts as a filter eliminating inapplicable rules that
fail to satisfy the condition~\(J \models \IDReductP{\Body{r}}{I}\).
That is,
first, all negatively occurring atoms in \(\Body{r}\) are evaluated w.r.t.\ the certain atoms in \(I\)
and replaced accordingly by \(\bot\) and \(\top\), respectively.
Then, it is checked whether the reduced body \IDReductP{\Body{r}}{I} is satisfiable by the possible atoms in \(J\).
Only in this case, the rule is kept in \(P^{I,J}\).
No simplifications are applied to the remaining rules.
This is illustrated in \Cref{ex:simplification} below.

Note that \(\Simp{P}{I,J}\) is finite whenever \((I,J)\) is finite.

Observe that for an \(\FormulasF\)-program~\(P\)
the head atoms in \({\Simp{P}{I,J}}\) correspond to the result of applying the provability operator of
program \(\IDReductP{P}{I}\) to the possible atoms in \(J\), i.e., \(\Head{\Simp{P}{I,J}} = \Step{\IDReductP{P}{I}}{J}\).

Our next result shows that programs simplified with their well-founded model maintain this model.

\begin{theorem}\label{thm:well-founded:simplification}
Let \(P\) be an \(\FormulasF\)-program and \((I,J)\) be the well-founded model of \(P\).

Then, \(P\) and \(\Simp{P}{I,J}\) have the same well-founded model.
\end{theorem}

\begin{example}\label{ex:simplification}
In \cref{ex:well-founded}, we computed the well-founded model~\((\{a, b, e\}, \{a, b, e\})\) of \(\Fi{P}{\ref{ex:well-founded}}\).
With this, we obtain the simplified program \(\Fpi{P}{\ref{ex:well-founded}} = \Simp{\Fi{P}{\ref{ex:well-founded}}}{\{a, b, e\}, \{a, b, e\}}\)
after dropping \(c \leftarrow \Naf b\) and \(d \leftarrow c\):
\begin{align*}
a &            \\
b & \leftarrow a \\
e & \leftarrow \Naf d
\end{align*}

Next, we check that the well-founded model of \(\Fpi{P}{\ref{ex:well-founded}}\) corresponds to the well-founded model of \(\Fi{P}{\ref{ex:well-founded}}\):
{\leqnomode \begin{align*}
\Stable{\Fpi{P}{\ref{ex:well-founded}}}{\Signature}  & = \{a, b\}    & \Stable{\Fpi{P}{\ref{ex:well-founded}}}{\emptyset}   & = \{a, b, e\} \tag*{\hspace{1.2em}1.}\\
\Stable{\Fpi{P}{\ref{ex:well-founded}}}{\{a, b, e\}} & = \{a, b, e\} & \Stable{\Fpi{P}{\ref{ex:well-founded}}}{\{a,b\}}     & = \{a, b, e\} \tag*{\hspace{1.2em}2.}\\
\Stable{\Fpi{P}{\ref{ex:well-founded}}}{\{a, b, e\}} & = \{a, b, e\} & \Stable{\Fpi{P}{\ref{ex:well-founded}}}{\{a, b, e\}} & = \{a, b, e\} \tag*{\hspace{1.2em}3.}
\end{align*}}

We observe that it takes two applications of the well-founded operator to obtain the well-founded model.
This could be reduced to one step if atoms false in the well-founded model would be removed from the negative bodies by the program simplification.
Keeping them is a design decision with the goal to simplify notation in the following.
\end{example}

The next series of results further elaborates on semantic invariants guaranteed by our concept of program simplification.
The first result shows that it preserves all stable models between the sets used for simplification.

\begin{theorem}\label{thm:stable:simplification}
Let \(P\) be an \(\FormulasF\)-program, and \(I\), \(J\) and \(X\) be two-valued interpretations.

If \(I \subseteq X \subseteq J\), then \(X\) is a stable model of \(P\) iff \(X\) is a stable model of \(\Simp{P}{I,J}\).
\end{theorem}

As a consequence,
we obtain that \(\FormulasR\)-programs simplified with their well-founded model also maintain stable models.

\begin{corollary}\label{cor:stable:simplification}
Let \(P\) be an \(\FormulasR\)-program and \((I,J)\) be the well-founded model of \(P\).

Then, \(P\) and \(\Simp{P}{I,J}\) have the same stable models.
\end{corollary}

For instance, the \FormulasR-program in \Cref{ex:well-founded} and its simplification in
\Cref{ex:simplification} have the same stable model.
Unlike this, the program from \Cref{ex:stable} consisting of rule \(p \leftarrow \Naf \Naf p\) induces two stable models,
while its simplification w.r.t.\ its well-founded model \((\emptyset,\emptyset)\) yields an empty program
admitting the empty stable model only.

Note that given an \FormulasR-program with a finite well-founded model,
we obtain a semantically equivalent finite program via simplification.
As detailed in the following sections,
grounding algorithms only compute approximations of the well-founded model.
However, as long as the approximation is finite,
we still obtain semantically equivalent finite programs.
This is made precise by the next two results showing that any program between the original and its simplification relative to its well-founded model
preserves the well-founded model,
and that this extends to all stable models for \(\FormulasR\)-programs.

\begin{theorem}\label{thm:well-founded:simplification:superset}
Let \(P\) and \(Q\) be \(\FormulasF\)-programs, and \((I,J)\) be the well-founded model of \(P\).

If \(\Simp{P}{I,J} \subseteq Q \subseteq P\), then \(P\) and \(Q\) have the same well-founded models.
\end{theorem}

\begin{corollary}\label{cor:stable:simplification:superset}
Let \(P\) and \(Q\) be \(\FormulasR\)-programs, and \((I,J)\) be the well-founded model of \(P\).

If \(\Simp{P}{I,J} \subseteq Q \subseteq P\), then \(P\) and \(Q\) are equivalent.
\end{corollary}

 \section{Splitting}\label{sec:splitting}

One of the first steps during grounding is to group rules into components suitable for successive instantiation.
This amounts to splitting a logic program into a sequence of subprograms.
The rules in each such component are then instantiated with respect to the atoms possibly derivable from previous components,
starting with some component consisting of facts only.
In other words, grounding is always performed relative to a set of atoms that provide a context.
Moreover, atoms found to be true or false can be used for on-the-fly simplifications.

Accordingly, this section parallels the above presentation by
extending the respective formal concepts with contextual information provided by atoms in a two- and four-valued setting.
We then assemble the resulting concepts to enable their consecutive application to sequences of subprograms.
Interestingly, the resulting notion of splitting allows for more fine-grained splitting than
the traditional concept~\cite{liftur94a} since it allows us to partition rules in an arbitrary way.
In view of grounding,
we show that once a program is split into a sequence of programs,
we can iteratively compute an approximation of the well-founded model by considering in turn each element in the sequence.

In what follows, we append letter `\(\Relative\)' to names of interpretations having a contextual nature.

To begin with,
we extend the one-step provability operator accordingly.

\begin{definition}
Let \(P\) be an \(\FormulasF\)-program and \(\Ir{I}\) be a two-valued interpretation.

For any two-valued interpretation \(I\),
we define the \emph{one-step provability operator of \(P\) relative to} \(\Ir{I}\) as
\begin{align*}
\StepR{P}{\Ir{I}}{I} &= \Step{P}{\Ir{I} \cup I}.
\end{align*}
\end{definition}

A prefixed point of \(\StepRO{P}{\Ir{I}}\) is a also a prefixed point of \(\StepO{P}\).
Thus, each prefixed point of \(\StepRO{P}{\Ir{I}}\) is a model of \(P\) but not vice versa.

To see this, consider program \(P=\{a \leftarrow b\}\).
We have \(\Step{P}{\emptyset} = \emptyset\) and \(\StepR{P}{\{b\}}{\emptyset} = \{a\}\).
Hence, \(\emptyset\) is a {(pre)}fixed point of \(\StepO{P}\) but not of \(\StepRO{P}{\{b\}}\) since \(\{a\} \not\subseteq \emptyset\).
The set \(\{a\}\) is a prefixed point of both operators.

\begin{proposition}\label{prop:rel-step-monotone}
Let \(P\) be a positive program, and \(\Ir{I}\) and \(J\) be two valued interpretations.

Then, the operators \(\StepRO{P}{\Ir{I}}\) and \(\StepR{P}{\,\boldsymbol{\cdot}}{J}\) are both monotone.
\end{proposition}

We use \cref{thm:knaster-tarski,prop:rel-step-monotone} to define a contextual stable operator.

\begin{definition}
Let \(P\) be an \(\FormulasF\)-program and \(\Ir{I}\) be a two-valued interpretation.

For any two-valued interpretation~\(J\),
we define the \emph{stable operator relative} to \(\Ir{I}\), written \(\StableR{P}{\Ir{I}}{J}\), as the least fixed point of \(\StepRO{\IDReductP{P}{J}}{\Ir{I}}\).
\end{definition}

While the operator is antimonotone w.r.t.\ its argument \(J\),
it is monotone regarding its parameter~\(\Ir{I}\).

\begin{proposition}\label{prp:stable:relative:monotonicity}
Let \(P\) be an \(\FormulasF\)-program, and \(\Ir{I}\) and \(J\) be two-valued interpretations.

Then, the operators \(\StableRO{P}{\Ir{I}}\) and \(\StableR{P}{\,\boldsymbol{\cdot}}{J}\) are antimonotone and monotone, respectively.
\end{proposition}

By building on the relative stable operator, we next define its well-founded counterpart.
Unlike above,
the context is now captured by a four-valued interpretation.

\begin{definition}\label{def:well-founded:operator:relative}
Let \(P\) be an \(\FormulasF\)-program and \((\Ir{I},\Ir{J})\) be a four-valued interpretation.

For any four-valued interpretation \((I,J)\),
we define the \emph{well-founded operator relative} to \((\Ir{I},\Ir{J})\) as
\begin{align*}
  \WellR{P}{(\Ir{I},\Ir{J})}{I,J} = (\StableR{P}{\Ir{I}}{J \cup \Ir{J}}, \StableR{P}{\Ir{J}}{I \cup \Ir{I}}).
\end{align*}
\end{definition}

As above, we drop parentheses and simply write \(\WellRO{P}{I,J}\) instead of \(\WellRO{P}{(I,J)}\).
Also, we keep refraining from prepending the prefix `\FOID' to the well-founded operator
along with all concepts derived from it below.

Unlike the stable operator,
the relative well-founded one is monotone on both its argument and parameter.

\begin{proposition}\label{prp:well-founded:relative:monotonicity}
Let \(P\) be an \(\FormulasF\)-program, and
\((I,J)\) and \((\Ir{I},\Ir{J})\) be four-valued interpretations.

Then, the operators \(\WellRO{P}{\Ir{I},\Ir{J}}\) and \(\WellR{P}{\,\boldsymbol{\cdot}}{I,J}\) are both monotone w.r.t.\ the precision ordering.
\end{proposition}

From \cref{thm:knaster-tarski,prp:well-founded:relative:monotonicity}, we get that the relative well-founded operator has a least fixed point.

\begin{definition}
Let \(P\) be an \(\FormulasF\)-program and \((\Ir{I},\Ir{J})\) be a four-valued interpretation.

We define the \emph{well-founded model} of \(P\) \emph{relative} to \((\Ir{I},\Ir{J})\),
written \(\WellRM{P}{(\Ir{I},\Ir{J})}\),
as the least fixed point of \(\WellRO{P}{\Ir{I},\Ir{J}}\).
\end{definition}

Whenever clear from context, we keep dropping parentheses and simply write \(\WellRM{P}{I,J}\) instead of \(\WellRM{P}{(I,J)}\).

In what follows,
we use the relativized concepts defined above to delineate the semantics and resulting simplifications of the sequence of subprograms
resulting from a grounder's decomposition of the original program.
For simplicity,
we first present a theorem capturing the composition under the well-founded operation,
before we give the general case involving a sequence of programs.

Just like suffix~\(\Relative\),
we use the suffix~\(\External\) (and similarly letter~\(E\) further below) to indicate atoms whose defining
rules are yet to come.

As in traditional splitting,
we begin by differentiating a bottom and a top program.
In addition to the input atoms~\((I,J)\) and context atoms in~\((\Ir{I},\Ir{J})\),
we moreover distinguish a set of external atoms, \((\Ie{I},\Ie{J})\),
which occur in the bottom program but are defined in the top program.
Accordingly,
the bottom program has to be evaluated relative to \((\Ir{I},\Ir{J}) \sqcup (\Ie{I},\Ie{J})\) (and not just \((\Ir{I},\Ir{J})\) as above)
to consider what could be derived by the top program.
Also, observe that our notion of splitting aims at computing well-founded models rather than stable models.

\begin{theorem}\label{prp:splitting:well-founded}
Let \(\Fb{P}\) and \(\Ft{P}\) be \(\FormulasF\)-programs,
\((\Ir{I},\Ir{J})\) be a four-valued interpretation,
\((I,J) = \WellRM{\Fb{P} \cup \Ft{P}}{\Ir{I},\Ir{J}}\),
\((\Ie{I},\Ie{J}) = (I, J) \sqcap (\IDPosNeg{\Body{\Fb{P}}} \cap \Head{\Ft{P}})\),
\((\Ib{I},\Ib{J}) = \WellRM{\Fb{P}}{(\Ir{I}, \Ir{J}) \sqcup (\Ie{I},\Ie{J})}\),
and \(({\It{I}},{\It{J}}) = \WellRM{\Ft{P}}{(\Ir{I}, \Ir{J}) \sqcup (\Ib{I},\Ib{J})}\).

Then, we have \((I,J) = (\Ib{I},\Ib{J}) \sqcup ({\It{I}},{\It{J}})\).
\end{theorem}

Partially expanding the statements of the two previous result nicely reflects the decomposition of
the application of the well-founded founded model of a program:
\begin{align*}
  \WellRM{\Fb{P} \cup \Ft{P}}{\Ir{I},\Ir{J}} &= \WellRM{\Fb{P}}{(\Ir{I}, \Ir{J}) \sqcup (\Ie{I},\Ie{J})} \sqcup \WellRM{\Ft{P}}{(\Ir{I}, \Ir{J}) \sqcup (\Ib{I},\Ib{J})}.
\end{align*}
Note that the formulation of the theorem forms the external interpretation \((\Ie{I},\Ie{J})\),
by selecting atoms from the overarching well-founded model \((I,J)\).
This warrants the correspondence of the overall interpretations to the union of the bottom and top well-founded model.
This global approach is dropped below (after the next example) and leads to less precise composed models.

\begin{example}\label{ex:splitting}
Let us illustrate the above approach via the following program:
\begin{align}
a &                    \tag{\(\Fb{P}\)}\\
b &                    \tag{\(\Fb{P}\)}\\
c & \leftarrow a       \tag{\(\Ft{P}\)}\\
d & \leftarrow \Naf b  \tag{\(\Ft{P}\)}
\end{align}

The well-founded model of this program relative to \((\Ir{I},\Ir{J})=(\emptyset, \emptyset)\) is
\[(I,J) = (\{a,b,c\},\{a,b,c\}).\]

First, we partition the four rules of the program into
\(\Fb{P}\)
and
\(\Ft{P}\)
as given above.
We get
\((\Ie{I},\Ie{J})= (\emptyset,\emptyset)\)
since
\(
\IDPosNeg{\Body{\Fb{P}}} \cap \Head{\Ft{P}}=\emptyset
\).
Let us evaluate \(\Fb{P}\) before \(\Ft{P}\).
The well-founded model of \(\Fb{P}\) relative to \((\Ir{I},\Ir{J}) \sqcup (\Ie{I},\Ie{J})\) is
\[(\Ib{I},\Ib{J}) = (\{a,b\},\{a,b\}).\]
With this, we calculate the well-founded model of \(\Ft{P}\) relative to \((\Ir{I},\Ir{J}) \sqcup (\Ib{I},\Ib{J})\):
\[(\It{I},\It{J}) = (\{c\},\{c\}).\]
We see that the union of \((\Ib{I},\Ib{J}) \sqcup (\It{I},\It{J})\) is the same as the well-founded model of \(\Fb{P} \cup \Ft{P}\) relative to \((\Ir{I},\Ir{J})\).

This corresponds to standard splitting in the sense that \(\{a,b\}\) is a splitting set for \(\Fb{P}\cup \Ft{P}\)
and \(\Fb{P}\) is the ``bottom'' and \(\Ft{P}\) is the ``top''~\cite{liftur94a}.
\end{example}

\begin{example}
For a complement, let us reverse the roles of programs
\(\Fb{P}\)
and
\(\Ft{P}\)
in \cref{ex:splitting}.
Unlike above, body atoms in \(\Fb{P}\) now occur in rule heads of \(\Ft{P}\), i.e.,
\(\IDPosNeg{\Body{\Fb{P}}} \cap {\Head{\Ft{P}}} = \{a,b\}\).
We thus get \((\Ie{I},\Ie{J})= (\{a,b\}, \{a, b\})\).
The well-founded model of \(\Fb{P}\) relative to \((\Ir{I},\Ir{J}) \sqcup (\Ie{I},\Ie{J})\) is
\begin{align*}
(\Ib{I},\Ib{J}) & = (\{c\}, \{c\}).
\end{align*}
And the well-founded model of \(\Ft{P}\) relative to \((\Ir{I},\Ir{J}) \sqcup (\Ib{I},\Ib{J})\) is
\begin{align*}
(\It{I},\It{J}) & = (\{a,b\}, \{a,b\}).
\end{align*}
Again, we see that the union of both models is identical to \((I,J)\).

This decomposition has no direct correspondence to standard splitting~\cite{liftur94a}
since there is no splitting set.
\end{example}

Next, we generalize the previous results from two programs to sequences of programs.
For this, we let \(\Index\) be a well-ordered index set and direct our attention to sequences \({(\Fi{P}{i})}_{i \in \Index}\) of \(\FormulasF\)-programs.

\begin{definition}\label{def:well-founded:sequence}
Let \({(\Fi{P}{i})}_{i \in \Index}\) be a sequence of \(\FormulasF\)-programs.

We define the \emph{well-founded model} of \({(\Fi{P}{i})}_{i \in \Index}\) as
\begin{align}
\WellM{{(\Fi{P}{i})}_{i \in \Index}} & = \bigsqcup_{i \in \Index} (\Ii{I}{i},\Ii{J}{i})\label{eq:well-founded:sequence}\\
\intertext{where}
\Ii{E}{i} & = \IDPosNeg{\Body{\Fi{P}{i}}} \cap \bigcup_{i < j} \Head{\Fi{P}{j}},\label{eq:well-founded:sequence:external}\\
(\Iri{I}{i},\Iri{J}{i}) & = \bigsqcup_{j < i} (\Ii{I}{j},\Ii{J}{j})\text{, and}\label{eq:well-founded:sequence:relative}\\
(\Ii{I}{i},\Ii{J}{i}) & = \WellRM{\Fi{P}{i}}{(\Iri{I}{i},\Iri{J}{i}) \sqcup (\emptyset,\Ii{E}{i})}.\label{eq:well-founded:sequence:intermediate}
\end{align}
\end{definition}

The well-founded model of a program sequence is itself assembled in~\eqref{eq:well-founded:sequence}
from a sequence of well-founded models of the individual subprograms in~\eqref{eq:well-founded:sequence:intermediate}.
This provides us with semantic guidance for successive program simplification, as shown below.
In fact,
proceeding along the sequence of subprograms reflects the iterative approach of a grounding algorithm,
one component is grounded at a time.
At each stage \(i \in \Index\),
this takes into account the truth values of atoms instantiated in previous iterations,
viz.\ \((\Iri{I}{i},\Iri{J}{i})\),
as well as dependencies to upcoming components in \(\Ii{E}{i}\).
Note that unlike \cref{prp:splitting:well-founded},
the external atoms in \(\Ii{E}{i}\) are identified purely syntactically,
and the interpretation \((\emptyset,\Ii{E}{i})\) treats them as unknown.
Grounding is thus performed under incomplete information and each well-founded model in~\eqref{eq:well-founded:sequence:intermediate}
can be regarded as an over-approximation of the actual one.
This is enabled by the monotonicity of the well-founded operator in \cref{prp:well-founded:relative:monotonicity}
that only leads to a less precise result when overestimating its parameter.

Accordingly, the next theorem shows that once we split a program into a sequence of \(\FormulasF\)-programs,
we can iteratively compute an approximation of the well-founded model by considering in turn each element in the sequence.

\begin{theorem}\label{thm:sequence:well-founded}
Let \({(\Fi{P}{i})}_{i \in \Index}\) be a sequence of \(\FormulasF\)-programs.

Then, \(\WellM{{(\Fi{P}{i})}_{i \in \Index}} \leq_p \WellM{\bigcup_{i \in \Index}\Fi{P}{i}}\).
\end{theorem}

The next two results transfer \cref{thm:sequence:well-founded} to program simplification
by successively simplifying programs with the respective well-founded models of the previous programs.

\begin{theorem}\label{thm:sequence:simplification}
Let \({(\Fi{P}{i})}_{i \in \Index}\) be a sequence of \(\FormulasF\)-programs,
\((I,J)=\WellM{{(\Fi{P}{i})}_{i \in \Index}}\),
and \(\Ii{E}{i}\), \((\Iri{I}{i},\Iri{J}{i})\), and \((\Ii{I}{i},\Ii{J}{i})\) be defined as in \cref{eq:well-founded:sequence:external,eq:well-founded:sequence:relative,eq:well-founded:sequence:intermediate}.

Then, \(\Simp{\Fi{P}{k}}{I,J} \subseteq \Simp{\Fi{P}{k}}{(\Iri{I}{k},\Iri{J}{k}) \sqcup (\Ii{I}{k},\Ii{J}{k}) \sqcup (\emptyset,\Ii{E}{k})} \subseteq \Fi{P}{k} \) for all \(k\in\Index\).
\end{theorem}

\begin{corollary}\label{cor:sequence:simplification:stable}
Let \({(\Fi{P}{i})}_{i \in \Index}\) be a sequence of \(\FormulasR\)-programs,
and \(\Ii{E}{i}\), \((\Iri{I}{i},\Iri{J}{i})\), and \((\Ii{I}{i},\Ii{J}{i})\) be defined as in \cref{eq:well-founded:sequence:external,eq:well-founded:sequence:relative,eq:well-founded:sequence:intermediate}.

Then, \(\bigcup_{i \in \Index}{\Fi{P}{i}}\) and \(\bigcup_{i \in \Index}\Simp{\Fi{P}{i}}{(\Iri{I}{i},\Iri{J}{i}) \sqcup
(\Ii{I}{i},\Ii{J}{i}) \sqcup (\emptyset,\Ii{E}{i})}\) have the same well-founded and stable models.
\end{corollary}

Let us mention that the previous result extends to sequences of \(\FormulasF\)-programs and their well-founded models
but not their stable models.

\begin{example}\label{ex:sequence:a}
To illustrate \cref{thm:sequence:well-founded},
let us consider the following programs, \(\Fi{P}{1}\) and \(\Fi{P}{2}\):
\begin{align*}
a & \leftarrow \Naf c \tag{\(\Fi{P}{1}\)} \\
b & \leftarrow \Naf d \tag{\(\Fi{P}{1}\)} \\
c & \leftarrow \Naf b \tag{\(\Fi{P}{2}\)} \\
d & \leftarrow e      \tag{\(\Fi{P}{2}\)}
\end{align*}

The well-founded model of \(\Fi{P}{1} \cup \Fi{P}{2}\) is
\[(I,J) = (\{a,b\},\{a,b\}).\]

Let us evaluate \(\Fi{P}{1}\) before \(\Fi{P}{2}\).
While no head literals of \(\Fi{P}{2}\) occur positively in \(\Fi{P}{1}\),
the head literals \(c\) and \(d\) of \(\Fi{P}{2}\) occur negatively in rule bodies of \(\Fi{P}{1}\).
Hence, we get \(\Ii{E}{1} = \{c,d\}\)
and treat both atoms as unknown while calculating the well-founded model of \(\Fi{P}{1}\) relative to \((\emptyset, \{c,d\})\):
\[(\Ii{I}{1},\Ii{J}{1}) = (\emptyset,\{a,b\}).\]

We obtain that both \(a\) and \(b\) are unknown.
With this and \(\Ii{E}{2} = \emptyset\), we can calculate the well-founded model of \(\Fi{P}{2}\) relative to \((\Ii{I}{1},\Ii{J}{1})\):
\[(\Ii{I}{2},\Ii{J}{2}) = (\emptyset,\{c\}).\]
We see that because \(a\) is unknown, we have to derive \(c\) as unknown, too.
And because there is no rule defining \(e\), we cannot derive \(d\).
Hence,
\((\Ii{I}{1},\Ii{J}{1}) \sqcup (\Ii{I}{2},\Ii{J}{2})\) is less precise than \((I,J)\) because, when evaluating \(\Fi{P}{1}\),
it is not yet known that \(c\) is true and \(d\) is false.

Next, we illustrate the simplified programs according to \cref{thm:sequence:simplification}:
\begin{align*}
a & \leftarrow \Naf c & a & \leftarrow \Naf c \tag{\(\Fi{P}{1}\)}\\
b & \leftarrow \Naf d & b & \leftarrow \Naf d \tag{\(\Fi{P}{1}\)}\\
  &                   & c & \leftarrow \Naf b \tag{\(\Fi{P}{2}\)}
\end{align*}
The left column contains the simplification of \(\Fi{P}{1} \cup \Fi{P}{2}\) w.r.t. \((I,J)\) and the right column the simplification of
\(\Fi{P}{1}\) w.r.t.\ \((\Ii{I}{1},\Ii{J}{1})\) and
\(\Fi{P}{2}\) w.r.t.\ \((\Ii{I}{1},\Ii{J}{1}) \sqcup (\Ii{I}{2},\Ii{J}{2})\).
Note that \(d \leftarrow e\) has been removed in both columns because \(e\) is false in both \((I,J)\) and \((\Ii{I}{1},\Ii{J}{1}) \sqcup (\Ii{I}{2},\Ii{J}{2})\).
But we can only remove \(c \leftarrow \Naf b\) from the left column because, while \(b\) is false in \((I,J)\), it is unknown in \((\Ii{I}{1},\Ii{J}{1}) \sqcup (\Ii{I}{2},\Ii{J}{2})\).

Finally, observe that in accordance with \cref{thm:stable:simplification,cor:stable:simplification,cor:sequence:simplification:stable}, the program \(\Fi{P}{1} \cup \Fi{P}{2}\) and the two simplified programs have the same stable and well-founded models.
\end{example}

Clearly, the best simplifications are obtained when simplifying with the actual well-founded model of the overall program.
This can be achieved for a sequence as well whenever \(\Ii{E}{i}\) is empty, that is, if there is no need to approximate the impact of upcoming atoms.

\begin{corollary}\label{cor:sequence:well-founded:stratified}
Let \({(\Fi{P}{i})}_{i \in \Index}\) be a sequence of \(\FormulasF\)-programs
and \(\Ii{E}{i}\) be defined as in \cref{eq:well-founded:sequence:external}.

If \(\Ii{E}{i} = \emptyset\) for all \(i \in \Index\)
then \(\WellM{{(\Fi{P}{i})}_{i \in \Index}} = \WellM{\bigcup_{i \in \Index}\Fi{P}{i}}\).
\end{corollary}

\begin{corollary}\label{cor:sequence:simplification:stratified}
Let \({(\Fi{P}{i})}_{i \in \Index}\) be a sequence of \(\FormulasF\)-programs,
\((I,J) = \WellM{{(\Fi{P}{i})}_{i \in \Index}}\),
and \(\Ii{E}{i}\), \((\Iri{I}{i},\Iri{J}{i})\), and \((\Ii{I}{i},\Ii{J}{i})\) be defined as in \cref{eq:well-founded:sequence:external,eq:well-founded:sequence:relative,eq:well-founded:sequence:intermediate}.

If \(\Ii{E}{i} = \emptyset\) for all \(i \in \Index\), then \(\Simp{\Fi{P}{k}}{I,J} = \Simp{\Fi{P}{k}}{(\Iri{I}{k},\Iri{J}{k}) \sqcup (\Ii{I}{k},\Ii{J}{k})}\) for all \(k\in\Index\).
\end{corollary}

\begin{example}\label{ex:sequence:b}
Next, let us illustrate \cref{cor:sequence:well-founded:stratified} on an example.
We take the same rules as in \cref{ex:sequence:a} but use a different sequence:
\begin{align*}
d & \leftarrow e      \tag{\(\Fi{P}{1}\)} \\
b & \leftarrow \Naf d \tag{\(\Fi{P}{1}\)} \\
c & \leftarrow \Naf b \tag{\(\Fi{P}{2}\)} \\
a & \leftarrow \Naf c \tag{\(\Fi{P}{2}\)}
\end{align*}

Observe that the head literals of \(\Fi{P}{2}\) do not occur in the bodies of \(\Fi{P}{1}\), i.e.,
\(\Ii{E}{1} = \IDPosNeg{\Body{\Fi{P}{1}}} \cap {\Head{\Fi{P}{2}}} = \emptyset\).
The well-founded model of \(\Fi{P}{1}\) is
\begin{align*}
(\Ii{I}{1},\Ii{J}{1}) & = (\{b\}, \{b\}).
\end{align*}
And the well-founded model of \(\Fi{P}{2}\) relative to \((\{b\}, \{b\})\) is
\begin{align*}
(\Ii{I}{2},\Ii{J}{2}) & = (\{a\}, \{a\}).
\end{align*}
Hence, the union of both models is identical to the well-founded model of \(\Fi{P}{1} \cup \Fi{P}{2}\).

Next, we investigate the simplified program according to \cref{cor:sequence:simplification:stratified}:
\begin{align*}
b & \leftarrow \Naf d \tag{\(\Fi{P}{1}\)}\\
a & \leftarrow \Naf c \tag{\(\Fi{P}{2}\)}
\end{align*}
As in \cref{ex:sequence:a}, we delete rule \(d \leftarrow e\) because \(e\) is false in \((\Ii{I}{1},\Ii{J}{1})\).
But this time, we can also remove rule \(c \leftarrow \Naf b\) because \(b\) is true in \((\Ii{I}{1},\Ii{J}{1}) \sqcup (\Ii{I}{2},\Ii{J}{2})\).
\end{example}

 \section{Aggregate programs}\label{sec:aggregates}

We now turn to programs with aggregates and, at the same time, to programs with variables.
That is, we now deal with finite nonground programs whose instantiation
may lead to infinite ground programs including infinitary subformulas.
This is made precise by~\citeN{haliya14a} and \citeN{gehakalisc15a} where
aggregate programs are associated with infinitary propositional formulas~\cite{truszczynski12a}.
However,
the primary goal of grounding is to produce a finite set of ground rules with finitary subformulas only.
In fact, the program simplification introduced in \cref{sec:simplification} allows us to produce
an equivalent finite ground program whenever the well-founded model is finite.
The source of infinitary subformulas lies in the instantiation of aggregates.
We address this below by introducing an aggregate translation bound by an interpretation that produces finitary formulas
whenever this interpretation is finite.
Together, our concepts of program simplification and aggregate translation provide the backbone for
turning programs with aggregates into semantically equivalent finite programs with finitary subformulas.

Our concepts follow the ones of~\citeN{gehakalisc15a}; the semantics of aggregates is aligned with that of~\citeN{ferraris11a}
yet lifted to infinitary formulas~\cite{truszczynski12a,haliya14a}.

We consider a \emph{signature} \(\Signature=(\SignatureF,\SignatureP,\SignatureV)\) consisting of sets of function, predicate, and variable symbols.
The sets of variable and function symbols are disjoint.
\emph{Function} and \emph{predicate symbols} are associated with non-negative arities.
For short, a predicate symbol \(p\) of arity \(n\) is also written as \(p/n\).
In the following, we use lower case strings for function and predicate symbols,
and upper case strings for variable symbols.
Also, we often drop the term `symbol' and simply speak of functions, predicates, and variables.

As usual, \emph{terms} over \Signature\ are defined inductively as follows:
\begin{itemize}
\item \(v \in \SignatureV\) is a term and
\item \(f(t_1,\dots,t_n)\) is a term if \(f \in \SignatureF\) is a function symbol of arity~\(n\) and each \(t_i\) is a term over \Signature.
\end{itemize}
Parentheses for terms over function symbols of arity~0 are omitted.

Unless stated otherwise,
we assume that the set of (zero-ary) functions includes a set of numeral symbols being in a one-to-one correspondence to
the integers.
For simplicity, we drop this distinction and identify numerals with the respective integers.

An \emph{atom} over signature~\Signature\ has form \(p(t_1,\dots,t_n)\) where
\(p \in \SignatureP\) is a predicate symbol of arity~\(n\) and each \(t_i\) is a term over \Signature.
As above, parentheses for atoms over predicate symbols of arity~0 are omitted.
Given an atom~\(a\) over~\Signature, a \emph{literal} over \Signature\ is either the atom itself or its negation \(\Naf a\).
A literal without negation is called \emph{positive}, and \emph{negative} otherwise.

A \emph{comparison} over \(\Signature\) has form
\begin{align}
t_1 \prec t_2\label{eq:comparison}
\end{align}
where \(t_1\) and \(t_2\) are terms over \(\Signature\) and
\(\prec\) is a relation symbol among \(<\), \(\leq\), \(>\), \(\geq\), \(=\), and \(\neq\).

An \emph{aggregate element} over \(\Signature\) has form
\begin{align}
t_1, \dots, t_m & : a_1 \wedge \dots \wedge a_n\label{eq:aggregate:element}
\end{align}
where \(t_i\) is a term and \(a_j\) is an atom, both over \(\Signature\) for \(0 \leq i \leq m\) and \(0 \leq j\leq n\).
The terms \(t_1, \dots, t_m\) are seen as a tuple, which is empty for \(m=0\);
the conjunction \(a_1 \wedge \dots \wedge a_n\) is called the \emph{condition} of the aggregate element.
For an aggregate element~\(e\) of form~\eqref{eq:aggregate:element},
we use \(\Tuple{e} = (t_1, \dots, t_m)\) and \(\Cond{e} = \{ a_1, \dots, a_n\}\).
We extend both to sets of aggregate elements in the straightforward way, that is,
\(\Tuple{E} = \{\Tuple{e}\mid e\in E\}\) and \(\Cond{E} = \{\Cond{e}\mid e\in E\}\).

An \emph{aggregate atom} over \(\Signature\) has form

\begin{align}
& f \{ e_1, \dots, e_n \} \prec s\label{eq:aggregate}
\end{align}
where \(n \geq 0\),
\(f\) is an aggregate name among \(\Count\), \(\Sum\), \(\SumP\), and \(\SumM\),
each \(e_i\) is an aggregate element,
\(\prec\) is a relation symbol among \(<\), \(\leq\), \(>\), \(\geq\), \(=\), and \(\neq\) (as above),
and \(s\) is a term representing the aggregate's \emph{bound}.

Without loss of generality,
we refrain from introducing negated aggregate atoms.\footnote{Grounders like \lparse\ and \gringo\ replace aggregates systematically by auxiliary atoms and
  place them in the body of new rules implying the respective auxiliary atom.
  This results in programs without occurrences of negated aggregates.}
We often refer to aggregate atoms simply as aggregates.

An \emph{aggregate program} over \(\Signature\) is a finite set of \emph{aggregate rules} of form
\begin{align*}
h & \leftarrow b_1 \wedge \dots \wedge b_n
\end{align*}
where \(n \geq 0\),
\(h\) is an atom over \(\Signature\) and
each \(b_i\) is either a literal, a comparison, or an aggregate over \(\Signature\).
We refer to \(b_1, \dots, b_n\) as body literals,
and extend functions \(\Head{r}\) and \(\Body{r}\) to any aggregate rule~$r$.

\begin{example}\label{ex:company}
An example for an aggregate program is shown below,
giving an encoding of the \emph{Company Controls Problem}~\cite{mupira90a}:
A company \(X\) controls a company \(Y\)
if \(X\) directly or indirectly controls more than 50\% of the shares of \(Y\).
\begin{align*}
\Pred{controls}(X,Y) \leftarrow {}
& \!\begin{aligned}[t]
\Sum^+ \{
& S : \Pred{owns}(X,Y,S); \\
& S,Z: \Pred{controls}(X,Z) \wedge \Pred{owns}(Z,Y,S) \} > 50
\end{aligned} \\
{} \wedge {}
& \Pred{company}(X) \wedge \Pred{company}(Y) \wedge X \neq Y
\end{align*}
The aggregate \(\Sum^+\) implements summation over positive integers.
Notably, it takes part in the recursive definition of predicate \Pred{controls}.
In the following, we use an instance with ownership relations between four companies:
\begin{align*}
\Pred{company}(c_1) & & \Pred{company}(c_2) & & \Pred{company}(c_3) & & \Pred{company}(c_4) & \\
\Pred{owns}(c_1,c_2,60) & & \Pred{owns}(c_1,c_3,20) & & \Pred{owns}(c_2,c_3,35) & & \Pred{owns}(c_3,c_4,51) &
\end{align*}
\end{example}

We say that an aggregate rule \(r\) is \emph{normal} if its body does not contain aggregates.
An aggregate program is normal if all its rules are normal.

A term, literal, aggregate element, aggregate, rule, or program is \emph{ground}
whenever it does not contain any variables.

We assume that all ground terms are totally ordered by a relation~\(\leq\),
which is used to define the relations \(<\), \(>\), \(\geq\), \(=\), and \(\neq\) in the standard way.
For ground terms \(t_1,t_2\) and a corresponding relation symbol \(\prec\),
we say that \(\prec\) holds between \(t_1\) and \(t_2\) whenever the corresponding relation holds between \(t_1\) and \(t_2\).
Furthermore, \(>\), \(\geq\), and \(\neq\) hold between \(\infty\) and any other term,
and \(<\), \(\leq\), and \(\neq\) hold between \(-\infty\) and any other term.
Finally, we require that integers are ordered as usual.

For defining sum-based aggregates, we define for a tuple \(t = t_1, \dots, t_m\) of ground terms the following weight functions:
\begin{align*}
\Weight{t} & =
\begin{cases}
t_1 & \text{if \(m > 0\) and \(t_1\) is an integer}\\
0 & \text{otherwise,}\\
\end{cases}\\
\WeightP{t} & = \max\{\Weight{t}, 0\} \text{, and} \\
\WeightM{t} & = \min\{\Weight{t}, 0\}.
\end{align*}

With this at hand,
we now define how to apply aggregate functions to sets of tuples of ground terms
in analogy to~\citeN{gehakalisc15a}.

\begin{definition}
Let \(T\) be a set of tuples of ground terms.

We define
\begin{align*}
\Count(T) & =
\begin{cases}
|T| & \text{if \(T\) is finite}, \\
\infty & \text{otherwise},
\end{cases}\\
\Sum(T) & =
\begin{cases}
\Sigma_{t \in T}\Weight{t} & \text{if \(\{ t \in T \mid \Weight{t} \neq 0\}\) is finite,} \\
0 & \text{otherwise},
\end{cases}\\
\SumP(T) & =
\begin{cases}
\Sigma_{t \in T}\WeightP{t} & \text{if \(\{ t \in T \mid \Weight{t} > 0\}\) is finite,} \\
\infty & \text{otherwise, and}
\end{cases}\\
\SumM(T) & =
\begin{cases}
\Sigma_{t \in T}\WeightM{t} & \text{if \(\{ t \in T \mid \Weight{t} < 0\}\) is finite,} \\
-\infty & \text{otherwise.}
\end{cases}
\end{align*}
\end{definition}
Note that in our setting the application of aggregate functions to infinite sets of ground terms is of
theoretical relevance only,
since we aim at reducing them to their finite equivalents so that they can be evaluated by a grounder.

A variable is \emph{global} in
\begin{itemize}
\item a literal if it occurs in the literal,
\item a comparison if it occurs in the comparison,
\item an aggregate if it occurs in its bound, and
\item a rule if it is global in its head atom or in one of its body literals.
\end{itemize}

For example, the variables \(X\) and \(Y\) are global in the aggregate rule in \cref{ex:company},
while \(Z\) and \(S\) are neither global in the rule nor the aggregate.

\begin{definition}\label{def:safety}
Let \(r\) be an aggregate rule.

We define \(r\) to be \emph{safe}
\begin{itemize}
\item if all its global variables occur in some positive literal in the body of \(r\) and
\item if all its non-global variables occurring in an aggregate element \(e\) of an aggregate in the body of \(r\), also
occur in some positive literal in the condition of \(e\).
\end{itemize}
\end{definition}

For instance, the aggregate rule in \cref{ex:company} is safe.

Note that comparisons are disregarded in the definition of safety.
That is, variables in comparisons have to occur in positive body literals.\footnote{In fact, \gringo\ allows for variables in some comparisons to guarantee safety, as detailed in \cref{sec:gringo}.}

An aggregate program is safe if all its rules are safe.

An \emph{instance} of an aggregate rule \(r\) is obtained by substituting ground terms for all its global variables.
We use~\Grd{r} to denote the set of all instances of~\(r\)
and~\Grd{P} to denote the set of all ground instances of rules in aggregate program \(P\).
An \emph{instance} of an aggregate element \(e\) is obtained by substituting ground terms for all its variables.
We let \Grd{E} stand for all instances of aggregate elements in a set \(E\).
Note that \Grd{E} consists of ground expressions, which is not necessarily the case for \Grd{r}.
As seen from the first example in the introductory section, both \Grd{r} and \Grd{E} can be infinite.

A literal, aggregate element, aggregate, or rule is \emph{closed} if it does not contain any global variables.

For example,
the following rule is an instance of the aggregate rule in \cref{ex:company}.
\begin{align*}
\Pred{controls}(c_1,c_2) \leftarrow {}
& \!\begin{aligned}[t]
\SumP \{
& S : \Pred{owns}(c_1,c_2,S); \\
& S,Z: \Pred{controls}(c_1,Z), \Pred{owns}(Z,c_2,S) \} > 50
\end{aligned} \\
{} \wedge {}
& \Pred{company}(c_1) \wedge \Pred{company}(c_2) \wedge c_1 \neq c_2
\end{align*}
Note that both the rule and its aggregate are closed.
It is also noteworthy to realize that the two elements of the aggregate induce an infinite set of instances,
among them
\begin{align*}
20     &: \Pred{owns}(c_1,c_2,20)  & \text{and}\\
35,c_2 &: \Pred{controls}(c_1,c_2), \Pred{owns}(c_2,c_3,35).
\end{align*}

We now turn to the semantics of aggregates as introduced by~\citeN{ferraris11a}
but follow its adaptation to closed aggregates by~\citeN{gehakalisc15a}:
Let \(a\) be a closed aggregate of form~\eqref{eq:aggregate} and \(E\) be its set of aggregate elements.
We say that a set \(\ElemJ \subseteq \Grd{E}\) of its elements' instances \emph{justifies} \(a\), written \(\ElemJ \justifies a\),
if \(f(\Head{\ElemJ}) \prec s\) holds.

An aggregate \(a\) is \emph{monotone} whenever
\(\ElemJ_1\justifies a\) implies \(\ElemJ_2\justifies a\)
for all \(\ElemJ_1 \subseteq \ElemJ_2 \subseteq \Grd{E}\),
and accordingly \(a\) is \emph{antimonotone} if
\(\ElemJ_2\justifies a\) implies \(\ElemJ_1\justifies a\)
for all \(\ElemJ_1 \subseteq \ElemJ_2 \subseteq \Grd{E}\).

We observe the following monotonicity properties.

\begin{proposition}[\protect\citeNP{haliya14a}]\label{prp:aggregate:monotonicity}\leavevmode
\begin{itemize}
\item Aggregates over functions \(\SumP\) and \(\Count\) together with relations \(>\) and \(\geq\) are monotone.
\item Aggregates over functions \(\SumP\) and \(\Count\) together with relations \(<\) and \(\leq\) are antimonotone.
\item Aggregates over function \(\SumM\) have the same monotonicity properties as \(\SumP\) aggregates with the complementary relation.
\end{itemize}
\end{proposition}

Next,
we give the translation \(\TransSym\) from aggregate programs to \(\FormulasR\)-programs,
derived from the ones of~\citeN{ferraris11a} and \citeN{haliya14a}:

For a closed literal \(l\), we have
\begin{align}
\Trans{l} & = l\text{,}\nonumber\\
\intertext{for a closed comparison \(l\) of form~\eqref{eq:comparison}, we have}
\Trans{l} & = \begin{cases}
\top & \text{if \(\prec\) holds between \(t_1\) and \(t_2\)}\\
\bot & \text{otherwise}
\end{cases}\nonumber\\
\intertext{and for a set \(L\) of closed literals, comparisons and aggregates, we have}
\Trans{L} &= {\{ \Trans{l} \mid l \in L \}}\text{.}\nonumber
\intertext{For a closed aggregate \(a\) of form~\eqref{eq:aggregate} and its set \(E\) of aggregate elements, we have}
\Trans{a} & = {\{\Trans{\ElemJ}^\wedge \rightarrow \TransD{a}{\ElemJ}^\vee \mid \ElemJ \subseteq \Grd{E}, \ElemJ \notjustifies a \}}^\wedge\label{eq:aggregate:translation}
\intertext{where}
\TransD{a}{\ElemJ} & = \Trans{\Grd{E} \setminus \ElemJ}\text{\ for \(\ElemJ \subseteq \Grd{E}\),}\nonumber\\
\Trans{\ElemJ} & = {\{\Trans{e} \mid e \in \ElemJ\}}\text{\ for \(\ElemJ \subseteq \Grd{E}\), and}\nonumber\\
\Trans{e} & = {\Trans{\Body{e}}}^\wedge \text{\ for \(e \in \Grd{E}\)}.\nonumber
\intertext{For a closed aggregate rule \(r\), we have}
\Trans{r} & = \Trans{\Head{r}} \leftarrow \Trans{\Body{r}}^\wedge.\nonumber
\intertext{For an aggregate program~\(P\), we have}
\Trans{P} & = \{\Trans{r} \mid r \in \Grd{P} \}.
\end{align}

While aggregate programs like \(P\) are finite sets of (non-ground) rules,
\(\Trans{P}\) can be infinite and contain (ground) infinitary expressions.
Observe that \(\Trans{P}\) is an \(\FormulasR\)-program.
In fact, only the translation of aggregates introduces \(\FormulasR\)-formulas;
rules without aggregates form \(\FormulasN\)-programs.

\begin{example}\label{ex:translation}
To illustrate Ferraris' approach to the semantics of aggregates, consider a count aggregate \(a\) of form
\begin{align*}
\CountA{\Elem{X}{p(X)}}{\geq}{n}.
\end{align*}
Since the aggregate is non-ground, the set \(\ElemG\) of its element's instances consists of all \(\Elem{t}{p(t)}\) for each ground term \(t\).

The count aggregate cannot be justified by any subset \(\ElemJ\) of \(\ElemG\)
satisfying
\(
|\{t \mid \Elem{t}{p(t)} \in \ElemJ\}| < n
\),
or \(\ElemJ \notjustifies a\) for short.
Accordingly,
we have that \Trans{a} is the conjunction of all formulas
\begin{align}\label{eq:translation}
{\{ p(t) \mid \Elem{t}{p(t)} \in \ElemJ \}}^\wedge
& \rightarrow {\{ p(t) \mid \Elem{t}{p(t)} \in (\ElemG \setminus \ElemJ) \}}^\vee
\end{align}
such that \(\ElemJ \subseteq \ElemG\) and \(\ElemJ \notjustifies a\).
Restricting the set of ground terms to the numerals \(1,2,3\) and letting \(n=2\) results in the formulas
\begin{align*}
\top & \rightarrow p(1) \vee p(2) \vee p(3),\\
p(1) & \rightarrow p(2) \vee p(3),\\
p(2) & \rightarrow p(1) \vee p(3), \AndT\\
p(3) & \rightarrow p(1) \vee p(2).
\end{align*}
Note that a smaller number of ground terms than \(n\) yields an unsatisfiable set of formulas.
\end{example}

However, it turns out that a Ferraris-style translation of aggregates~\cite{ferraris11a,haliya14a}
is too weak for propagating monotone aggregates in our \FOID-based setting.
That is,
when propagating possible atoms (i.e., the second component of the well-founded model),
an \FOID-reduct may become satisfiable although the original formula is not.
So, we might end up with too many possible atoms and a well-founded model that is not as precise as it could be.
To see this, consider the following example.

\begin{example}\label{ex:translation:reduct}
For some \(m,n \geq 0\), the program \(\Fi{P}{m,n}\) consists of the following rules:
\begin{align*}
p(i) & \leftarrow \Naf q(i) & & \ForT 1 \leq i \leq m\\
q(i) & \leftarrow \Naf p(i) & & \ForT 1 \leq i \leq m\\
r    & \leftarrow \CountA{\Elem{X}{p(X)}}{\geq}{n}
\end{align*}

Given the ground instances \(\ElemG\) of the aggregate's elements and some two-valued interpretation \(I\),
observe that
\begin{align}
& \IDReductP{\Trans{\CountA{\Elem{X}{p(X)}}{\geq}{n}}}{I}\nonumber
\intertext{is classically equivalent to}
& \IDReductP{\Trans{\CountA{\Elem{X}{p(X)}}{\geq}{n}}}{I} \vee {\{ p(t) \in \Body{\ElemG} \mid p(t) \notin I \}}^\vee.\label{eq:count:translation}
\end{align}
To see this, observe that the formula obtained via \TransSym\ for the aggregate in the last rule's body consists of
positive occurrences of implications of the form \(\mathcal{G}^\wedge \rightarrow \mathcal{H}^\vee\)
where either \(p(t) \in \mathcal{G}\) or \(p(t) \in \mathcal{H}\).
The \FOID-reduct makes all such implications with some \(p(t) \in \mathcal{G}\) such that \(p(t) \notin I\) true
because their antecedent is false.
All of the remaining implications in the \FOID-reduct are equivalent to \(\mathcal{H}^\vee\)
where \(\mathcal{H}\) contains all \(p(t) \notin I\).
Thus, we can factor out the formula on the right-hand side of~\eqref{eq:count:translation}.

Next, observe that for \(1 \leq m < n\),
the four-valued interpretation \((I,J) = (\emptyset, \Head{\Trans{\Fi{P}{m,n}}})\) is the well-founded model of \(\Fi{P}{m,n}\):
\begin{align*}
\Stable{\Trans{\Fi{P}{m,n}}}{J} & = I\text{\ and}\\
\Stable{\Trans{\Fi{P}{m,n}}}{I} & = J.
\end{align*}
Ideally, atom \(r\) should not be among the possible atoms because it can never be in a stable model.
Nonetheless, it is due to the second disjunct in~\eqref{eq:count:translation}.
\end{example}

Note that not just monotone aggregates exhibit this problem.
In general, we get for a closed aggregate \(a\) with elements \(E\) and an interpretation \(I\) that
\begin{align*}
\IDReductP{\Trans{a}}{I} \text{\ is classically equivalent to\ } \IDReductP{\Trans{a}}{I} \vee {\{c \in \Cond{\Grd{E}} \mid I \not\models c\}}^\vee.
\end{align*}
The second disjunct is undesirable when propagating possible atoms.

To address this shortcoming,
we augment the aggregate translation so that it provides stronger propagation.
The result of the augmented translation is strongly equivalent to that of the original translation
(cf.~\cref{prp:translation:equivalence}).
Thus, even though we get more precise well-founded models, the stable models are still contained in them.

\begin{definition}\label{def:translation:aggregate}
We define \TransASym\ as the translation obtained from \TransSym\ by replacing
the case of closed aggregates in~\eqref{eq:aggregate:translation} by the following:

For a closed aggregate \(a\) of form~\eqref{eq:aggregate} and its set \(E\) of aggregate elements, we have
\begin{align*}
\TransA{a} & = {\{ {\Trans{\ElemJ}}^\wedge \rightarrow \TransAD{a}{\ElemJ}^\vee \mid \ElemJ \subseteq \Grd{E}, \ElemJ \notjustifies a \}}^\wedge\\
\intertext{where}
\TransAD{a}{\ElemJ} & = \{ {\Trans{\ElemC}}^\wedge \mid \ElemC \subseteq \Grd{E} \setminus \ElemJ, \ElemC \cup \ElemJ \justifies a \}
\ForT \ElemJ \subseteq \Grd{E}.
\end{align*}
\end{definition}

Note that just as \TransSym\ also \TransASym\ is recursively applied to the whole program.

Let us illustrate the modified translation by revisiting \cref{ex:translation}.

\begin{example}\label{ex:translation:aggregate}
Let us reconsider the count aggregate \(a\):
\begin{align*}
\CountA{\Elem{X}{p(X)}}{\geq}{n}
\end{align*}
As with \Trans{a} in \cref{ex:translation},
\TransA{a} yields a conjunction of formulas,
one conjunct for each set \(\ElemJ\subseteq \Grd{E}\) satisfying \(\ElemJ\notjustifies a\) of the form:
\begin{align}\label{ex:aggregate:kaminski}
{\{ \Body{e} \mid e \in \ElemJ\}}^\wedge
\rightarrow
{\big\{ {\{ \Body{e} \mid e \in (\ElemC \setminus \ElemJ) \}}^\wedge \mid \ElemC \justifies a, \ElemJ \subseteq \ElemC \subseteq \Grd{E} \big\}}^\vee
\end{align}
Restricting again the set of ground terms to the numerals \(1,2,3\) and letting \(n=2\) results now in the formulas
\begin{align*}
\top & \rightarrow (p(1) \wedge p(2)) \vee (p(1) \wedge p(3)) \vee (p(2)\wedge p(3))\vee(p(1)\wedge p(2)\wedge p(3)),\\
p(1) & \rightarrow p(2) \vee p(3) \vee (p(2) \wedge p(3)),\\
p(2) & \rightarrow p(1) \vee p(3) \vee (p(1) \wedge p(3)), \AndT\\
p(3) & \rightarrow p(1) \vee p(2) \vee (p(1) \wedge p(2)).
\end{align*}
Note that the last disjunct can be dropped from each rule's consequent.
And as above, a smaller number of ground terms than \(n\) yields an unsatisfiable set of formulas.
\end{example}

The next result ensures that \(\Trans{P}\) and \(\TransA{P}\) have the same stable models for any aggregate program $P$.

\begin{proposition}\label{prp:translation:equivalence}
Let \(a\) be a closed aggregate.

Then, \(\Trans{a}\) and \(\TransA{a}\) are strongly equivalent.
\end{proposition}

The next example illustrates that we get more precise well-founded models using the strongly equivalent refined translation.

\begin{example}\label{ex:translation:continued}
Reconsider Program~\(\Fi{P}{m,n}\) from \cref{ex:translation:reduct}.

As above, we apply the well-founded operator to program \(\Fi{P}{m,n}\) for \(m < n\) and four-valued interpretation \((I,J) = (\emptyset, \Head{\TransA{\Fi{P}{m,n}}})\):
\begin{align*}
\Stable{\TransA{\Fi{P}{m,n}}}{J} & = I \AndT\\
\Stable{\TransA{\Fi{P}{m,n}}}{I} & = J \setminus \{ r \}.
\end{align*}
Unlike before,
\(r\) is now found to be false since it does not belong to \(\Stable{\TransA{\Fi{P}{m,n}}}{\emptyset}\).
\end{example}

To see this,
we can take advantage of the following proposition.

\begin{proposition}\label{prp:translation:monotonicity}
Let \(a\) be a closed aggregate.

If \(a\) is monotone, then
\(\IDReductP{\TransA{a}}{I}\) is classically equivalent to \(\TransA{a}\)
for any two-valued interpretation~\(I\).
\end{proposition}

Note that \(\TransA{a}\) is a negative formula whenever \(a\) is antimonotone; cf.\ \cref{prp:translation:antimonotonicity}.

Let us briefly return to \cref{ex:translation:continued}.
We now observe that \(\IDReductP{\TransA{a}}{I}=\TransA{a}\)
for \(a=\CountA{\Elem{X}{p(X)}} \geq n\) and any interpretation \(I\)
in view of the last proposition.
Hence, our refined translation \(\TransASym\) avoids the problematic disjunct in~\cref{eq:count:translation} on the right.
By \cref{prp:translation:equivalence}, we can use \(\TransA{\Fi{P}{m,n}}\) instead of \(\Trans{\Fi{P}{m,n}}\);
both formulas have the same stable models.

Using \cref{prp:translation:monotonicity},
we augment the translation \(\TransASym\) to replace monotone aggregates~\(a\) by the strictly positive formula~\(\IDReductP{\TransA{a}}{\emptyset}\).
That is, we only keep the implication with the trivially true antecedent in the aggregate translation
(cf.~\cref{sec:dependency}).

While \(\TransASym\) improves on propagation,
it may still produce infinitary \(\FormulasR\)-formulas when applied to aggregates.
This issue is addressed by restricting the translation to a set of (possible) atoms.

\begin{definition}\label{def:translation:aggregate:restricted}
Let \(J\) be a two-valued interpretation.
We define the translation \TransRSym{J} as the one obtained from \TransSym\ by replacing
the case of closed aggregates in~\eqref{eq:aggregate:translation} by the following:

For a closed aggregate \(a\) of form~\eqref{eq:aggregate} and its set \(E\) of aggregate elements, we have
\begin{align*}
\TransR{a}{J} & = {\{ \Trans{\ElemJ}^\wedge \rightarrow \TransR{\ElemJ}{a,J}^\vee \mid \ElemJ \subseteq \Restrict{\Grd{E}}{J}, \ElemJ \notjustifies a \}}^\wedge
\intertext{where}
\TransR{\ElemJ}{a,J}  & = {\{ \Trans{\ElemC}^\wedge \mid \ElemC \subseteq \Restrict{\Grd{E}}{J} \setminus \ElemJ, \ElemC \cup \ElemJ \justifies a \}}\AndT\\
\Restrict{\Grd{E}}{J} & = \{ e \in \Grd{E} \mid \Body{e} \subseteq J \}.
\end{align*}
\end{definition}

Note that \(\TransAD{J}{a}\) is a finitary formula whenever \(J\) is finite.

Clearly,
\(\TransRSym{J}\) also conforms to \(\TransASym\)
except for the restricted translation for aggregates defined above.
The next proposition elaborates this by showing that \(\TransRSym{J}\) and \(\TransASym\) behave alike whenever \(J\)
limits the set of possible atoms.

\begin{theorem}\label{prp:translation:restricted:properties}
Let \(a\) be a closed aggregate, and \(I \subseteq J\) and \(X \subseteq J\) be two-valued interpretations.

Then,
\begin{alphaenum}
\item \(X \models \TransA{a}\) iff \(X \models \TransR{a}{J}\),\label{prp:translation:restricted:properties:a}
\item \(X \models \IDReductP{\TransA{a}}{I}\) iff \(X \models \IDReductP{\TransR{a}{J}}{I}\), and\label{prp:translation:restricted:properties:b}
\item \(X \models \Reduct{\TransA{a}}{I}\) iff \(X \models \Reduct{\TransR{a}{J}}{I}\).\label{prp:translation:restricted:properties:c}
\end{alphaenum}

\end{theorem}

In view of \cref{prp:translation:equivalence},
this result extends to Ferraris' original aggregate translation~\cite{ferraris11a,haliya14a}.

The next example illustrates how a finitary formula can be obtained for an aggregate,
despite a possibly infinite set of terms in the signature.

\begin{example}
Let \(\Fi{P}{m,n}\) be the program from \cref{ex:translation:reduct}.
The well-founded model \((I,J)\) of \(\TransA{\Fi{P}{m,n}}\) is \((\emptyset, \Head{\TransA{\Fi{P}{m,n}}})\) if \(n \leq m\).

The translation \(\TransAD{J}{\Fi{P}{3,2}}\) consists of the rules
\begin{align*}
p(1) & \leftarrow \Naf q(1), &
q(1) & \leftarrow \Naf p(1), \\
p(2) & \leftarrow \Naf q(2), &
q(2) & \leftarrow \Naf p(2), \\
p(3) & \leftarrow \Naf q(3), &
q(3) & \leftarrow \Naf p(3), \AndT\\
r & \omit\rlap{\({} \leftarrow \TransAD{J}{\CountA{\Elem{X}{p(X)}}{\geq}{2}}\)}
\end{align*}
where the aggregate translation corresponds to the conjunction of the formulas in \cref{ex:translation:aggregate}.
Note that the translation \(\TransA{\Fi{P}{3,2}}\) depends on the signature whereas the translation \(\TransAD{J}{\Fi{P}{3,2}}\) is fixed by the atoms in \(J\).
\end{example}

Importantly,
\cref{prp:translation:restricted:properties} shows that given a finite (approximation of the) well-founded
model of an \(\FormulasR\)-program,
we can replace aggregates with finitary formulas.
Moreover, in this case, \cref{thm:stable:simplification} and \cref{prp:translation:equivalence} together indicate
how to turn a program with aggregates into a semantically equivalent finite \(\FormulasR\)-program with finitary formulas as bodies.
That is,
given a finite well-founded model of an \(\FormulasR\)-program,
the program simplification from \cref{def:simplification} results in a finite program and
the aggregate translation from \cref{def:translation:aggregate:restricted} produces finitary formulas only.

This puts us in a position to outline how and when (safe non-ground) aggregate programs can be turned into
equivalent finite ground programs consisting of finitary subformulas only.
To this end, consider an aggregate program \(P\) along with the well-founded model \((I,J)\) of \(\TransA{P}\).
We have already seen in \cref{cor:stable:simplification} that \(\TransA{P}\) and its simplification \(\Simp{\TransA{P}}{I,J}\)
have the same stable models,
just like \(\Simp{\TransA{P}}{I,J}\) and its counterpart \(\Simp{\TransAD{J}{P}}{I,J}\) in view of \cref{prp:translation:restricted:properties}.

Now, if \((I,J)\) is finite, then \(\Simp{\TransA{P}}{I,J}\) is finite, too.
Seen from the perspective of grounding,
the safety of all rules in \(P\) implies that all global variables appear in positive body literals.
Thus, the number of ground instances of each rule in \(\Simp{\TransA{P}}{I,J}\) is determined by
the number of possible substitutions for its global variables.
Clearly, there are only finitely many possible substitutions such that
all positive body literals are satisfied by a finite interpretation~\(J\)
(cf.\ \cref{def:simplification}).
Furthermore, if \(J\) is finite,
aggregate translations in \(\Simp{\TransAD{J}{P}}{I,J}\) introduce finitary subformulas only.
Thus, in this case, we obtain from \(P\) a finite set of rules with finitary propositional formulas as bodies, viz.~\(\Simp{\TransAD{J}{P}}{I,J}\), that has the same stable models as \(\TransA{P}\)
(as well as \(\Trans{P}\), the traditional Ferraris-style semantics of \(P\)~\cite{ferraris11a,haliya14a}).

An example of a class of aggregate programs inducing finite well-founded models as above
consists of programs over a signature with nullary function symbols only.
Any such program can be turned into an equivalent finite set of propositional rules with finitary bodies.

 \section{Dependency analysis}\label{sec:dependency}

We now further refine our semantic approach to reflect actual grounding processes.
In fact,
modern grounders process programs on-the-fly by grounding one rule after another without storing any rules.
At the same time, they try to determine certain, possible, and false atoms.
Unfortunately, well-founded models cannot be computed on-the-fly,
which is why we introduce below the concept of an approximate model.
More precisely,
we start by defining instantiation sequences of (non-ground) aggregate programs based on their rule dependencies.
We show that approximate models of instantiation sequences are underapproximations
of the well-founded model of the corresponding sequence of (ground) \FormulasR-programs, as defined in \cref{sec:splitting}.
The precision of both types of models coincides on stratified programs.
We illustrate our concepts comprehensively at the end of this section in \cref{ex:dependency,ex:dependency:company}.

To begin with,
we extend the notion of positive and negative literals to aggregate programs.
For atoms \(a\), we define \(a^+ = {(\Naf a)}^- = \{a\}\) and \(a^- = {(\Naf a)}^+ = \emptyset\).
For comparisons \(a\), we define \(a^+ = a^- = \emptyset\).
For aggregates \(a\) with elements \(E\),
we define positive and negative atom occurrences,
using \cref{prp:translation:monotonicity} to refine the case for monotone aggregates:
\begin{itemize}
\item \(a^+ = \bigcup_{e \in E} \Body{e}\),
\item \(a^- = \emptyset\) if \(a\) is monotone, and
\item \(a^- = \bigcup_{e \in E} \Body{e}\) if \(a\) is not monotone.
\end{itemize}
For a set of body literals \(B\), we define \(B^+ = \bigcup_{b \in B}b^+\) and \(B^- = \bigcup_{b \in B}b^-\),
as well as \(B^\pm=B^+\cup B^-\).

We see in the following, that a special treatment of monotone aggregates yields better approximations of well-founded models.
A similar case could be made for antimonotone aggregates but had led to a more involved algorithmic treatment.

Inter-rule dependencies are determined via the predicates appearing in their heads and bodies.
We define~\(\Predicate{a}\) to be the predicate symbol associated with atom~\(a\)
and~\(\Predicate{A} = \{ \Predicate{a} \mid a \in A \}\) for a set~\(A\) of atoms.
An aggregate rule \(r_1\) \emph{depends} on another aggregate rule \(r_2\)
if \(\Predicate{\Head{r_2}} \in \Predicate{{\Body{r_1}}^\pm}\).
Rule \(r_1\) depends \emph{positively} or \emph{negatively} on \(r_2\) if \(\Predicate{\Head{r_2}} \in \Predicate{\Body{r_1}^+}\) or \(\Predicate{\Head{r_2}} \in \Predicate{{\Body{r_2}}^-}\), respectively.

For simplicity,
we first focus on programs without aggregates in examples and delay a full example with aggregates until the end of the section.

\begin{example}\label{ex:aggregate-program:dependency}
Let us consider the following rules from the introductory example:
\begin{align*}
u(1) & \tag{\(r_1\)} \\
p(X) & \leftarrow \Naf q(X) \wedge u(X) \tag{\(r_2\)} \\
q(X) & \leftarrow \Naf p(X) \wedge v(X) \tag{\(r_3\)}
\end{align*}

We first determine the rule heads and positive and negative atom occurrences in rule bodies:
\begin{align*}
\Head{r_1}   & = u(1) &
\Body{r_1}^+ & = \emptyset &
\Body{r_1}^- & = \emptyset \\
\Head{r_2}   & = p(X) &
\Body{r_2}^+ & = \{ u(X) \} &
\Body{r_2}^- & = \{ q(X) \} \\
\Head{r_3}   & = q(X) &
\Body{r_3}^+ & = \{ v(X) \} &
\Body{r_3}^- & = \{ p(X) \}
\end{align*}

With this, we infer the corresponding predicates:
\begin{align*}
\Predicate{\Head{r_1}}   & = u/1 &
\Predicate{\Body{r_1}^+} & = \emptyset &
\Predicate{\Body{r_1}^-} & = \emptyset \\
\Predicate{\Head{r_2}}   & = p/1 &
\Predicate{\Body{r_2}^+} & = \{ u/1 \} &
\Predicate{\Body{r_2}^-} & = \{ q/1 \} \\
\Predicate{\Head{r_3}}   & = q/1 &
\Predicate{\Body{r_3}^+} & = \{ v/1 \} &
\Predicate{\Body{r_3}^-} & = \{ p/1 \}
\end{align*}

We see that \(r_2\) depends positively on \(r_1\) and that \(r_2\) and \(r_3\) depend negatively on each other.
View~\cref{fig:dependency} in~\cref{ex:aggregate-program:sequence} for a graphical representation of these inter-rule dependencies.
\end{example}

The \emph{strongly connected components} of an aggregate program \(P\) are the equivalence classes
under the transitive closure of the dependency relation between all rules in \(P\).
A strongly connected component \(\Fi{P}{1}\) \emph{depends} on another strongly connected component \(\Fi{P}{2}\) if there is a rule in \(\Fi{P}{1}\) that depends on some rule in \(\Fi{P}{2}\).
The transitive closure of this relation is antisymmetric.

A strongly connected component of an aggregate program is \emph{unstratified}
if it depends negatively on itself or
if it depends on an unstratified component.
A component is \emph{stratified} if it is not unstratified.

A topological ordering of the strongly connected components is then used to guide grounding.

For example,
the sets \(\{r_1\}\) and \(\{r_2, r_3\}\) of rules from \cref{ex:aggregate-program:dependency} are strongly connected components
in a topological order.
There is only one topological order because \(r_2\) depends on \(r_1\).
While the first component is stratified, the second component is unstratified
because \(r_2\) and \(r_3\) depend negatively on each other.

\begin{definition}
We define an \emph{instantiation sequence} for \(P\) as a sequence~\(\Fsqi{P}{\Index}{i}\)
of its strongly connected components such that \(i < j\) if \(\Fi{P}{j}\) depends on \(\Fi{P}{i}\).
\end{definition}

Note that the components can always be well ordered because aggregate programs consist of finitely many rules.

The consecutive construction of the well-founded model along an instantiation sequence results
in the well-founded model of the entire program.

\begin{theorem}\label{thm:instantiation-sequence:well-founded}
Let \(\Fsqi{P}{\Index}{i}\) be an instantiation sequence for aggregate program~\(P\).

Then, \(\WellM{\Fsq{\TransA{P_i}}{i\in\Index}} = \WellM{\TransA{P}}\).
\end{theorem}

\begin{example}\label{ex:aggregate-program:sequence}
The following example shows how to split an aggregate program into an instantiation sequence and
gives its well-founded model.
Let \(P\) be the following aggregate program,
extending the one from the introductory section:
\begin{align*}
u(1) & &
u(2) & \\
v(2) & &
v(3) & \\
p(X) & \leftarrow \Naf q(X) \wedge u(X) &
q(X) & \leftarrow \Naf p(X) \wedge v(X) \\
x & \leftarrow \Naf p(1) &
y & \leftarrow \Naf q(3)
\end{align*}

\begin{figure}
\centering
\includegraphics{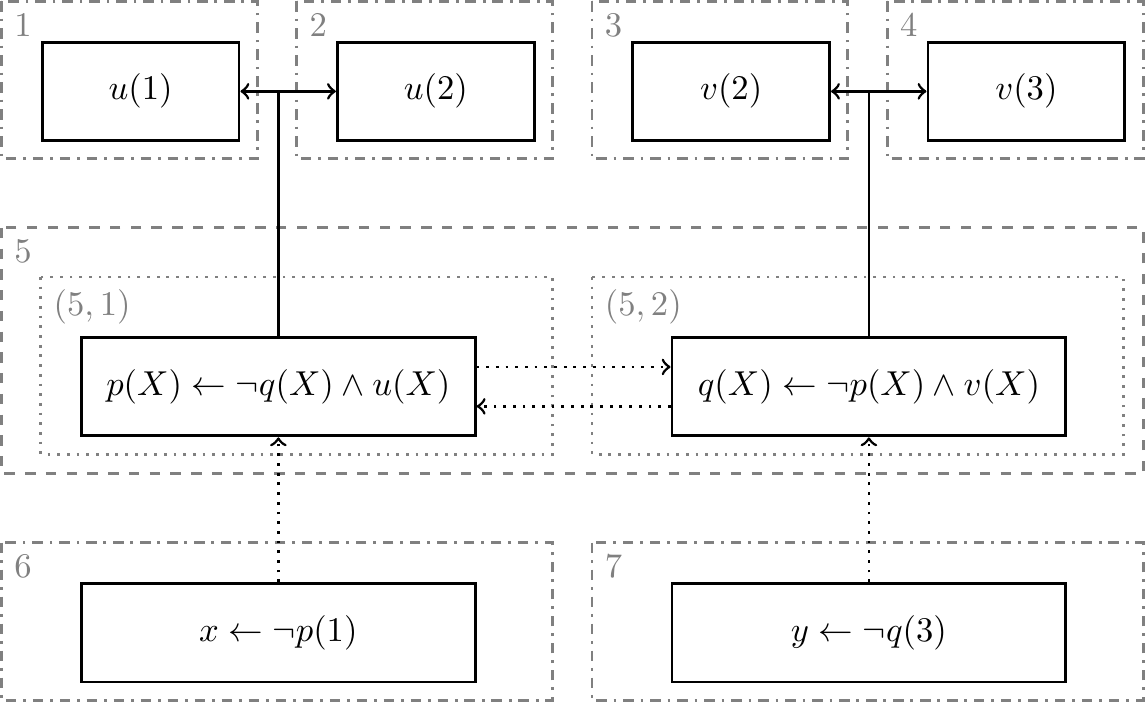}
\caption{Rule dependencies for \cref{ex:dependency}.\label{fig:dependency}}
\end{figure}

We have already seen how to determine inter-rule dependencies in \cref{ex:aggregate-program:dependency}.
A possible instantiation sequence for program \(P\) is given in \cref{fig:dependency}.
Rules are depicted in solid boxes.
Solid and dotted edges between such boxes depict positive and negative dependencies between the corresponding rules, respectively.
Dashed and dashed/dotted boxes represent components in the instantiation sequence (we ignore dotted boxes for now but turn to them in \cref{ex:aggregate-program:sequence:refined}).
The number in the corner of a component box indicates the index in the corresponding instantiation sequence.

For \(F = \{u(1),u(2),v(2),v(3)\}\),
the well-founded model of \(\TransA{P}\) is
\begin{align*}
\WellM{\TransA{P}} &= (\{p(1),q(3)\},\{p(1),p(2),q(2),q(3)\}) \sqcup F.
\end{align*}

By \cref{thm:instantiation-sequence:well-founded},
the ground sequence~\(\Fsqti{P}{\Index}{i}\) has the same well-founded model as \(\TransA{P}\):
\begin{align*}
\WellM{\Fsq{\TransA{P}}{i\in\Index}} &= (\{p(1),q(3)\},\{p(1),p(2),q(2),q(3)\}) \sqcup F.
\end{align*}

Note that the set \(F\) comprises the facts derived from stratified components.
In fact, for stratified components, the set of external atoms~\eqref{eq:well-founded:sequence:external} is empty.
We can use \cref{cor:sequence:well-founded:stratified} to confirm that
the well founded model \((F,F)\) of sequence \(\Fsq{\TransA{\Fi{P}{i}}}{1 \leq i \leq 4}\) is total.
In fact, each of the intermediate interpretations~\eqref{eq:well-founded:sequence:intermediate} is total
and can be computed with just one application of the stable operator.
For example, \(\Ii{I}{1}=\Ii{J}{1}=\{u(1)\}\) for component \(\Fi{P}{1}\).
\end{example}

We further refine instantiation sequences by partitioning each component along its positive dependencies.

\begin{definition}
Let \(P\) be an aggregate program and \(\Fsqi{P}{\Index}{i}\) be an instantiation sequence for \(P\).
Furthermore, for each \(i \in \Index\), let \(\Fsq{\Fi{P}{i,j}}{j \in \Index_i}\) be an instantiation sequence of \(\Fi{P}{i}\) considering positive dependencies only.

A \emph{refined instantiation sequence} for \(P\) is a sequence~\(\Fsqic{P}{\IndexAlt}{i,j}\) where
the index set~\(\IndexAlt = \{ (i,j) \mid i \in \Index, j \in \Index_i \}\) is ordered lexicographically.

We call \(\Fsqic{P}{\IndexAlt}{i,j}\) a refinement of \(\Fsqi{P}{\Index}{i}\).

We define a component \(\Fi{P}{i,j}\) to be \emph{stratified} or \emph{unstratified} if the encompassing component \(\Fi{P}{i}\) is stratified or unstratified, respectively.
\end{definition}

Examples of refined instantiation sequences are given in \cref{fig:dependency,fig:dependency:company}.

The advantage of such refinements is that they yield better or equal approximations
(cf.\ \cref{thm:approximate:sequence:well-founded,ex:dependency}).
On the downside, we do not obtain that \(\WellM{\Fsqi{P}{\IndexAlt}{i}}\) equals \(\WellM{\TransA{P}}\) for refined instantiation sequences in general.

\begin{example}\label{ex:aggregate-program:sequence:refined}
The refined instantiation sequence for program \(P\) from \cref{ex:aggregate-program:sequence} is given in \cref{fig:dependency}.
A dotted box indicates a component in a refined instantiation sequence.
Components that cannot be refined further are depicted with a dashed/dotted box.
The number or pair in the corner of a component box indicates the index in the corresponding refined instantiation sequence.

Unlike in \cref{ex:aggregate-program:sequence},
the well-founded model of the refined sequence of ground programs~\(\Fsqtic{P}{\IndexAlt}{i,j}\) is
\begin{align*}
\WellM{\Fsqtic{P}{\IndexAlt}{i,j}} &= (\{q(3)\},\{p(1),p(2),q(2),q(3),x\}) \sqcup F,
\end{align*}
which is actually less precise than the well-founded model of \(P\).
This is because literals over \(\Naf q(X)\) are unconditionally assumed to be true
because their instantiation is not yet available when \(\Fi{P}{5,1}\) is considered.
Thus, we get \((\Ii{I}{5,1},\Ii{J}{5,1}) = (\emptyset, \{ p(1), p(2) \})\) for the intermediate interpretation~\eqref{eq:well-founded:sequence:intermediate}.
Unlike this, the atom \(p(3)\) is false when considering component \(\Fi{P}{5,2}\) and \(q(3)\) becomes true.
In fact, we get \((\Ii{I}{5,2},\Ii{J}{5,2}) = (\{q(3)\}, \{ q(2), q(3) \})\).
Observe that \((\Ii{I}{5}, \Ii{J}{5})\) from above is less precise than \((\Ii{I}{5,1}, \Ii{J}{5,1}) \sqcup (\Ii{I}{5,2}, \Ii{J}{5,2})\).

We continue with this example below and show in~\cref{ex:dependency} that
refined instantiation sequences can still be advantageous to get better approximations of well-founded models.
\end{example}

We have already seen in \cref{sec:splitting} that external atoms may lead to less precise semantic characterizations.
This is just the same in the non-ground case, whenever a component comprises predicates that are defined in a following
component of a refined instantiation sequence.
This leads us to the concept of an approximate model obtained by overapproximating the extension of such externally defined predicates.

\begin{definition}\label{def:approximate}
Let \(P\) be an aggregate program,
\((\Ir{I},\Ir{J})\) be a four-valued interpretation,
\(\PredE\) be a set of predicates, and
\(P'\) be the program obtained from \(P\) by removing all rules \(r\) with \(\Predicate{\Body{r}^-} \cap \PredE \neq \emptyset\).

We define the \emph{approximate model of \(P\) relative} to \((\Ir{I},\Ir{J})\) as
\begin{align*}
\ApproxRM{P}{\PredE}{(\Ir{I},\Ir{J})} & = (I,J)\nonumber\\
\intertext{where}
I &= \StableR{\TransA{P'}}{\Ir{I}}{\Ir{J}}\AndT\\
J &= \StableR{\TransA{P}}{\Ir{J}}{\Ir{I} \cup I}
\end{align*}
\end{definition}

We keep dropping parentheses and simply write \(\ApproxRM{P}{\PredE}{\Ir{I},\Ir{J}}\) instead of \(\ApproxRM{P}{\PredE}{(\Ir{I},\Ir{J})}\).

The approximate model amounts to an immediate consequence operator,
similar to the relative well-founded operator in \cref{def:well-founded:operator:relative};
it refrains from any iterative applications, as used for defining a well-founded model.
More precisely,
the relative stable operator is applied twice to obtain the approximate model.
This is similar to Van~Gelder's alternating transformation~\cite{vangelder93a}.
The certain atoms in \(I\) are determined by applying the operator to
the ground program obtained after removing all rules whose negative body literals comprise externally defined predicates,
while the possible atoms \(J\) are computed from the entire program by taking the already computed certain atoms in \(I\) into account.
In this way, the approximate model may result in fewer unknown atoms than the relative well-founded operator
when applied to the least precise interpretation (as an easy point of reference).
How well we can approximate the certain atoms with the approximate operator depends on the set of external predicates~\(\PredE\).
When approximating the model of a program~\(P\) in a sequence, the set~\(\PredE\) comprises all negative predicates occurring in~\(P\) for which possible atoms have not yet been fully computed.
This leads to fewer certain atoms obtained from the reduced program,
\(P'=\{r \in P\mid \Predicate{\Body{r}^-} \cap \PredE = \emptyset\}\),
stripped of all rules from \(P\) that have negative body literals whose predicates occur in \(\PredE\).

The next theorem identifies an essential prerequisite for
an approximate model of a non-ground program to be an underapproximation of the well-founded model of the corresponding ground program.

\begin{theorem}\label{lem:approximate:well-founded}
Let \(P\) be an aggregate program,
\(\PredE\) be a set of predicates, and
\((\Ir{I},\Ir{J})\) be a four-valued interpretation.

If \(\Predicate{\Head{P}} \cap \Predicate{\Body{P}^-} \subseteq \PredE\) then \(\ApproxRM{P}{\PredE}{\Ir{I},\Ir{J}} \leq_p \WellRM{\TransA{P}}{\Ir{I},\Ir{J} \cup \Ir{E}}\) where \(\Ir{E}\) is the set of all ground atoms over predicates in \(\PredE\).
\end{theorem}

In general, a grounder cannot calculate on-the-fly a well-founded model.
Implementing this task efficiently requires an algorithm storing the grounded program, as, for example, implemented in an ASP solver.
But modern grounders are able to calculate the stable operator on-the-fly.
Thus, an approximation of the well-founded model is calculated.
This is where we use the approximate model, which might be less precise than the well-founded model but can be computed more easily.

With the condition of \cref{lem:approximate:well-founded} in mind,
we define the approximate model for an instantiation sequence.
We proceed similar to \cref{def:well-founded:sequence}
but treat in~\eqref{def:approximate:sequence:external} all atoms over negative predicates that have not been completely defined as external.

\begin{definition}\label{def:approximate:sequence}
Let \(\Fsqi{P}{\Index}{i}\) be a (refined) instantiation sequence for \(P\).

Then, the \emph{approximate model} of \(\Fsqi{P}{\Index}{i}\) is
\begin{align}
\ApproxM{\Fsqi{P}{\Index}{i}} &= \bigsqcup_{i \in \Index}(\Ii{I}{i},\Ii{J}{i}).\nonumber\\
\intertext{where}
\Ii{\PredE}{i} & = \Predicate{\IDNeg{\Body{\Fi{P}{i}}}} \cap \Predicate{\bigcup_{i \leq j}\Head{\Fi{P}{j}}},\label{def:approximate:sequence:external}\\
(\Iri{I}{i},\Iri{J}{i}) & = \bigsqcup_{j < i} (\Ii{I}{j},\Ii{J}{j})\text{, and}\label{def:approximate:sequence:relative}\\
(\Ii{I}{i},\Ii{J}{i}) & = \ApproxRM{\Fi{P}{i}}{\Ii{\PredE}{i}}{\Iri{I}{i},\Iri{J}{i}}\label{def:approximate:sequence:intermediate}.
\end{align}
\end{definition}

Note that the underlying approximate model removes rules containing negative literals over predicates in \(\Fi{\PredE}{i}\) when calculating certain atoms.
This amounts to assuming all ground instances of atoms over \(\Fi{\PredE}{i}\) to be possible.\footnote{To be precise, rules involving aggregates that could in principle derive certain atoms might be removed, too.
Here, we are interested in a syntactic criteria that allows us to underapproximate the set of certain atoms.}
Compared to~\eqref{eq:well-founded:sequence:external}, however,
this additionally includes recursive predicates in~\eqref{def:approximate:sequence:external}.
The set \(\Fi{\PredE}{i}\) is empty for stratified components.

The next result relies on \cref{thm:sequence:well-founded} to show that
an approximate model of an instantiation sequence
constitutes an underapproximation of the
well-founded model of the translated entire program.
In other words,
the approximate model of a sequence of aggregate programs (as computed by a grounder)
is less precise than the well-founded model of the whole ground program.

\begin{theorem}\label{thm:approximate:sequence:well-founded}
Let \(\Fsqi{P}{\Index}{i}\) be an instantiation sequence for aggregate program~\(P\)
and \(\Fsqi{P}{\IndexAlt}{j}\) be a refinement of \(\Fsqi{P}{\Index}{i}\).

Then, \(\ApproxM{\Fsqi{P}{\Index}{i}} \leq_p \ApproxM{\Fsqi{P}{\IndexAlt}{j}} \leq_p \WellM{\TransA{P}}\).
\end{theorem}

The finer granularity of refined instantiation sequences leads to more precise models.
Intuitively, this is because a refinement of a component may result in a series of approximate models,
which yield a more precise result than the approximate model of the entire component
because in some cases fewer predicates are considered external in~\eqref{def:approximate:sequence:external}.

We remark that all instantiation sequences of a program have the same approximate model.
However, this does not carry over to refined instantiation sequences because their evaluation is order
dependent.

The two former issues are illustrated in \cref{ex:dependency}.

The actual value of approximate models for grounding lies in their underlying series of consecutive interpretations delineating each ground program
in a (refined) instantiation sequence.
In fact, as outlined after \cref{prp:translation:restricted:properties},
whenever all interpretations \((\Ii{I}{i},\Ii{J}{i})\) in~\eqref{def:approximate:sequence:intermediate} are finite
so are the \FormulasR-programs \(\Simp{\TransR{\Fi{P}{i}}{\Iri{J}{i} \cup \Ii{J}{i}}}{(\Iri{I}{i},\Iri{J}{i}) \sqcup (\Ii{I}{i},\Ii{J}{i})}\)
obtained from each \(\Fi{P}{i}\) in the instantiation sequence.

\begin{theorem}\label{thm:approximate:sequence:simplification}
Let \(\Fsqi{P}{\Index}{i}\) be a (refined) instantiation sequence of an aggregate program~\(P\),
and let \((\Iri{I}{i},\Iri{J}{i})\) and \((\Ii{I}{i},\Ii{J}{i})\) be defined as in \cref{def:approximate:sequence:relative,def:approximate:sequence:intermediate}.

Then,
\(\bigcup_{i \in \Index}\Simp{\TransR{\Fi{P}{i}}{\Iri{J}{i} \cup \Ii{J}{i}}}{(\Iri{I}{i},\Iri{J}{i}) \sqcup (\Ii{I}{i},\Ii{J}{i})}\) and \(\TransA{P}\) have the same well-founded and stable models.
\end{theorem}

Notably,
this union of \FormulasR-programs is exactly the one obtained by the grounding algorithm proposed in the next section
(cf.\ \cref{thm:ground:program}).

\begin{example}\label{ex:dependency}
We continue \cref{ex:aggregate-program:sequence:refined}.

The approximate model of the instantiation sequence~\(\Fsqi{P}{\Index}{i}\),
defined in \cref{def:approximate:sequence}, is less precise than the well-founded model of the sequence, viz.\
\begin{align*}
\ApproxM{\Fsqi{P}{\Index}{i}} &= (\emptyset,\{p(1),p(2),q(2),q(3),x,y\}) \sqcup F.
\end{align*}
This is because we have to use \(\ApproxRM{\Fi{P}{5}}{\PredE}{F,F}\) to approximate the well-founded model of component \(\Fi{P}{5}\).
Here, the set \(\PredE = \{ a/1, b/1 \}\) determined by \cref{def:approximate:sequence:external} forces us to unconditionally assume instances of \(\Naf q(X)\) and \(\Naf p(X)\) to be true.
Thus, we get \((\Ii{I}{5}, \Ii{J}{5}) = (\emptyset, \{p(1),p(2),q(2),q(3)\})\) for the intermediate interpretation in~\eqref{def:approximate:sequence:intermediate}.
This is also reflected in \cref{def:approximate}, which makes us drop all rules containing negative literals over predicates in~\(\PredE\) when calculating true atoms.

In accord with \cref{thm:approximate:sequence:well-founded},
we approximate the well-founded model w.r.t.\ the refined instantiation sequence~\(\Fsqic{P}{\IndexAlt}{i,j}\)
and obtain
\begin{align*}
\ApproxM{\Fsqic{P}{\IndexAlt}{i,j}} &= (\{q(3)\},\{p(1),p(2),q(2),q(3),x\}) \sqcup F,
\end{align*}
which, for this example, is equivalent to the well-founded model of the corresponding ground refined instantiation sequence
and more precise than the approximate model of the instantiation sequence.

In an actual grounder implementation the approximate model is only a byproduct,
instead, it outputs a program equivalent to the one in \cref{thm:approximate:sequence:simplification}:
\begin{align*}
u(1) & &
u(2) & \\
v(2) & &
v(3) & \\
p(1) & \leftarrow \Naf q(1) \wedge u(1) &
p(2) & \leftarrow \Naf q(2) \wedge u(2) \\
q(2) & \leftarrow \Naf p(2) \wedge v(2) &
q(3) & \leftarrow \Naf p(3) \wedge v(3) \\
x & \leftarrow \Naf p(1)
\end{align*}
Note that the rule \(y \leftarrow \Naf q(3)\) is not part of the simplification because \(q(3)\) is certain.
\end{example}

\begin{remark}
The reason why we use the refined grounding is that we cannot expect a grounding algorithm to calculate the well-founded model for a component without further processing.
But at least some consequences should be considered.
\Gringo\ is designed to ground on-the-fly without storing any rules,
so it cannot be expected to compute all possible consequences
but it should at least take all consequences from preceding interpretations into account.
With the help of a solver, we could calculate the exact well-founded model of a component after it has been grounded.
\end{remark}

Whenever an aggregate program is stratified,
the approximate model of its instantiation sequence is total
(and coincides with the well-founded model of the entire ground program).

\begin{theorem}\label{thm:approximate:sequence:well-founded:stratified}
Let \(\Fsqi{P}{\Index}{i}\) be an instantiation sequence of an aggregate program~\(P\)
such that \(\Ii{\PredE}{i}=\emptyset\) for each \(i \in \Index\) as defined in \cref{def:approximate:sequence:external}.

Then, \(\ApproxM{\Fsqi{P}{\Index}{i}}\) is total.
\end{theorem}

\begin{example}\label{ex:dependency:company}
The dependency graph of the company controls encoding is given in \cref{fig:dependency:company}
and follows the conventions in \cref{ex:dependency}.
Because the encoding only uses positive literals and monotone aggregates,
grounding sequences cannot be refined further.
Since the program is positive, we can apply \cref{thm:approximate:sequence:well-founded:stratified}.
Thus, the approximate model of the grounding sequence is total and corresponds to the well-founded model of the program.
We use the same abbreviations for predicates as in \cref{fig:dependency:company}.
The well-founded model is \((F \cup I,F \cup I)\) where
\begin{align*}
F = \{ & c(c_1),c(c_2),c(c_3),c(c_4),\\ & o(c_1,c_2,60),o(c_1,c_3,20),o(c_2,c_3,35), o(c_3,c_4,51) \} \AndT\\
I = \{ & s(c_1,c_2),s(c_3,c_4),s(c_1,c_3),s(c_1,c_4) \}.
\end{align*}
\begin{figure}[ht]
\centering
\includegraphics{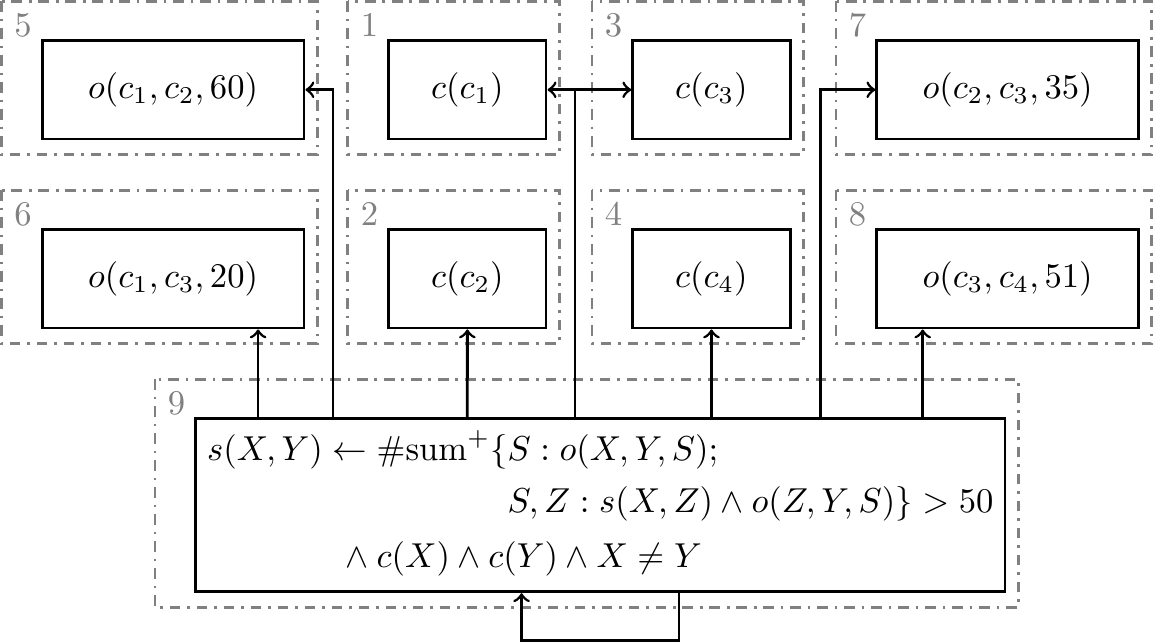}
\caption{Rule dependencies for the company controls encoding and instance in \cref{ex:company} where \(c=\Pred{company}\), \(o=\Pred{owns}\), and \(s=\Pred{controls}\).\label{fig:dependency:company}}
\end{figure}
\end{example}

 \section{Algorithms}\label{sec:algorithms}

This section lays out the basic algorithms for grounding rules, components, and entire programs
and characterizes their output in terms of the semantic concepts developed in the previous sections.
Of particular interest is the treatment of aggregates, which are decomposed into dedicated normal rules before grounding
and reassembled afterward.
This allows us to ground rules with aggregates by means of grounding algorithms for normal rules.
Finally,
we show that our grounding algorithm guarantees that
an obtained finite ground program is equivalent to the original non-ground program.

In the following, we refer to terms, atoms, comparisons, literals, aggregate elements, aggregates, or rules as \emph{expressions}.
As in the preceding sections,
all expressions, interpretations, and concepts introduced below
operate on the same (implicit) signature \(\Signature\) unless mentioned otherwise.

A \emph{substitution} is a mapping from the variables in \(\Signature\) to terms over \(\Signature\).
We use \(\iota\) to denote the identity substitution mapping each variable to itself.
A \emph{ground substitution} maps all variables to ground terms or themselves.
The result of \emph{applying} a substitution \(\sigma\) to an expression \(e\), written \(e\sigma\),
is the expression obtained by replacing each variable \(v\) in \(e\) by \(\sigma(v)\).
This directly extends to sets \(E\) of expressions, that is, \(E\sigma = \{e\sigma \mid e \in E\}\).

The \emph{composition} of substitutions \(\sigma\) and \(\theta\)
is the substitution \(\sigma \circ \theta\) where \((\sigma \circ \theta)(v) = \theta(\sigma(v))\) for each variable \(v\).

A substitution \(\sigma\) is a \emph{unifier} of a set \(E\) of expressions
if \(e_1 \sigma = e_2\sigma\) for all \(e_1,e_2 \in E\).
In what follows, we are interested in one-sided unification, also called \emph{matching}.
A substitution \(\sigma\) matches a non-ground expression~\(e\) to a ground expression~\(g\),
if \(e\sigma=g\) and \(\sigma\) maps all variables not occurring in~\(e\) to themselves.
We call such a substitution the \emph{matcher} of~\(e\) to~\(g\).
Note that a matcher is a unique ground substitution unifying~\(e\) and~\(g\), if it exists.
This motivates the following definition.

For a (non-ground) expression~\(e\) and a ground expression~\(g\),
we define:
\begin{align*}
\Match{e}{g} = \begin{cases}
\{ \sigma \} & \text{if there is a matcher \(\sigma\) from \(e\) to \(g\)}\\
\emptyset    & \text{otherwise}
\end{cases}
\end{align*}

When grounding rules,
we look for matches of non-ground body literals in the possibly derivable atoms accumulated so far.
The latter is captured by a four-valued interpretation to distinguish certain atoms among the possible ones.
This is made precise in the next definition.

\begin{definition}\label{def:matches}
Let \(\sigma\) be a substitution,
\(l\) be a literal or comparison,
and \((I,J)\) be a four-valued interpretation.

We define the set of \emph{matches} for \(l\) in \((I,J)\) w.r.t.\ \(\sigma\),
written \Matches{l}{I,J}{\sigma},
\begin{align*}
\intertext{for an atom \(l = a\) as}
\Matches{a}{I,J}{\sigma} &= \{ \sigma \circ \sigma' \mid a' \in J, \sigma' \in \Match{a\sigma}{a'} \}\text{,}
\intertext{for a ground literal \(l=\Naf a\) as}
\Matches{\Naf a}{I,J}{\sigma} &= \{ \sigma \mid a\sigma \not\in I \}\text{, and}
\intertext{for a ground comparison \(l=t_1 \prec t_2\) as in~\eqref{eq:comparison} as}
\Matches{t_1 \prec t_2}{I,J}{\sigma} &= \{ \sigma \mid \text{\(\prec\) holds between \(t_1\sigma\) and \(t_2\sigma\)} \}.
\end{align*}
\end{definition}

In this way,
positive body literals yield a (possibly empty) set of substitutions, refining the one at hand,
while negative and comparison literals are only considered when ground and then act as a test on the given substitution.

\begin{algorithm}[ht]
\caption{Grounding Rules\label{fun:ground:rule}}
\Fn{\(\GroundRule{r}{I}{J}{J'}{\sigma}{L}{f}\)}{\uIf(\tcp*[f]{match next}){\(L\neq\emptyset\)}{\label{fun:ground:rule:next}
    \((G, l) \leftarrow (\emptyset, \Select{\sigma}{L})\)\label{fun:ground:rule:select}\;
    \ForEach{\(\sigma' \in \Matches{l}{I,J}{\sigma}\)}{\label{fun:ground:rule:loop:begin}\(G \leftarrow G \cup \GroundRule{r}{I}{J}{J'}{\sigma'}{L\setminus \{l\}}{f}\)\label{fun:ground:rule:recurse}\label{fun:ground:rule:loop:end}\;
    }
    \Return\ \(G\)\label{fun:ground:rules:result}\;
  }
  \uElseIf(\tcp*[f]{rule instance}){\(f=\True\)\ \Or\ \(\IDPos{\Body{r\sigma}} \nsubseteq J'\)\label{fun:ground:rule:test}}{\Return\ \(\{r \sigma\}\)\label{fun:ground:rule:result}\;
  }
  \Else(\tcp*[f]{rule seen}){\label{fun:ground:rule:seen}\Return\ \(\emptyset\)\;
  }
}
\end{algorithm}

Our function for rule instantiation is given in \cref{fun:ground:rule}.
It takes a substitution~\(\sigma\) and a set~\(L\) of literals and yields a set of ground instances of a safe normal rule~\(r\),
passed as a parameter;
if called with the identity substitution and the body literals \(\Body{r}\) of \(r\),
it yields ground instances of the rule.
The other parameters consist of
a four-valued interpretation~\((I,J)\) comprising the set of possibly derivable atoms along with the certain ones,
a two-valued interpretation~\(J'\) reflecting the previous value of~\(J\), and
a Boolean flag~\(f\) used to avoid duplicate ground rules in consecutive calls to \cref{fun:ground:rule}.
The idea is to extend the current substitution in \crefrange{fun:ground:rule:loop:begin}{fun:ground:rule:loop:end}
until we obtain a ground substitution~\(\sigma\) that induces a ground instance~\(r\sigma\) of rule~\(r\).
To this end,
\(\Select{\sigma}{L}\) picks for each call some literal \(l \in L\) such that \(l \in L^+\) or \(l\sigma\) is ground.
That is, it yields either a positive body literal or a ground negative or ground comparison literal,
as needed for computing \Matches{l}{I,J}{\sigma}.
Whenever an application of \KwMatches\ for the selected literal in \(\Body{r}\) results in a non-empty set of substitutions,
the function is called recursively for each such substitution.
The recursion terminates if at least one match is found for each body literal and an instance \(r\sigma\) of \(r\) is obtained in \cref{fun:ground:rule:result}.
The set of all such ground instances is returned in \cref{fun:ground:rules:result}.
(Note that we refrain from applying any simplifications to the ground rules and rather leave them intact to
obtain more direct formal characterizations of the results of our grounding algorithms.)
The test \(\IDPos{\Body{r\sigma}} \nsubseteq J'\) in \cref{fun:ground:rule:test} makes sure that no ground rules are
generated that were already obtained by previous invocations of \cref{fun:ground:rule}.
This is relevant for recursive rules and reflects the approach of semi-naive database evaluation~\cite{abhuvi95a}.

For characterizing the result of \cref{fun:ground:rule} in terms of aggregate programs, we need the following definition.

\begin{definition}
Let \(P\) be an aggregate program and \((I,J)\) be a four-valued interpretation.

We define \(\GrdSimp{P}{I,J}\subseteq\Grd{P}\) as the set of all instances \(g\) of rules in \(P\)
satisfying \(J \models \IDReductP{\TransA{\Body{g}}^\wedge}{I}\).
\end{definition}

In terms of the program simplification in \cref{def:simplification},
an instance \(g\) belongs to \(\GrdSimp{P}{I,J}\)
iff \(\Head{g} \leftarrow \TransA{\Body{g}}^\wedge \in \Simp{\TransA{r}}{I,J}\).
Note that the members of \(\GrdSimp{P}{I,J}\) are not necessarily ground, since non-global variables may remain within aggregates;
though they are ground for normal rules.

We use \cref{fun:ground:rule} to iteratively compute ground instances of a rule w.r.t.\ an increasing set of atoms.
The Boolean flag~\(f\) and the set of atoms \(J'\) are used to avoid duplicating ground instances in successive iterations.
The flag \(f\) is initially set to true to not filter any rule instances.
In subsequent iterations,
duplicates are omitted by setting the flag to false and filtering rules whose positive bodies are a subset of the atoms \(J'\) used in previous iterations.
This is made precise in the next result.

\begin{proposition}\label{lem:ground:rule:split}
Let \(r\) be a safe normal rule and \((I,J)\) be a finite four-valued interpretation.

Then,
\begin{alphaenum}
\item \(\GrdSimp{\{r\}}{I,J} = \GroundRule{r}{I}{J}{\emptyset}{\iota}{\Body{r}}{\True}\) and\label{lem:ground:rule:split:a}
\item \(\GrdSimp{\{r\}}{I,J} = \GrdSimp{\{r\}}{I,J'} \cup \GroundRule{r}{I}{J}{J'}{\iota}{\Body{r}}{\False}\) for all \(J' \subseteq J\).\label{lem:ground:rule:split:b}
\end{alphaenum}
\end{proposition}

Now, let us turn to the treatment of aggregates.
To this end,
we define the following translation of aggregate programs to normal programs.

\begin{definition}\label{def:rewrite}
Let \(P\) be a safe aggregate program over signature \(\Sigma\).

Let \(\Sigma'\) be the signature obtained by extending \(\Sigma\) with fresh predicates
\begin{align}
& \AggregateAtomP{a}{r}{n}\text{, and}\label{eq:rewrite:aggregate:predicate}\\
& \EmptyAtomP{a}{r}{n}\label{eq:rewrite:empty:predicate}
\end{align}
for each aggregate \(a\) occurring in a rule~\(r \in P\)
where \(n\) is the number of global variables in \(a\), and
fresh predicates
\begin{align}
& \ElementAtomP{e}{a}{r}{(m+n)}\label{eq:rewrite:element:predicate}
\end{align}
for each aggregate element \(e\) occurring in aggregate~\(a\) in rule \(r\)
where \(m\) is the size of the tuple~\(\Tuple{e}\).

We define \(\Fa{P}\), \(\Fv{P}\), and \(\Fe{P}\) as normal programs over \(\Sigma'\) as follows.

\begin{itemize}
\item
Program~\(\Fa{P}\) is obtained from \(P\) by replacing each aggregate occurrence~\(a\) in~\(P\) with
\begin{align}
\AggregateAtom{a}{r}{X_1,\dots,X_n}\label{eq:rewrite:aggregate:atom}
\end{align}
where \(\AggregateAtomP{a}{r}{n}\) is defined as in~\eqref{eq:rewrite:aggregate:predicate} and \(X_1,\dots,X_n\) are the global variables in \(a\).
\item
Program~\(\Fv{P}\) consists of rules
\begin{align}
\EmptyAtom{a}{r}{X_1,\dots,X_n} & \leftarrow t \prec b \wedge b_1 \wedge \dots \wedge b_l\label{eq:rewrite:empty:rule}
\end{align}
for each predicate~\(\EmptyAtomP{a}{r}{n}\) as in~\eqref{eq:rewrite:empty:predicate}
where
\(X_1,\dots,X_n\) are the global variables in \(a\),
\(a\) is an aggregate of form \(f \{ E \} \prec b\) occurring in \(r\),
\(t = f(\emptyset)\) is the value of the aggregate function applied to the empty set,
and~\(b_1,\dots,b_l\) are the body literals of \(r\) excluding aggregates.
\item
Program~\(\Fe{P}\) consists of rules
\begin{align}
\ElementAtom{e}{a}{r}{t_1,\dots,t_m,X_1,\dots,X_n} & \leftarrow c_1 \wedge \dots \wedge c_k \wedge b_1 \wedge \dots \wedge b_l\label{eq:rewrite:element:rule}
\end{align}
for each predicate~\(\ElementAtomP{e}{a}{r}{m+n}\) as in~\eqref{eq:rewrite:element:predicate}
where~\((t_1,\dots,t_m) = \Head{e}\),
\(X_1,\dots,X_n\) are the global variables in \(a\),
\(\{c_1,\dots,c_k\} = \Body{e}\),
and~\(b_1,\dots,b_l\) are the body literals of \(r\) excluding aggregates.
\end{itemize}
\end{definition}

Summarizing the above,
we translate
an aggregate program~\(P\) over \(\Sigma\) into
a normal program~\(\Fa{P}\) along with auxiliary normal rules in~\(\Fv{P}\) and~\(\Fe{P}\),
all over a signature extending \(\Sigma\) by the special-purpose predicates in
\cref{eq:rewrite:aggregate:predicate,eq:rewrite:empty:predicate,eq:rewrite:element:predicate}.
In fact, there is a one-to-one correspondence between the rules in~\(P\) and~\(\Fa{P}\),
so that we get \(P=\Fa{P}\) and \(\Fv{P}=\Fe{P}=\emptyset\) whenever~\(P\) is normal.

\begin{example}\label{ex:rewrite}
We illustrate the translation of aggregate programs on the company controls example in \Cref{ex:company}.
We rewrite the rule
\begin{align*}
\begin{split}
\Pred{controls}(X,Y) \leftarrow {}
& \!\overbrace{\begin{aligned}[t]
\Sum^+ \{
& \underbrace{S : \Pred{owns}(X,Y,S)}_{e_1}; \\
& \underbrace{S,Z: \Pred{controls}(X,Z) \wedge \Pred{owns}(Z,Y,S)}_{e_2} \} > 50
\end{aligned}}^{a} \\
{} \wedge {}
& \Pred{company}(X) \wedge \Pred{company}(Y) \wedge X \neq Y
\end{split}
\tag{\(r\)}
\end{align*}
containing aggregate~\(a\) with elements \(e_1\) and \(e_2\), into rules \(r_1\) to \(r_4\):
\begin{align*}
\begin{split}
\Pred{controls}(X,Y) \leftarrow {} & \AggregateAtom{a}{r}{X,Y}\\
{} \wedge {} & \Pred{company}(X) \wedge \Pred{company}(Y) \wedge X \neq Y,
\end{split}
\tag{\(r_1\)}\label{ex:rewrite:r1} \\
\begin{split}
\EmptyAtom{a}{r}{X,Y} \leftarrow {} & 0 >50 \\
{} \wedge {} & \Pred{company}(X) \wedge \Pred{company}(Y) \wedge X \neq Y,
\end{split}
\tag{\(r_2\)}\label{ex:rewrite:r2}\\
\begin{split}
\ElementAtom{e_1}{a}{r}{S,X,Y} \leftarrow {} & \Pred{owns}(X,Y,S) \\
{} \wedge {} & \Pred{company}(X) \wedge \Pred{company}(Y) \wedge X \neq Y,\AndT
\end{split}
\tag{\(r_3\)}\label{ex:rewrite:r3}\\
\begin{split}
\ElementAtom{e_2}{a}{r}{S,Z,X,Y} \leftarrow {} & \Pred{controls}(X,Z) \wedge \Pred{owns}(Z,Y,S)\\
{} \wedge {} & \Pred{company}(X) \wedge \Pred{company}(Y) \wedge X \neq Y.
\end{split}
\tag{\(r_4\)}\label{ex:rewrite:r4}
\end{align*}
We have \(\Fa{P} = \{r_1\}\), \(\Fv{P} = \{ r_2 \}\), and \(\Fe{P} = \{ r_3, r_4 \}\).
\end{example}

This example illustrates how possible instantiations of aggregate elements are gathered via the rules in \(\Fe{P}\).
Similarly, the rules in \(\Fv{P}\) collect instantiations warranting that the result of applying aggregate functions to
the empty set is in accord with the respective bound.
In both cases, the relevant variable bindings are captured by the special head atoms of the rules.
In turn, groups of corresponding instances of aggregate elements are used in
\cref{def:propagate} to sanction the derivation of ground atoms of form \cref{eq:rewrite:aggregate:atom}.
These atoms are ultimately replaced in \(\Fa{P}\) with the original aggregate contents.

We next define two functions gathering information from instances of rules in \(\Fv{P}\) and \(\Fe{P}\).
In particular, we make precise how groups of aggregate element instances are obtained from ground rules in \(\Fe{P}\).

\begin{definition}
Let \(P\) be an aggregate program,
and \(\Fv{G}\) and \(\Fe{G}\) be subsets of ground instances of rules in \(\Fv{P}\) and \(\Fe{P}\), respectively.
Furthermore, let \(a\) be an aggregate occurring in some rule~\(r \in P\)
and \(\sigma\) be a substitution mapping the global variables in \(a\) to ground terms.

We define
\begin{align*}
\AggrEmpty{r}{a}{\Fv{G}}{\sigma} & = \bigcup_{g \in \Fv{G}} \Match{r_a\sigma}{g}
\intertext{where \(r_a\) is a rule of form~\eqref{eq:rewrite:empty:rule} for aggregate occurrence \(a\), and}
\AggrElem{r}{a}{\Fe{G}}{\sigma} & = \{ e\sigma\theta \mid g \in \Fe{G}, e \in E, \theta \in \Match{r_e\sigma}{g} \}
\end{align*}
where \(E\) are the aggregate elements of \(a\)
and \(r_e\) is a rule of form~\eqref{eq:rewrite:element:rule} for aggregate element occurrence~\(e\) in \(a\).
\end{definition}

Given that \(\sigma\) maps the global variables in \(a\) to ground terms,
\(r_a\sigma\) is ground whereas
\(r_e\sigma\) may still contain local variables from \(a\).
The set \(\AggrEmpty{r}{a}{\Fv{G}}{\sigma}\) has an indicative nature:
For an aggregate \(a\sigma\), it
contains the identity substitution when the result of applying its aggregate function to the empty set is in accord with its bound,
and it is empty otherwise.
The construction of \(\AggrElem{r}{a}{\Fe{G}}{\sigma}\) goes one step further and
reconstitutes all ground aggregate elements of \(a\sigma\)
from variable bindings obtained by rules in \(\Fe{G}\).
Both functions play a central role below in defining the function \KwPropagate\
for deriving ground aggregate placeholders of form \cref{eq:rewrite:aggregate:atom} from ground rules in \(\Fv{G}\) and \(\Fe{G}\).

\begin{example}\label{ex:rewrite:gather}
We show how to extract aggregate elements from ground instances of rules~\eqref{ex:rewrite:r3} and~\eqref{ex:rewrite:r4} in \cref{ex:rewrite}.

Let \(\Fv{G}\) be empty and \(\Fe{G}\) be the program consisting of the following rules:
\begin{align*}
\ElementAtom{e_1}{a}{r}{60,c_1,c_2} \leftarrow {} & \Pred{owns}(c_1,c_2,60)\\
{} \wedge {} & \Pred{company}(c_1) \wedge \Pred{company}(c_2) \wedge c_1 \neq c_2,\\
\ElementAtom{e_1}{a}{r}{20,c_1,c_3} \leftarrow {} & \Pred{owns}(c_1,c_3,20)\\
{} \wedge {} & \Pred{company}(c_1) \wedge \Pred{company}(c_3) \wedge c_1 \neq c_3,\\
\ElementAtom{e_1}{a}{r}{35,c_2,c_3} \leftarrow {} & \Pred{owns}(c_2,c_3,35)\\
{} \wedge {} & \Pred{company}(c_2) \wedge \Pred{company}(c_3) \wedge c_2 \neq c_3,\\
\ElementAtom{e_1}{a}{r}{51,c_3,c_4} \leftarrow {} & \Pred{owns}(c_3,c_4,51)\\
{} \wedge {} & \Pred{company}(c_3) \wedge \Pred{company}(c_4) \wedge c_3 \neq c_4,\AndT\\
\ElementAtom{e_2}{a}{r}{35,c_2,c_1,c_3} \leftarrow {} & \Pred{controls}(c_1,c_2) \wedge \Pred{owns}(c_2,c_3,35)\\
{} \wedge {} & \Pred{company}(c_1) \wedge \Pred{company}(c_3) \wedge c_1 \neq c_3.
\end{align*}

Clearly, we have \(\AggrEmpty{r}{a}{\Fv{G}}{\sigma} = \emptyset\) for any substitution~\(\sigma\) because \(\Fv{G} = \emptyset\).
This means that aggregate \(a\) can only be satisfied if at least one of its elements is satisfiable.
In fact,
we obtain non-empty sets \(\AggrElem{r}{a}{\Fe{G}}{\sigma}\) of ground aggregate elements for four substitutions~\(\sigma\):
\begin{align*}
\AggrElem{r}{a}{\Fe{G}}{\sigma_1} &= \{ 60: \Pred{owns}(c_1,c_2,60) \}  & \ForT & \sigma_1: X \mapsto c_1, Y \mapsto c_2,\\
\AggrElem{r}{a}{\Fe{G}}{\sigma_2} &= \{ 51: \Pred{owns}(c_3,c_4,51) \}  & \ForT & \sigma_2: X \mapsto c_3, Y \mapsto c_4,\\
\AggrElem{r}{a}{\Fe{G}}{\sigma_3} &= \rlap{\(\{ 35, c_2: \Pred{controls}(c_1,c_2) \wedge \Pred{owns}(c_2,c_3,35)\);}\\
                                  &\phantom{{} = \{} 20: \Pred{owns}(c_1,c_3,20) \}  & \ForT & \sigma_3: X \mapsto c_1, Y \mapsto c_3, \AndT\\
\AggrElem{r}{a}{\Fe{G}}{\sigma_4} &= \{ 35: \Pred{owns}(c_2,c_3,35) \}  & \ForT & \sigma_4: X \mapsto c_2, Y \mapsto c_3.
\end{align*}
\end{example}

For capturing the result of grounding aggregates relative to groups of aggregate elements gathered via \(\Fe{P}\),
we restrict their original translation to subsets of their ground elements.
That is, while \TransA{a} and \TransAD{a}{\cdot} draw in \cref{def:translation:aggregate} on all instances of
aggregate elements in \(a\),
their counterparts \TransAD{\ElemG}{a} and \TransAD{a,\ElemG}{\cdot} are restricted to a subset
of such aggregate element instances:\footnote{Note that the restriction to sets of ground aggregate elements is similar to the one to two-valued interpretations in~\cref{def:translation:aggregate:restricted}.}

\begin{definition}\label{def:translate:restricted:elements}
Let \(a\) be a closed aggregate and of form~\eqref{eq:aggregate}, \(E\) be the set of its aggregate elements,
and \(\ElemG \subseteq \Grd{E}\) be a set of aggregate element instances.

We define the translation \(\TransAD{\ElemG}{a}\) of \(a\) w.r.t.\ \(\ElemG\) as follows:
\begin{align*}
\TransAD{\ElemG}{a} & = {\{ {\Trans{\ElemJ}}^\wedge \rightarrow {\TransAD{a,\ElemG}{\ElemJ}}^\vee \mid \ElemJ \subseteq \ElemG, \ElemJ \notjustifies a \}}^\wedge
\intertext{where}
\TransAD{a,\ElemG}{\ElemJ} & = \{ {\Trans{\ElemC}}^\wedge \mid \ElemC \subseteq \ElemG \setminus \ElemJ, \ElemC \cup \ElemJ \justifies a \}.
\end{align*}
\end{definition}

As before,
this translation maps aggregates, possibly including non-global variables, to a conjunction of (ground) \(\FormulasR\)-rules.
The resulting \(\FormulasR\)-formula is used below in the definition of functions \KwPropagate\ and \KwAssemble.

\begin{example}\label{ex:rewrite:trans:elem}
We consider the four substitutions~\(\sigma_1\) to~\(\sigma_4\)
together with the sets~\(G_1 = \AggrElem{r}{a}{\Fe{G}}{\sigma_1}\) to~\(G_4 = \AggrElem{r}{a}{\Fe{G}}{\sigma_4}\)
from \cref{ex:rewrite:gather} for aggregate~\(a\).

Following the discussion after \cref{prp:translation:monotonicity}, we get the formulas
\begin{align*}
\TransAD{G_1}{a\sigma_1} & = \Pred{owns}(c_1, c_2, 60),\\
\TransAD{G_2}{a\sigma_2} & = \Pred{owns}(c_3, c_4, 51),\\
\TransAD{G_3}{a\sigma_3} & = \Pred{controls}(c_1, c_2) \wedge \Pred{owns}(c_2,c_3,35) \wedge \Pred{owns}(c_1, c_3, 20)\CandT\\
\TransAD{G_4}{a\sigma_4} & = \bot.
\end{align*}
\end{example}

The function \KwPropagate\ yields a set of ground atoms of form~\eqref{eq:rewrite:aggregate:atom}
that are used in \cref{fun:ground:component} to ground rules having such placeholders among their body literals.
Each such special atom is supported by a group of ground instances of its aggregate elements.

\begin{definition}\label{def:propagate}
Let \(P\) be an aggregate program,
\((I,J)\) be a four-valued interpretation,
and \(\Fv{G}\) and \(\Fe{G}\) be subsets of ground instances of rules in \(\Fv{P}\) and \(\Fe{P}\), respectively.

We define \(\Propagate{P}{\Fv{G}}{\Fe{G}}{I}{J}\) as the set of all atoms of form \(\alpha\sigma\)
such that \(\AggrEmpty{r}{a}{\Fv{G}}{\sigma} \cup G \neq \emptyset\) and \(J \models \IDReductP{\TransAD{G}{a\sigma}}{I}\)
with \(G = \AggrElem{r}{a}{\Fe{G}}{\sigma}\)
where \(\alpha\) is an atom of form~\eqref{eq:rewrite:aggregate:atom} for aggregate \(a\) in rule \(r\)
and \(\sigma\) is a ground substitution for~\(r\) mapping all global variables in \(a\) to ground terms.
\end{definition}

An atom \(\alpha\sigma\) is only considered if \(\sigma\) warrants ground rules in \(\Fv{G}\) or \(\Fe{G}\), signaling
that the application of \(\alpha\) to the empty set is feasible when applying \(\sigma\) or
that there is a non-empty set of ground aggregate elements of \(\alpha\) obtained after applying \(\sigma\), respectively.
If this is the case, it is checked whether the set \(G\) of aggregate element instances warrants that
\(\TransAD{G}{a\sigma}\) admits stable models between \(I\) and  \(J\).

\begin{example}\label{ex:rewrite:progagate}
We show how to propagate aggregates using the sets \(G_1\) to \(G_4\) and their associated formulas from \cref{ex:rewrite:trans:elem}.
Suppose that \(I = J = F \cup \{ \Pred{controls}(c_1,c_2) \}\) using \(F\) from \cref{ex:dependency:company}.

Observe that \(J \models \IDReductP{\TransAD{G_1}{a\sigma_1}}{I}\), \(J \models \IDReductP{\TransAD{G_2}{a\sigma_2}}{I}\), \(J \models \IDReductP{\TransAD{G_3}{a\sigma_3}}{I}\), and \(J \not\models \IDReductP{\TransAD{G_4}{a\sigma_4}}{I}\).
Thus, we get \(\Propagate{P}{\Fv{G}}{\Fe{G}}{I}{J} = \{ \AggregateAtom{a}{r}{c_1,c_2}, \AggregateAtom{a}{r}{c_1,c_3}, \AggregateAtom{a}{r}{c_3,c_4} \}\).
\end{example}

The function \KwAssemble\ yields an \FormulasR-program
in which aggregate placeholder atoms of form~\eqref{eq:rewrite:aggregate:atom}
have been replaced by their corresponding \(\FormulasR\)-formulas.

\begin{definition}\label{def:assemble}
Let \(P\) be an aggregate program,
and \(\Fa{G}\) and \(\Fe{G}\) be subsets of ground instances of rules in \(\Fa{P}\) and \(\Fe{P}\), respectively.

We define \(\Assemble{\Fa{G}}{\Fe{G}}\) as the \(\FormulasR\)-program obtained from \(\Fa{G}\) by replacing
\begin{itemize}
\item all comparisons by \(\top\) and
\item
all atoms of form \(\alpha\sigma\) by the corresponding formulas \(\TransAD{G}{a\sigma}\)
with \(G = \AggrElem{r}{a}{\Fe{G}}{\sigma}\)
where \(\alpha\) is an atom of form~\eqref{eq:rewrite:aggregate:atom} for aggregate~\(a\) in rule~\(r\)
and~\(\sigma\) is a ground substitution for~\(r\) mapping all global variables in~\(a\) to ground terms.
\end{itemize}
\end{definition}

\begin{example}\label{ex:rewrite:assemble}
We show how to assemble aggregates using the sets \(G_1\) to \(G_3\) for aggregate atoms that have been propagated in \cref{ex:rewrite:progagate}.
Therefore, let \(\Fa{G}\) be the program consisting of the following rules:
\begin{align*}
\Pred{controls}(c_1,c_2) \leftarrow {} & \AggregateAtom{a}{r}{c_1,c_2} \wedge {} \Pred{company}(c_1) \wedge \Pred{company}(c_2) \wedge c_1 \neq c_2,\\
\Pred{controls}(c_3,c_4) \leftarrow {} & \AggregateAtom{a}{r}{c_3,c_4} \wedge {} \Pred{company}(c_3) \wedge \Pred{company}(c_4) \wedge c_3 \neq c_4, \AndT\\
\Pred{controls}(c_1,c_3) \leftarrow {} & \AggregateAtom{a}{r}{c_1,c_3} \wedge {} \Pred{company}(c_1) \wedge \Pred{company}(c_3) \wedge c_1 \neq c_3.
\end{align*}

Then, program \(\Assemble{\Fa{G}}{\Fe{G}}\) consists of the following rules:
\begin{align*}
\Pred{controls}(c_1,c_2) \leftarrow {} & \TransAD{G_1}{a\sigma_1} \wedge {} \Pred{company}(c_1) \wedge \Pred{company}(c_2) \wedge \top,\\
\Pred{controls}(c_3,c_4) \leftarrow {} & \TransAD{G_2}{a\sigma_2} \wedge {} \Pred{company}(c_3) \wedge \Pred{company}(c_4) \wedge \top, \AndT\\
\Pred{controls}(c_1,c_3) \leftarrow {} & \TransAD{G_3}{a\sigma_3} \wedge {} \Pred{company}(c_1) \wedge \Pred{company}(c_3) \wedge \top.
\end{align*}
\end{example}

The next result shows how a (non-ground) aggregate program \(P\) is transformed into a (ground) \(\FormulasR\)-program \(\Simp{\TransR{P}{J}}{I,J}\)
in the context of certain and possible atoms \((I,J)\)
via the interplay of
grounding \(\Fv{P}\) and \(\Fe{P}\),
deriving aggregate placeholders from their ground instances \(\Fv{G}\) and \(\Fe{G}\),
and finally replacing them in \(\Fa{G}\) by the original aggregates' contents.

\begin{proposition}\label{lem:propagate}
Let \(P\) be an aggregate program,
\((I,J)\) be a finite four-valued interpretation,
\(\Fv{G} = \GrdSimp{\Fv{P}}{I,J}\),
\(\Fe{G} = \GrdSimp{\Fe{P}}{I,J}\),
\(\Ia{J} = \Propagate{P}{\Fv{G}}{\Fe{G}}{I}{J}\), and
\(\Fa{G} = \GrdSimp{\Fa{P}}{I,J \cup \Ia{J}}\).

Then,
\begin{alphaenum}
\item \(\Assemble{\Fa{G}}{\Fe{G}} = \Simp{\TransR{P}{J}}{I,J}\) and\label{lem:propagate:a}
\item \(\Head{\Fa{G}} = T_{\IDReductP{\TransA{P}}{I}}(J)\).\label{lem:propagate:b}
\end{alphaenum}
\end{proposition}

\cref{lem:propagate:b} highlights the relation of the possible atoms contributed by \Fa{G} to
a corresponding application of the immediate consequence operator.
In fact, this is the first of three such relationships between grounding algorithms and consequence operators.

\begin{algorithm}[ht]
\caption{Grounding Components\label{fun:ground:component}}
\Fn{\(\GroundComponent{P}{I}{J}\)}{\((\Fa{G}, \Fv{G}, \Fe{G}, f, \Ia{J}, \Ipa{J}, J') \leftarrow (\emptyset, \emptyset, \emptyset, \True, \emptyset, \emptyset, \emptyset)\)\;
    \Repeat{\(J'=J\)}{\label{fun:ground:component:loop:begin}
      \tcp{ground aggregate elements}
      \(\Fv{G} \leftarrow \Fv{G} \cup \bigcup_{r\in\Fv{P}} \GroundRule{r}{I}{J}{J'}{\iota}{\Body{r}}{f}\)\;\label{fun:ground:component:empty}
      \(\Fe{G} \leftarrow \Fe{G} \cup \bigcup_{r\in\Fe{P}} \GroundRule{r}{I}{J}{J'}{\iota}{\Body{r}}{f}\)\;\label{fun:ground:component:elements}
      \tcp{propagate aggregates}
      \(\Ia{J} \leftarrow \Propagate{P}{\Fv{G}}{\Fe{G}}{I}{J}\)\;\label{fun:ground:component:propagate}
      \tcp{ground remaining rules}
      \(\Fa{G} \leftarrow \Fa{G} \cup \bigcup_{r \in \Fa{P}}\GroundRule{r}{I}{J \cup \Ia{J}}{J' \cup \Ipa{J}}{\iota}{\Body{r}}{f}\)\;\label{fun:ground:component:aggregate}
      \((f, \Ipa{J}, J', J) \leftarrow (\False, \Ia{J}, J, J \cup \Head{\Fa{G}})\)\;\label{fun:ground:component:variables:set}
    }\label{fun:ground:component:loop:end}
    \Return\ \(\Assemble{\Fa{G}}{\Fe{G}}\)\;\label{fun:ground:component:assemble}
}
\end{algorithm}

Let us now turn to grounding components of instantiation sequences in \cref{fun:ground:component}.
The function \KwGroundComponent\ takes an aggregate program \(P\) along with two sets \(I\) and \(J\) of ground atoms.
Intuitively, \(P\) is a component in a (refined) instantiation sequence
and
\(I\) and \(J\) form a four-valued interpretation \((I,J)\)
comprising the certain and possible atoms gathered while grounding previous components
(although their roles get reversed in \cref{fun:ground:program}).
After variable initialization,
\KwGroundComponent\ loops over consecutive rule instantiations in \(\Fa{P}\), \(\Fv{P}\), and \(\Fe{P}\)
until no more possible atoms are obtained.
In this case, it returns in \cref{fun:ground:component:assemble} the \FormulasR-program obtained from \(\Fa{G}\) by replacing all
ground aggregate placeholders of form~\eqref{eq:rewrite:aggregate:atom}
with the \(\FormulasR\)-formula corresponding to the respective ground aggregate.
The body of the loop can be divided into two parts:
\crefrange{fun:ground:component:empty}{fun:ground:component:propagate} deal with aggregates and
\cref{fun:ground:component:aggregate,fun:ground:component:variables:set} care about grounding the actual program.
In more detail,
\cref{fun:ground:component:empty,fun:ground:component:elements} instantiate programs \(\Fv{P}\) and \(\Fe{P}\),
whose ground instances, \(\Fv{G}\) and \(\Fe{G}\), are then used in \cref{fun:ground:component:propagate} to
derive ground instances of aggregate placeholders of form~\eqref{eq:rewrite:aggregate:atom}.
The grounded placeholders are then added via variable \(\Ia{J}\) to the possible atoms \(J\) when grounding
the actual program \(\Fa{P}\) in \cref{fun:ground:component:aggregate},
where \(J'\) and \(\Ipa{J}\) hold the previous value of \(J\) and \(\Ia{J}\), respectively.
For the next iteration,
\(J\) is augmented in \cref{fun:ground:component:variables:set} with all rule heads in \(\Fa{G}\)
and the flag \(f\) is set to false.
Recall that the purpose of \(f\) is to ensure that initially all rules are grounded.
In subsequent iterations,
duplicates are omitted by setting the flag to false and filtering rules whose positive bodies are a subset of the atoms \(J' \cup \Ipa{J}\) used in previous iterations.

While the inner workings of \cref{fun:ground:component} follow the blueprint given by \cref{lem:propagate}.
its outer functionality boils down to applying the stable operator of the corresponding ground program in the
context of the certain and possible atoms gathered so far.

\begin{proposition}\label{prp:ground:component}
Let \(P\) be an aggregate program, \((\Ir{I},\Ir{J})\) be a finite four-valued interpretation,
and \(J = \StableR{\TransA{P}}{\Ir{J}}{\Ir{I}}\).

Then,
\begin{alphaenum}
\item \(\GroundComponent{P}{\Ir{I}}{\Ir{J}}\) terminates iff \(J\) is finite.\label{prp:ground:component:a}
\end{alphaenum}

If \(J\) is finite, then
\begin{alphaenum}[resume]
\item \(\GroundComponent{P}{\Ir{I}}{\Ir{J}} = \Simp{\TransR{P}{\Ir{J} \cup J}}{\Ir{I}, \Ir{J} \cup J}\) and\label{prp:ground:component:b}
\item \(\Head{\GroundComponent{P}{\Ir{I}}{\Ir{J}}} = J\).\label{prp:ground:component:c}
\end{alphaenum}
\end{proposition}

\begin{algorithm}[ht]
\caption{Grounding Programs}\label{fun:ground:program}
\Fn{\(\GroundProgram{P}\)}{\Let\ \({(\Fi{P}{i})}_{i\in\Index}\) be a refined instantiation sequence for \(P\)\;\label{fun:ground:program:sequence}
  \((F,G) \leftarrow (\emptyset,\emptyset)\)\;
  \ForEach{\(i \in \Index\)}{\label{fun:ground:program:loop:begin}
    \Let\ \(\Fpi{P}{i}\) be the program obtained from \(\Fi{P}{i}\) as in \cref{def:approximate}\;\label{fun:ground:program:rewrite}
    \(F \leftarrow F \cup \GroundComponent{\Fpi{P}{i}}{\Head{G}}{\Head{F}}\)\;\label{fun:ground:program:true}
    \(G \leftarrow G \cup \GroundComponent{\Fi{P}{i}}{\Head{F}}{\Head{G}}\)\;\label{fun:ground:program:possible}\label{fun:ground:program:loop:end}
  }
  \Return\ \(G\)\;\label{fun:ground:program:return}
}
\end{algorithm}

Finally,
\cref{fun:ground:program} grounds an aggregate program by iterating over the components of one of its refined instantiation sequences.
Just as \cref{fun:ground:component} reflects the application of a stable operator,
function \KwGroundProgram\ follows the definition of an approximate model when grounding a component (cf.\ \cref{def:approximate}).
At first, facts are computed in \cref{fun:ground:program:true} by using the program stripped from
rules being involved in a negative cycle overlapping with the present or subsequent components.
The obtained head atoms are then used in \cref{fun:ground:program:possible} as certain context atoms when computing the ground version of the component at hand.
The possible atoms are provided by the head atoms of the ground program built so far,
and with roles reversed in \cref{fun:ground:program:true}.
Accordingly, the whole iteration aligns with the approximate model of the chosen refined instantiation
sequence (cf.\ \cref{def:approximate:sequence}), as made precise next.

Our grounding algorithm computes implicitly the approximate model of the chosen instantiation sequence
and outputs the corresponding ground program;
it terminates whenever the approximate model is finite.

\begin{theorem}\label{thm:ground:program}
Let \(P\) be an aggregate program, \({(\Fi{P}{i})}_{i \in \Index}\) be a refined instantiation sequence for \(P\),
and \((\Iri{I}{i},\Iri{J}{i})\) and \((\Ii{I}{i},\Ii{J}{i})\) be defined as in \cref{def:approximate:sequence:relative,def:approximate:sequence:intermediate}.

If \({(\Fi{P}{i})}_{i \in \Index}\) is selected by \cref{fun:ground:program} in \cref{fun:ground:program:sequence},
then we have that
\begin{alphaenum}
\item
the call \(\GroundProgram{P}\) terminates iff \(\ApproxM{{(\Fi{P}{i})}_{i\in\Index}}\) is finite, and\label{thm:ground:program:a}
\item
if \(\ApproxM{{(\Fi{P}{i})}_{i\in\Index}}\) is finite, then
\(\GroundProgram{P} = \bigcup_{i \in \Index}\Simp{\TransR{\Fi{P}{i}}{\Iri{J}{i} \cup \Ii{J}{i}}}{(\Iri{I}{i},\Iri{J}{i}) \sqcup (\Ii{I}{i},\Ii{J}{i})}\).\label{thm:ground:program:b}
\end{alphaenum}
\end{theorem}

As already indicated by \cref{thm:approximate:sequence:simplification}, grounding is governed by the series of consecutive
approximate models \((\Ii{I}{i},\Ii{J}{i})\) in~\eqref{def:approximate:sequence:intermediate}
delineating the stable models of each ground program in a (refined) instantiation sequence.
Whenever each of them is finite, we also obtain a finite grounding of the original program.
Note that the entire ground program is composed of the ground programs of each component in the chosen instantiation sequence.
Hence, different sequences may result in different overall ground programs.

Most importantly, our grounding machinery guarantees that an obtained finite ground program has the
same stable models as the original non-ground program.

\begin{corollary}[Main result]\label{col:main-result}
Let \(P\) be an aggregate program.

If \(\GroundProgram{P}\) terminates, then \(P\) and \(\GroundProgram{P}\) have the same well-founded and stable models.
\end{corollary}

\begin{example}\label{ex:ground:normal}
\begin{table}[ht]
\begin{subtable}[2.3cm]{.5\linewidth}
\centering
\vphantom{\includegraphics{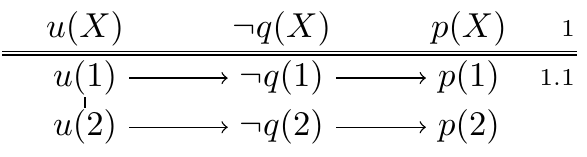}}
\raisebox{.5cm}{\(\Fpi{P}{5,1} = \emptyset\)}\\
\caption{\(\Ii{I}{5,1} = \Ii{I}{4} \cup \Head{\GroundComponentAbbrev{\Fpi{P}{5,1}}{\Ii{J}{4}}{\Ii{I}{4}}}\)\label{tab:ground:normal:a}}
\end{subtable}\begin{subtable}[2.3cm]{.5\linewidth}
\centering
\includegraphics{grd_J51}
\caption{\(\Ii{J}{5,1} = \Ii{J}{4} \cup \Head{\GroundComponentAbbrev{\vphantom{\Fpi{P}{5,1}}\Fi{P}{5,1}}{\Ii{I}{5,1}}{\Ii{J}{4}}}\)\label{tab:ground:normal:b}}
\end{subtable}

\begin{subtable}[2.3cm]{.5\linewidth}
\centering
\includegraphics{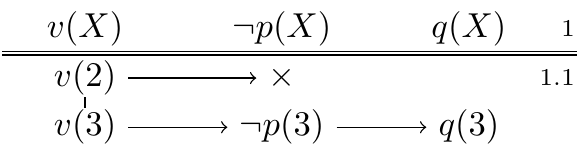}
\caption{\(\Ii{I}{5,2} = \Ii{I}{5,1} \cup \Head{\GroundComponentAbbrev{\Fpi{P}{5,2}}{\Ii{J}{5,1}}{\Ii{I}{5,1}}}\)\label{tab:ground:normal:c}}
\end{subtable}\begin{subtable}[2.3cm]{.5\linewidth}
\centering
\includegraphics{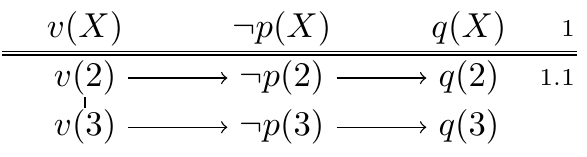}
\caption{\(\Ii{J}{5,2} = \Ii{J}{5,1} \cup \Head{\GroundComponentAbbrev{\vphantom{\Fpi{P}{5,2}}\Fi{P}{5,2}}{\Ii{I}{5,2}}{\Ii{J}{5,1}}}\)\label{tab:ground:normal:d}}
\end{subtable}

\begin{subtable}[1.8cm]{.5\linewidth}
\centering
\includegraphics{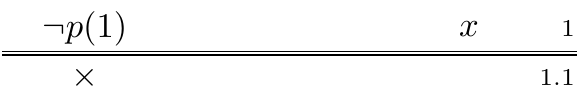}
\caption{\(\Ii{I}{6} = \Ii{I}{5,2} \cup \Head{\GroundComponentAbbrev{\Fpi{P}{6}}{\Ii{J}{5,2}}{\Ii{I}{5,2}}}\)\label{tab:ground:normal:e}}
\end{subtable}\begin{subtable}[1.8cm]{.5\linewidth}
\centering
\includegraphics{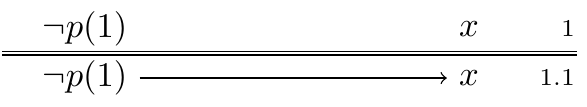}
\caption{\(\Ii{J}{6} = \Ii{J}{5,2} \cup \Head{\GroundComponentAbbrev{\vphantom{\Fpi{P}{6}}\Fi{P}{6}}{\Ii{I}{6}}{\Ii{J}{5,2}}}\)\label{tab:ground:normal:f}}
\end{subtable}

\begin{subtable}[1.8cm]{.5\linewidth}
\centering
\includegraphics{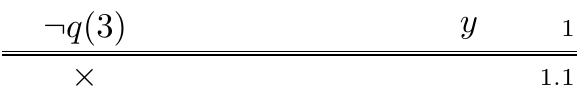}
\caption{\(\Ii{I}{7} = \Ii{I}{6} \cup \Head{\GroundComponentAbbrev{\Fpi{P}{7}}{\Ii{J}{6}}{\Ii{I}{6}}}\)\label{tab:ground:normal:g}}
\end{subtable}\begin{subtable}[1.8cm]{.5\linewidth}
\centering
\includegraphics{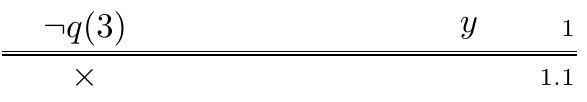}
\caption{\(\Ii{J}{7} = \Ii{J}{6} \cup \Head{\GroundComponentAbbrev{\vphantom{\Fpi{P}{7}}\Fi{P}{7}}{\Ii{I}{7}}{\Ii{J}{6}}}\)\label{tab:ground:normal:h}}
\end{subtable}

\caption{Grounding of components \(\Fi{P}{5,1}\), \(\Fi{P}{5,2}\), \(\Fi{P}{6}\), and \(\Fi{P}{7}\) from \Cref{ex:dependency} where \(\protect\KwGroundComponentShort = \protect\KwGroundComponent\).\label{tab:ground:normal}}
\end{table}
The execution of the grounding algorithms on \cref{ex:dependency} is illustrated in \cref{tab:ground:normal}.
Each individual table depicts a call to \(\KwGroundComponent\)
where the header above the double line contains the (literals of the) rules to be grounded
and the rows below trace how nested calls to \(\KwGroundRule\) proceed.
The rules in the header contain the body literals in the order as they are selected by \(\KwGroundRule\)
with the rule head as the last literal.
Calls to \(\KwGroundRule\) are depicted with vertical lines and horizontal arrows.
A vertical line represents the iteration of the loop in \crefrange{fun:ground:rule:loop:begin}{fun:ground:rule:loop:end}.
A horizontal arrow represents a recursive call to \(\KwGroundRule\) in \cref{fun:ground:rule:recurse}.
Each row in the table not marked with \(\times\) corresponds to a ground instance as returned by \(\KwGroundRule\).
Furthermore, because all components are stratified,
we only show the first iteration of the loop in \crefrange{fun:ground:component:loop:begin}{fun:ground:component:loop:end} of \cref{fun:ground:component}
as the second iteration does not produce any new ground instances.

Grounding components \(\Fi{P}{1}\) to \(\Fi{P}{4}\) results in the programs \(F = G = \{ u(1) \leftarrow \top, u(2) \leftarrow \top, v(2) \leftarrow \top, v(3) \leftarrow \top \}\).
Since grounding is driven by the sets of true and possible atoms, we focus on the interpretations \(\Ii{I}{i}\) and \(\Ii{J}{i}\) where \(i\) is a component index in the refined instantiation sequence.
We start tracing the grounding starting with \(\Ii{I}{4} = \Ii{J}{4} = \{ u(1), u(2), v(2), v(3) \}\).

The grounding of \(\Fi{P}{5,1}\) is depicted in \cref{tab:ground:normal:a,tab:ground:normal:b}.
We have \(\PredE = \{ b/1 \}\) because predicate \(b/1\) is used in the head of the rule in \(\Fi{P}{5,2}\).
Thus, \(\GroundComponent{\emptyset}{\Ii{J}{4}}{\Ii{I}{4}}\) in \cref{fun:ground:program:true} returns the empty set because \(\Fpi{P}{5,1} = \emptyset\).
We get \(\Ii{I}{5,1} = \Ii{I}{4}\).
In the next line, the algorithm calls \(\GroundComponent{\Fi{P}{5,1}}{\Ii{I}{5,1}}{\Ii{J}{4}}\) and we get \(\Ii{J}{5,1} = \{ p(1), p(2) \}\).
Note that at this point, it is not known that \(q(1)\) is not derivable and so the algorithm does not derive \(p(1)\) as a fact.

The grounding of \(\Fi{P}{5,2}\) is given in \cref{tab:ground:normal:c,tab:ground:normal:d}.
This time, we have \(\PredE = \emptyset\) and \(\Fi{P}{5,2} = \Fpi{P}{5,2}\).
Thus, the first call to \(\KwGroundRule\) determines \(q(3)\) to be true
while the second call additionally determines the possible atom \(q(2)\).

The grounding of \(\Fi{P}{6}\) is illustrated in \cref{tab:ground:normal:e,tab:ground:normal:f}.
Note that we obtain that \(x\) is possible because \(p(1)\) was not determined to be true.

The grounding of \(\Fi{P}{7}\) is depicted in \cref{tab:ground:normal:g,tab:ground:normal:h}.
Note that, unlike before, we obtain that \(y\) is false because \(q(3)\) was determined to be true.

Furthermore, observe that the choice of the refined instantiation sequence determines the output of the algorithm.
In fact, swapping \(\Fi{P}{5,1}\) and \(\Fi{P}{5,2}\) in the sequence would result in \(x\) being false and \(y\) being possible.

\begin{table}
\begin{align*}
u(1) & &
u(2) & \\
v(2) & &
v(3) & \\
p(1) & \leftarrow \Naf q(1) \wedge {\DG u(1)} &
p(2) & \leftarrow \Naf q(2) \wedge {\DG u(2)} \\
q(2) & \leftarrow \Naf p(2) \wedge {\DG v(2)} &
q(3) & \leftarrow {\DG \Naf p(3)} \wedge {\DG v(3)} \\
x & \leftarrow \Naf p(1)
\end{align*}
\caption{Grounding of \cref{ex:dependency} as output by \gringo.}\label{ex:ground:normal:gringo}
\end{table}

To conclude, we give the grounding of the program as output by \gringo\ in \cref{ex:ground:normal:gringo}.
The grounder furthermore omits the true body literals marked in green.
\end{example}

\begin{example}
We illustrate the grounding of aggregates on the company controls example in \cref{ex:company}
using the grounding sequence \(\Fsq{\Fi{P}{i}}{1 \leq i \leq 9}\) and the set of facts~\(F\) from \cref{ex:dependency:company}.
Observe that the grounding of components \(\Fi{P}{1}\) to \(\Fi{P}{8}\) produces the program \(\{ a \leftarrow \top \mid a \in F \}\).
We focus on how component \(\Fi{P}{9}\) is grounded.
Because there are no negative dependencies,
the components \(\Fi{P}{9}\) and \(\Fpi{P}{9}\) in \cref{fun:ground:program:rewrite} of \cref{fun:ground:program} are equal.
To ground component \(\Fi{P}{9}\), we use the rewriting from \cref{ex:rewrite}.

\begin{table}
\includegraphics{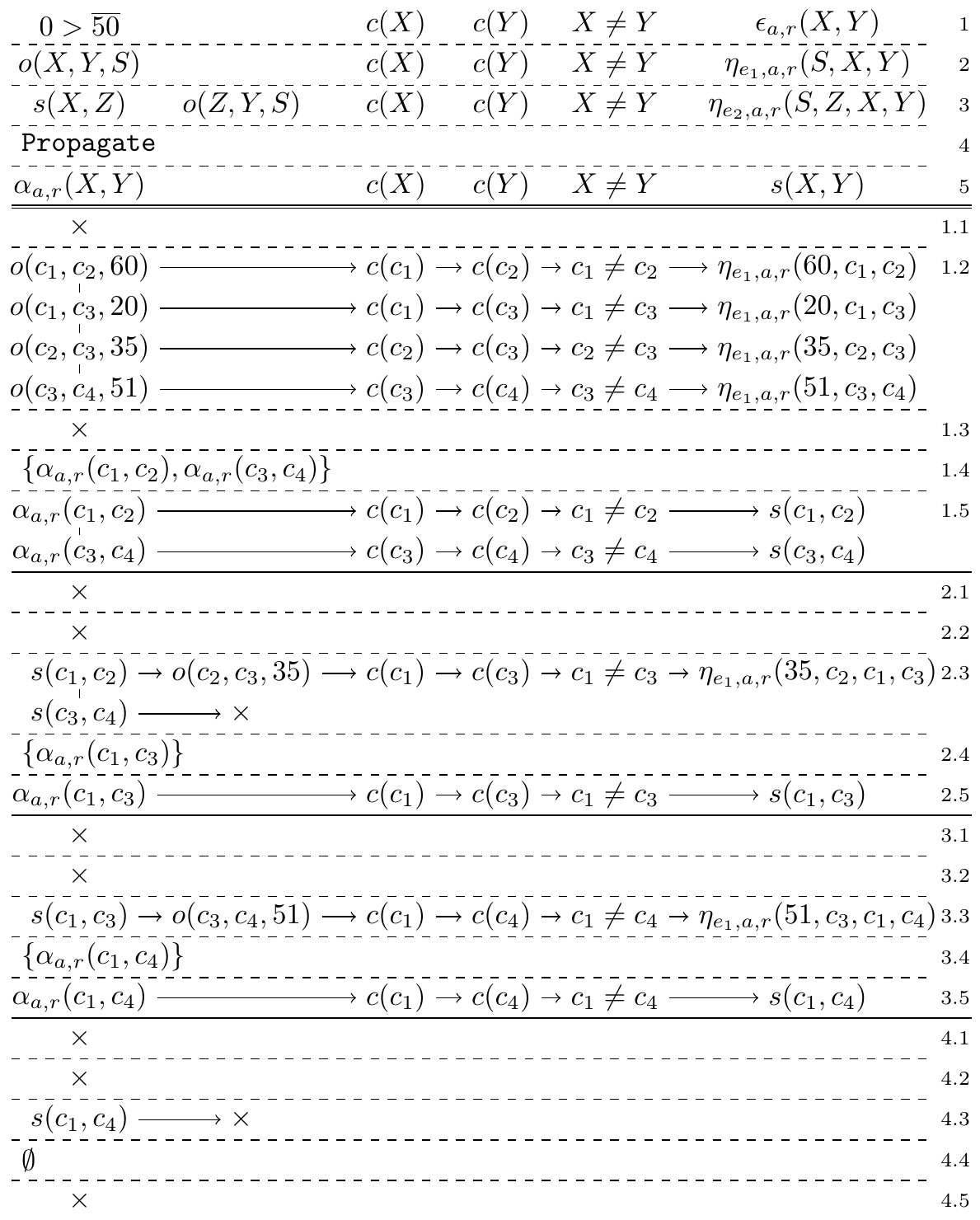}
\caption{Tracing grounding of component \(\Fi{P}{9}\) where \(c=\Pred{company}\), \(o=\Pred{owns}\), and \(s=\Pred{controls}\).\label{tab:ground:company}}
\end{table}

The grounding of component \(\Fi{P}{9}\) is illustrated in \cref{tab:ground:company},
which follows the same conventions as in \cref{ex:ground:normal}.
Because the program is positive,
the calls in \cref{fun:ground:program:true,fun:ground:program:possible} in \cref{fun:ground:program} proceed in the same way
and we depict only one of them.
Furthermore, because this example involves a recursive rule with an aggregate,
the header consists of five rows separated by dashed lines
processed by \cref{fun:ground:component}.
The first row corresponds to \(\Fvi{P}{9}\) grounded in \cref{fun:ground:component:empty},
the second and third to \(\Fei{P}{9}\) grounded in \cref{fun:ground:component:elements},
the fourth to aggregate propagation in \cref{fun:ground:component:propagate}, and
the fifth to \(\Fai{P}{9}\) grounded in \cref{fun:ground:component:aggregate}.
After the header follow the iterations of the loop in \crefrange{fun:ground:component:loop:begin}{fun:ground:component:loop:end}.
Because the component is recursive, processing the component requires four iterations, which are separated by solid lines in the table.
The right-hand side column of the table contains the iteration number and a number indicating which row in the header is processed.
The row for aggregate propagation lists the aggregate atoms that have been propagated.

The grounding of rule~\(r_2\) in Row~1.1 does not produce any rule instances in any iteration
because the comparison~\(0 > 50\) is false.
By first selecting this literal when grounding the rule,
the remaining rule body can be completely ignored.
Actual systems implement heuristics to prioritize such literals.
Next, in the grounding of rule~\(r_3\) in Row~1.2,
direct shares given by facts over \(\Pred{owns}/3\) are accumulated.
Because the rule does not contain any recursive predicates,
it only produces ground instances in the first iteration.
Unlike this, rule \(r_4\) contains the recursive predicate \(\Pred{controls/2}\).
It does not produce instances in the first iteration in Row~1.3 because there are no corresponding atoms yet.
Next, aggregate propagation is triggered in Row~1.4,
resulting in aggregate atoms~\(\AggregateAtom{a}{r}{c_1,c_2}\) and \(\AggregateAtom{a}{r}{c_3,c_4}\),
for which enough shares have been accumulated in Row~1.2.
Note that this corresponds to the propagation of the sets~\(G_1\) and~\(G_2\) in \cref{ex:rewrite:progagate}.
With these atoms,
rule \(r_1\) is instantiated in Row~1.5,
leading to new atoms over \(\Pred{controls}/2\).
Observe that, by selecting atom \(\AggregateAtom{a}{r}{X,Y}\) first,
\(\KwGroundRule\) can instantiate the rule without backtracking.

In the second iteration,
the newly obtained atoms over predicate \(\Pred{controls}/2\) yield atom \(\ElementAtom{e_1}{a}{r}{35,c_2,c_1,c_3}\) in Row~2.3,
which in turn leads to the aggregate atom \(\AggregateAtom{a}{r}{c_1,c_3}\)
resulting in further instances of \(r_4\).
Note that this corresponds to the propagation of the set~\(G_3\) in \cref{ex:rewrite:progagate}.

The following iterations proceed in a similar fashion until no new atoms
are accumulated and the grounding loop terminates.
Note that the utilized selection strategy affects the amount of backtracking in rule instantiation.
One particular strategy used in \gringo\ is to prefer atoms over recursive predicates.
If there is only one such atom,
\(\KwGroundRule\) can select this atom first and only has to consider newly derived atoms for instantiation.
The table is made more compact by applying this strategy.
Further techniques are available in the database literature~\cite{ullman88a}
that also work in case of multiple atoms over recursive predicates.

To conclude, we give the ground rules as output by a grounder like \gringo\ in \cref{ex:company:grounding}.
We omit the translation of aggregates because its main objective is to show correctness of the algorithms.
Solvers like \clasp\ implement translations or even native handling of aggregates geared toward efficient solving~\cite{gekakasc09a}.
Since our example program is positive, \gringo\ is even able to completely evaluate the rules to facts omitting true literals from rule bodies marked in green.
\begin{table}
\begin{align*}
\Pred{controls}(c_1,c_2) \leftarrow {}
& \DG \Sum^+ \{ 60 : \Pred{owns}(c_1,c_2,60) \} > 50 \\
{} \wedge {}
& {\DG \Pred{company}(c_1)} \wedge {\DG \Pred{company}(c_2)} \wedge {\DG c_1 \neq c_2} \\
\Pred{controls}(c_3,c_4) \leftarrow {}
& \DG \Sum^+ \{ 51 : \Pred{owns}(c_3,c_4,51) \} > 50 \\
{} \wedge {}
& {\DG \Pred{company}(c_3)} \wedge {\DG \Pred{company}(c_4)} \wedge {\DG c_3 \neq c_4} \\
\Pred{controls}(c_1,c_3) \leftarrow {}
& \!\begin{aligned}[t]
\DG \Sum^+ \{
& \DG 20 : \Pred{owns}(c_1,c_3,20); \\
& \DG 35,c_2: \Pred{controls}(c_1,c_2) \wedge \Pred{owns}(c_2,c_3,35) \} > 50
\end{aligned} \\
{} \wedge {}
& {\DG \Pred{company}(c_1)} \wedge {\DG \Pred{company}(c_3)} \wedge {\DG c_1 \neq c_3} \\
\Pred{controls}(c_1,c_4) \leftarrow {}
& \DG \Sum^+ \{ 51,c_3: \Pred{controls}(c_1,c_3) \wedge \Pred{owns}(c_3,c_4,51) \} > 50 \\
{} \wedge {}
& {\DG \Pred{company}(c_1)} \wedge {\DG \Pred{company}(c_4)} \wedge {\DG c_1 \neq c_4}
\end{align*}
\caption{Grounding of the company controls problem from \cref{ex:company} as output by \gringo.}\label{ex:company:grounding}
\end{table}

\end{example}

 \section{Refinements}\label{sec:refinements}

Up to now, we were primarily concerned by characterizing the theoretical and algorithmic cornerstones of grounding.
This section refines these concepts by further detailing aggregate propagation, algorithm specifics, and the treatment
of language constructs from \gringo's input language.

\subsection{Aggregate propagation}\label{sec:propagate-aggregates}

We used in \cref{sec:algorithms} the relative translation of aggregates for propagation,
namely, formula \(\TransAD{G}{a\sigma}\) in \cref{def:propagate}, to check whether an aggregate is satisfiable.
In this section, we identify several aggregate specific properties that allow us to implement more efficient algorithms to perform this check.

To begin with,
we establish some properties that greatly simplify the treatment of (arbitrary) monotone or antimonotone aggregates.

We have already seen in \cref{prp:translation:monotonicity} that
\(\IDReductP{\TransA{a}}{I}\) is classically equivalent to \(\TransA{a}\)
for any closed aggregate \(a\) and two-valued interpretation~\(I\).
Here is its counterpart for antimonotone aggregates.

\begin{proposition}\label{prp:translation:antimonotonicity}
Let \(a\) be a closed aggregate.

If \(a\) is antimonotone, then
\(\IDReductP{\TransA{a}}{I}\) is classically equivalent to \(\top\) if \(I \models \TransA{a}\) and \(\bot\) otherwise
for any two-valued interpretation~\(I\).
\end{proposition}

\begin{example}
In \cref{ex:rewrite:progagate}, we check whether the interpretation \(J\) satisfies the formulas~\(\IDReductP{\TransAD{G_1}{a\sigma_1}}{I}\) to \(\IDReductP{\TransAD{G_4}{a\sigma_4}}{I}\).

Using \cref{prp:translation:monotonicity}, this boils down to checking \(\sum_{e \in G_i, J \models \Body{e}}\Head{e} > 50\) for each \(1 \leq i \leq 4\).
We get \(60 > 50\), \(51 > 50\), \(55 > 50\), and \(35 \not> 50\) for each  \(G_i\),
which agrees with checking \(J \models \IDReductP{\TransAD{G_i}{a\sigma_i}}{I}\).
\end{example}

An actual implementation can maintain a counter for the current value of the sum for each closed aggregate instance,
which can be updated incrementally and compared with the bound as new instances of aggregate elements are grounded.

Next, we see that such counter based implementations are also possible \(\Sum\) aggregates
using the \(<\), \(\leq\), \(>\), or \(\geq\) relations.
We restrict our attention to finite interpretations
because \cref{prp:sum-aggregate-pm} is intended to give an idea on
how to implement an actual propagation algorithm for aggregates
(infinite interpretations would add more special cases).
Furthermore, we just consider the case that the bound is an integer here;
the aggregate is constant for any other ground term.

\begin{proposition}\label{prp:sum-aggregate-pm}
Let \(I\) be a finite two-valued interpretation,
\(E\) be a set of aggregate elements, and
\(b\) be an integer.

For \(T = \Tuple{\{e \in \Grd{E} \mid I \models \Cond{e}\}}\),
we get
\begin{alphaenum}
\item
\(\IDReductP{\TransA{\SumA{\ElemF}{\succ}{b}}}{I}\) is classically equivalent to \(\TransA{\SumPA{\ElemF}{\succ}{b'}}\)\\ with
\({\succ} \in \{ {\geq}, {>} \}\) and
\(b' = b - \SumM(T)\), and\label{prp:sum-aggregate-pm:a}
\item
\(\IDReductP{\TransA{\SumA{\ElemF}{\prec}{b}}}{I}\) is classically equivalent to \(\TransA{\SumMA{\ElemF}{\prec}{b'}}\)\\ with
\({\prec} \in \{ {\leq}, {<} \}\) and
\(b' = b - \SumP(T)\).\label{prp:sum-aggregate-pm:b}
\end{alphaenum}
\end{proposition}

The remaining propositions identify properties that can be exploited when propagating aggregates over the \(=\) and \(\neq\) relations.

\begin{proposition}\label{prp:aggregate-generic-implies}
Let \(I\) be a two-valued interpretation,
\(E\) be a set of aggregate elements, and
\(b\) be a ground term.

We get the following properties:
\begin{alphaenum}
\item \(\IDReductP{\TransA{\Aggregate{f}{\ElemF}{<}{b}}}{I} \vee \IDReductP{\TransA{\Aggregate{f}{\ElemF}{>}{b}}}{I}
  \text{ implies }
  \IDReductP{\TransA{\Aggregate{f}{\ElemF}{\neq}{b}}}{I}\), and\label{prp:aggregate-generic-implies:a}
\item \(\IDReductP{\TransA{\Aggregate{f}{\ElemF}{=}{b}}}{I}
  \text{ implies }
  \IDReductP{\TransA{\Aggregate{f}{\ElemF}{\leq}{b}}}{I} \wedge \IDReductP{\TransA{\Aggregate{f}{\ElemF}{\geq}{b}}}{I}\).\label{prp:aggregate-generic-implies:b}
\end{alphaenum}
\end{proposition}

The following proposition identifies special cases when the implications in \cref{prp:aggregate-generic-implies} are equivalences.
Another interesting aspect of this proposition is that we can actually replace \(\Sum\) aggregates over \(=\) and \(\neq\) with a conjunction or disjunction, respectively,
at the expense of calculating a less precise approximate model.
The conjunction is even strongly equivalent to the original aggregate under Ferraris' semantics but not the disjunction.

\begin{proposition}\label{prp:sum-aggregate-implies}
Let \(I\) and \(J\) be two-valued interpretations,
\(f\) be an aggregate function among \(\Count\), \(\SumP\), \(\SumM\) or \(\Sum\),
\(E\) be a set of aggregate elements, and
\(b\) be an integer.

We get the following properties:
\begin{alphaenum}
\item for \(I \subseteq J\), we have
\(J \models \IDReductP{\TransA{\Aggregate{f}{\ElemF}{<}{b}}}{I} \vee \IDReductP{\TransA{\Aggregate{f}{\ElemF}{>}{b}}}{I}\) iff
\(J \models \IDReductP{\TransA{\Aggregate{f}{\ElemF}{\neq}{b}}}{I}\), and\label{prp:sum-aggregate-implies:a}
\item for \(J \subseteq I\), we have
\(J \models \IDReductP{\TransA{\Aggregate{f}{\ElemF}{=}{b}}}{I}\) iff
\(J \models \IDReductP{\TransA{\Aggregate{f}{\ElemF}{\leq}{b}}}{I} \wedge \IDReductP{\TransA{\Aggregate{f}{\ElemF}{\geq}{b}}}{I}\).\label{prp:sum-aggregate-implies:b}
\end{alphaenum}
\end{proposition}

The following proposition shows
that full propagation of \(\Sum\), \(\SumP\), or \(\SumM\) aggregates over relations \(=\) and \(\neq\) involves solving the subset sum problem~\cite{martot90a}.
We assume that we propagate w.r.t.\ some polynomial number of aggregate elements.
Propagating possible atoms when using the \(=\) relation, i.e., when \(I \subseteq J\), involves deciding an NP problem
and propagating certain atoms when using the \(\neq\) relation, i.e., when \(J \subseteq I\), involves deciding a co-NP problem.\footnote{Note that \clingo's grounding algorithm does not attempt to solve these problems in all cases.
It simply over- or underapproximates the satisfiability using \cref{prp:aggregate-generic-implies}.}
Note that the decision problem for \(\Count\) aggregates is polynomial, though.

\begin{proposition}\label{prp:propagate-generic-iff}
Let \(I\) and \(J\) be finite two-valued interpretations,
\(f\) be an aggregate function,
\(E\) be a set of aggregate elements, and
\(b\) be a ground term.

For
\(T_I = \{ \Head{e} \mid e \in \Grd{E}, I \models \Body{e} \}\) and
\(T_J = \{ \Head{e} \mid e \in \Grd{E}, J \models \Body{e} \}\),
we get the following properties:
\begin{alphaenum}
\item for \(J \subseteq I\), we have \(J \models \IDReductP{\TransA{\Aggregate{f}{\ElemF}{\neq}{b}}}{I}\) iff
there is no set \(X \subseteq T_I\) such that \(f(X \cup T_J) = b\), and\label{prp:propagate-generic-iff:a}
\item for \(I \subseteq J\), we have \(J \models \IDReductP{\TransA{\Aggregate{f}{\ElemF}{=}{b}}}{I}\) and iff\label{prp:propagate-generic-iff:b}
there is a set \(X \subseteq T_J\) such that \(f(X \cup T_I) = b\).
\end{alphaenum}
\end{proposition}

\subsection{Algorithmic refinements}

The calls in \cref{fun:ground:program:true,fun:ground:program:possible} in \cref{fun:ground:program} can sometimes be combined to calculate certain and possible atoms simultaneously.
This can be done whenever a component does not contain recursive predicates.
In this case, it is sufficient to just calculate possible atoms along with rule instances in \cref{fun:ground:program:possible}
augmenting \cref{fun:ground:rule} with an additional check to detect whether a rule instance produces a certain atom.
Observe that this condition applies to all stratified components but can also apply to components depending on unstratified components.
In fact, typical programs following the generate, define, and test methodology~\cite{lifschitz02a,niemela08b} of ASP,
where the generate part uses choice rules~\cite{siniso02a} (see below),
do not contain unstratified negation at all.
When a grounder is combined with a solver built to store rules and perform inferences,
one can optimize for the case that there are no negative recursive predicates in a component.
In this case, it is sufficient to compute possible atoms along with their rule instances and leave the computation of certain atoms to the solver.
Finally, note that \gringo\ currently does not separate the calculation of certain and possible atoms
at the expense of computing a less precise approximate model and possibly additional rule instances.

\begin{example}
For the following example, \gringo\ computes atom~\(p(4)\) as unknown but the algorithms in \cref{sec:algorithms} identify it as true.
\begin{align*}
r(1,4) & & p(1) & \leftarrow \Naf q(1) \\
r(2,3) & & q(1) & \leftarrow \Naf p(1) \\
r(3,1) & & p(2) & \\
p(Y)   & \leftarrow \rlap{\(p(X) \wedge r(X,Y)\)}.
\end{align*}
When grounding the last rule, \gringo\ determines \(p(4)\) to be possible in the first iteration because \(p(1)\) is unknown at this point.
In the second iteration, it detects that \(p(1)\) is a fact but does not use it for grounding again.
If there were further rules depending negatively on predicate~\(p/1\), inapplicable rules might appear in \gringo's output.
\end{example}

Another observation is that the loop in \cref{fun:ground:component} does not produce further rule instances in a second iteration for components without recursive predicates.
\Gringo\ maintains an index~\cite{gaulwi09a} for each positive body literal to speed up matching of literals;
whenever none of these indexes, used in rules of the component at hand, are updated, further iterations can be skipped.

Just like \dlv's grounder, \gringo\ adapts algorithms for semi-naive evaluation from the field of databases.
In particular, it works best on linear programs~\cite{abhuvi95a},
having at most one positive literal occurrence over a recursive predicate in a rule body.
The program in \cref{ex:company} for the company controls problem is such a linear program
because \(\Pred{controls}/2\) is the only recursive predicate.
\Cref{fun:ground:rule} can easily be adapted to efficiently ground linear programs
by making sure that the recursive positive literal is selected first.
We then only have to consider matches that induce atoms not already used for instantiations in previous iterations of the loop in \cref{fun:ground:component} to reduce the amount of backtracking to find rule instances.
In fact, the order in which literals are selected in \cref{fun:ground:rule:select} is crucial for the performance of \Cref{fun:ground:rule}.
\Gringo\ uses an adaptation of the selection heuristics presented by~\citeN{lepesc01a} that additionally takes into account recursive predicates and terms with function symbols.

To avoid unnecessary backtracking when grounding general logic programs,
\gringo\ instantiates rules using an algorithm similar to the improved semi-naive evaluation with optimizations for linear rules~\cite{abhuvi95a}.

\subsection{Capturing \gringo's input language}\label{sec:gringo}

We presented aggregate programs where rule heads are simple atoms.
Beyond that, \gringo's input language offers more elaborate language constructs to ease modeling.

A prominent such construct are so-called choice rules~\cite{siniso02a}.
Syntactically, one-element choice rules have the form \(\{a\}\leftarrow B\), where \(a\) is an atom and \(B\) a body.
Semantically, such a rule amounts to \(a\vee\neg a\leftarrow B\) or equivalently \(a\leftarrow \neg\neg a\wedge B\).
We can easily add support for grounding choice rules,
that is, rules where the head is not a plain atom but an atom marked as a choice,
by discarding choice rules when calculating certain atoms and treating them like normal rules when grounding possible atoms.
A translation that allows for supporting head aggregates using a translation to aggregate rules and choice rules
is given by~\citeN{gehakalisc15a}.
Note that \gringo\ implements further refinements to omit deriving head atoms if a head aggregate cannot be satisfied.

Another language feature that can be instantiated in a similar fashion as body aggregates are conditional literals.
\Gringo\ adapts the rewriting and propagation of body aggregates to also support grounding of conditional literals.

Yet another important language feature are disjunctions in the head of rules~\cite{gellif91a}.
As disjunctive logic programs,
aggregate programs allow us to solve problems from the second level of the polynomial hierarchy.
In fact, using {\L}ukasiewicz' theorem~\cite{lukasiewicz41a}, we can write a disjunctive rule of form
\begin{align*}
a \vee b & \leftarrow B
\intertext{as the shifted strongly equivalent \FormulasR-program:}
a & \leftarrow (b \rightarrow a)\wedge B\\
b & \leftarrow (a \rightarrow b)\wedge B.
\end{align*}
We can use this as a template to design grounding algorithms for disjunctive programs.
In fact, \gringo\ calculates the same approximate model for the disjunctive rule and the shifted program.

The usage of negation as failure is restricted in \FormulasR-programs.
Note that any occurrence of a negated literal~\(l\) in a rule body can be replaced by an auxiliary atom~\(a\) adding rule \(a \leftarrow l\) to the program.
The resulting program preserves the stable models modulo the auxiliary atoms.
This translation can serve as a template for double negation or negation in aggregate elements as supported by \gringo.

Integrity constraints are a straightforward extension of logic programs.
They can be grounded just like normal rules deriving an auxiliary atom that stands for \(\bot\).
Grounding can be stopped whenever the auxiliary atom is derived as certain.
Integrity constraints also allow for supporting negated head atoms,
which can be shifted to rule bodies~\cite{janhunen01a} resulting in integrity constraints,
and then treated like negation in rule bodies.

A frequently used convenience feature of \gringo\ are term pools~\cite{PotasscoUserGuide,gehakalisc15a}.
The grounder handles them by removing them in a rewriting step.
For example, a rule of form
\begin{align*}
h(X;Y,Z) & \leftarrow p(X;Y), q(Z)
\intertext{is factored out into the following rules:}
h(X,Z) & \leftarrow p(X), q(Z)\\
h(X,Z) & \leftarrow p(Y), q(Z)\\
h(Y,Z) & \leftarrow p(X), q(Z)\\
h(Y,Z) & \leftarrow p(Y), q(Z)
\end{align*}
We can then apply the grounding algorithms developed in \cref{sec:algorithms}.

To deal with variables ranging over integers, \gringo\ supports interval terms~\cite{PotasscoUserGuide,gehakalisc15a}.
Such terms are handled by a translation to inbuilt range predicates.
For example the program
\begin{align*}
h(l..u) &
\intertext{for terms \(l\) and \(u\) is rewritten into}
h(A) & \leftarrow \Pred{rng}(A,l,u)
\end{align*}
by introducing auxiliary variable \(A\) and range atom \(\Pred{rng}(A,l,u)\).
The range atom provides matches including all substitutions that assign integer values between \(l\) and \(u\) to \(A\).
Special care has to be taken regarding rule safety,
the range atom can only provide bindings for variable \(A\) but needs variables in the terms \(l\) and \(u\) to be provided elsewhere.

A common feature used when writing logic programs are terms involving arithmetic expressions and assignments.
Both influence which rules are considered safe by the grounder.
For example, the rule
\begin{align*}
h(X,Y) & \leftarrow p(X+Y,Y) \wedge X = Y + Y
\intertext{is rewritten into}
h(X,Y) & \leftarrow p(A,Y) \wedge X = Y + Y \wedge A = X + Y
\end{align*}
by introducing auxiliary variable \(A\).
The rule is safe because we can match the literals in the order as given in the rewritten rule.
The comparison \(X = Y + Y\) extends the substitution with an assignment for \(X\) and the last comparison serves as a test.
\Gringo\ does not try to solve complicated equations but supports simple forms like the one given above.

Last but not least,
\gringo\ does not just support terms in assignments but it also supports aggregates in assignments.
To handle such kind of aggregates,
the rewriting and propagation of aggregates has to be extended.
This is achieved by adding an additional variable to aggregate replacement atoms~\eqref{eq:rewrite:aggregate:atom}, which is assigned by propagation.
For example, the rule
\begin{align*}
\rlap{\(\underbrace{h(X,Y) \leftarrow q(X) \wedge \overbrace{Y = \Sum \{ Z: p(X,Z) \}}^a}_r\)}
\phantom{h(X,Y)} & \phantom{\leftarrow q(X) \wedge \overbrace{Y = \Sum \{ Z: p(X,Z) \}}^a}
\intertext{is rewritten into}
\EmptyAtom{a}{r}{X} & \leftarrow q(X)\\
\ElementAtom{e}{a}{r}{Z,X} & \leftarrow p(X,Z) \wedge q(X)\\
h(X,Y) & \leftarrow \AggregateAtom{a}{r}{X,Y}
\end{align*}
Aggregate elements are grounded as before but values for variable \(Y\) are computed during aggregate propagation.
In case of multiple assignment aggregates, additional care has to to be taken during the rewriting to ensure that the rewritten rules are safe.

 \section{Related work}\label{sec:related}
\newcommand\lfp[1]{\ensuremath{\mathit{lfp}(#1)}}
\newcommand\MirekFittingS{\ensuremath{\mathit{\Phi}}}
\newcommand\MirekStableS{\ensuremath{\mathit{St}}}
\newcommand\MirekWellFoundedS{\ensuremath{\mathit{Wf}}}
\newcommand\MirekFitting[2]{\ensuremath{\MirekFittingS_{#1}(#2)}}
\newcommand\MirekStable[2]{\ensuremath{\MirekStableS_{#1}(#2)}}
\newcommand\MirekWellFounded[2]{\ensuremath{\MirekWellFoundedS_{#1}(#2)}}

This section aims at inserting our contributions into the literature, starting with theoretical aspects over
algorithmic ones to implementations.

Splitting for infinitary formulas has been introduced by~\citeN{harlif16a}
generalizing results of~\citeN{jaoitowo07a} and~\citeN{felelipa09a}.
To this end, the concept of an \(A\)-stable model is introduced~\cite{harlif16a}.
We obtain the following relationship between our definition of a stable model relative to a set \(\Ir{I}\) and \(A\)-stable models:
For an \FormulasN-program~\(P\), we have that
if \(X\) is a stable model of \(P\) relative to \(\Ir{I}\),
then \(X \cup \Ir{I}\) is an \((\Hb\setminus\Ir{I})\)-stable model of \(P\).
Similarly, we get that
if \(X\) is an \(A\)-stable model of \(P\),
then \(\StableR{P}{X \setminus A}{X}\) is a stable model of \(P\) relative to \(X \setminus A\).
The difference between the two concepts is that we fix atoms \(\Ir{I}\) in our definition while \(A\)-stable models
allow for assigning arbitrary truth values to atoms in \(\Hb \setminus A\)~\cite[Proposition~1]{harlif16a}.
With this, let us compare our handling of program sequences to symmetric splitting~\cite{harlif16a}.
Let \(\Fsqi{P}{\Index}{i}\) be a refined instantiation sequence of aggregate program \(P\), and
\(F = \bigcup_{i < j} \TransA{\Fi{P}{i}}\) and \(G = \bigcup_{i \geq j} \TransA{\Fi{P}{i}} \)
for some \(j \in \Index\) such that \(\Head{F} \neq \Head{G}\).
We can use the infinitary splitting theorem of~\citeN{harlif16a} to calculate the stable model of \(F^\wedge \wedge G^\wedge\)
through the \(\Head{F}\)- and \(\Hb \setminus \Head{F}\)-stable models of \(F^\wedge \wedge G^\wedge\).
Observe that instantiation sequences do not permit positive recursion between their components
and infinite walks are impossible because an aggregate program consists of finitely many rules inducing a finite dependency graph.
Note that we additionally require the condition \(\Head{F} \neq \Head{G}\)
because components can be split even if their head atoms overlap.
Such a split can only occur if overlapping head atoms in preceding components are not involved in positive recursion.

Next, let us relate our operators to the ones defined by~\citeN{truszczynski12a}.
First of all, it is worthwhile to realize that the motivation of~\citeN{truszczynski12a} is to conceive
operators mimicking model expansion in \FOID-logic by adding certain atoms.
More precisely, let \MirekFittingS, \MirekStableS, and \MirekWellFoundedS\ stand for the versions of
the Fitting, stable, and well-founded operators defined by~\citeN{truszczynski12a}.
Then, we get the following relations to the operators defined in the previous sections:
\begin{align*}
\MirekStable{P,\Ir{I}}{J}
& = \lfp{\MirekFitting{P,\Ir{I}}{\cdot,J}}\\
& = \lfp{T^{\Ir{I}}_{P_J}} \cup \Ir{I}\\
& = \StableR{P}{\Ir{I}}{J} \cup \Ir{I}.
\intertext{For the well-founded operator we obtain}
\MirekWellFounded{P,\Ir{I}}{I,J}
& = \WellR{P}{\Ir{I},\Ir{I}}{I,J} \sqcup \Ir{I}.
\end{align*}
Our operators allow us to directly calculate the atoms derived by a program.
The versions of~\citeN{truszczynski12a} always include the input facts in their output and
the well-founded operator only takes certain but not possible atoms as input.

In fact, we use operators as~\citeN{dematr00a} to approximate the well-founded model and to obtain a ground program.
While we apply operators to infinitary formulas (resulting from a translation of aggregates) as introduced by~\citeN{truszczynski12a},
there has also been work on applying operators directly to aggregates.
\citeN{vabrde21a} provide an overview.
Interestingly, the high complexity of approximating the aggregates pointed out in \cref{prp:propagate-generic-iff} has already been identified by~\citeN{pedebr07a}.

Simplification can be understood as a combination of unfolding
(dropping rules if a literal in the positive body is not among the head atoms of a program, i.e., not among the possible atoms)
and negative reduction
(dropping rules if an atom in the negative body is a fact, i.e., the literal is among the certain atoms)~\cite{bradix99a,brdifrzu01a}.
Even the process of grounding can be seen as a directed way of applying unfolding (when matching positive body literals)
and negative reduction (when matching negative body literals).
When computing facts,
only rules whose negative body can be removed using positive reduction are considered.

The algorithms of~\citeN{kestsr91a} to calculate well-founded models perform a computation
inspired by the alternating sequence to define the well-founded model as~\citeN{vangelder93a}.
Our work is different in so far as we are not primarily interested in computing the well-founded model but the grounding of a program.
Hence, our algorithms stop after the second application of the stable operator (the first to compute certain and the second to compute possible atoms).
At this point, a grounder can use algorithms specialized for propositional programs to simplify the logic program at hand.
Algorithmic refinements for normal logic programs as proposed by~\citeN{kestsr91a} also apply in our setting.

Last but not least, let us outline the evolution of grounding systems over the last two decades.

The \lparse~\cite{lparseManual} grounder introduced domain- or omega-restricted programs~\cite{syrjanen01a}.
Unlike safety, omega-restrictedness is not modular.
That is, the union of two omega-restricted programs is not necessarily omega-restricted while the union of two safe programs is safe.
Apart from this,
\lparse\ supports recursive monotone and antimonotone aggregates.
However, our company controls encoding in~\cref{ex:company} is not accepted
because it is not omega-restricted.
For example,
variable~\(X\) in the second aggregate element needs a domain predicate.
Even if we supplied such a domain predicate,
\lparse\ would instantiate variable \(X\) with all terms provided by the domain predicate
resulting in a large grounding.
As noted by~\citeN{ferlif02a}, recursive nonmonotone aggregates (sum aggregates with negative weights) are not supported correctly by \lparse.

\Gringo~1 and~2 add support for lambda-restricted programs~\cite{gescth07a} extending omega-restricted programs.
This augments the set of predicates that can be used for instantiation but is still restricted as compared to safe programs.
That is, lambda-restrictedness is also not modular and our company controls program is still not accepted.
At the time,
the development goal was to be compatible to \lparse\ but extend the class of accepted programs.
Notably, \gringo~2 adds support for additional aggregates~\cite{gekaosscth09a}.
Given its origin, \gringo\ up to version~\(4\) handles recursive nonmonotone aggregates in the same incorrect
way as \lparse.

The grounder of the \dlv~system has been the first one to implement grounding algorithms based on semi-naive evaluation~\cite{dlv97a}.
Furthermore, it implements various techniques to efficiently ground logic programs~\cite{lepesc01a,falepepf01a,pesccale07a}.
The \(\dlv^\mathcal{A}\) system is the first \dlv-based system to support recursive aggregates~\cite{defaielepf03a},
which is nowadays also available in recent versions of \idlv~\cite{cafupeza17a}.

\Gringo~3 closed up to \dlv\ being the first \gringo\ version to implement grounding algorithms based on semi-naive evaluation~\cite{gekakosc11a}.
The system accepts safe rules but still requires lambda-restrictedness for predicates within aggregates.
Hence, our company controls encoding is still not accepted.

\Gringo~4 implements grounding of aggregates with algorithms similar to the ones presented in \cref{sec:algorithms}~\cite{gekasc15a}.
Hence, it is the first version that accepts our company controls encoding.

Finally, \gringo~5 refines the translation of aggregates as proposed by~\citeN{alfage15a} to properly support nonmonotone recursive aggregates and refines the semantics of pools and undefined arithmetics~\cite{gehakalisc15a}.

Another system with a grounding component is the \idp\ system~\cite{debobrde14a}.
Its grounder instantiates a theory by assigning sorts to variables.
Even though it supports inductive definitions,
it relies solely on the sorts of variables~\cite{wimade10a} to instantiate a theory.
In case of inductive definitions, this can lead to instances of definitions that can never be applied.
We believe that the algorithms presented in \cref{sec:algorithms} can also be implemented in an \idp\ system
decreasing the instantiation size of some problems (e.g., the company controls problem presented in \cref{ex:company}).

Last but not least, we mention that not all ASP systems follow a two-phase approach of grounding and
solving but rather adapt a lazy approach by grounding on-the-fly during solving~\cite{padoporo09a,lebestga17a,wetafr20a}.

 \section{Conclusion}\label{sec:conclusion}

We have provided a first comprehensive elaboration of the theoretical foundations of grounding in {ASP}.
This was enabled by the establishment of semantic underpinnings of ASP's modeling language in terms of
infinitary (ground) formulas~\cite{haliya14a,gehakalisc15a}.
Accordingly,
we start by identifying a restricted class of infinitary programs, namely, \(\FormulasR\)-programs,
by limiting the usage of implications.
Such programs allow for tighter semantic characterizations than general \(\FormulasF\)-programs,
while being expressive enough to capture logic programs with aggregates.
Interestingly, we rely on well-founded models~\cite{brdetr16a,truszczynski18a} to approximate the stable models of \(\FormulasR\)-programs
(and simplify them in a stable-models preserving way).
This is due do the fact that the (\FOID-)well-founded-operator enjoys monotonicity, which lends itself to the characterization
of iterative grounding procedures.
The actual semantics of non-ground aggregate programs is then defined via a translation to \(\FormulasR\)-programs.
This setup allows us to characterize the inner workings of our grounding algorithms for aggregate programs in
terms of the operators introduced for \(\FormulasR\)-programs.
It turns out that grounding amounts to calculating an approximation of the well-founded model together with a ground program simplified with that model.
This does not only allow us to prove the correctness of our grounding algorithms but moreover to characterize
the output of a grounder like \gringo\ in terms of established formal means.
To this end,
we have shown how to split aggregate programs into components and to compute their approximate models
(and corresponding simplified ground programs).
The key instruments for obtaining finite ground programs with finitary subformulas have been
dedicated forms of program simplification and aggregate translation.
Even though, we limit ourselves to \(\FormulasR\)-programs,
we capture the core aspects of grounding: a monotonically increasing set of possibly derivable atoms and on-the-fly (ground) rule generation.
Additional language features of \gringo's input language are relatively straightforward to accommodate by extending the algorithms presented in this paper.

For reference,
we implemented the presented algorithms in a prototypical grounder, $\mu$-\gringo, supporting aggregate programs (see \cref{fn:mu-gringo}).
While it is written to be as concise as possible and not with efficiency in mind,
it may serve as a basis for experiments with custom grounder implementations.
The actual \gringo\ system supports a much larger language fragment.
There are some differences compared to the algorithms presented here.
First, certain atoms are removed from rule bodies if not explicitly disabled via a command line option.
Second, translation~\TransASym\ is only used to characterize aggregate propagation.
In practice, \gringo\ translates ground aggregates to monotone aggregates~\cite{alfage15a}.
Further translation~\cite{bogeja14a} or even native handling~\cite{gekakasc09a} of them is left to the solver.
Finally, in some cases, \gringo\ might produce more rules than the algorithms presented above.
This should not affect typical programs.
A tighter integration of grounder and solver to further reduce the number of ground rules is an interesting topic of future research.

\bibliographystyle{acmtrans}
\bibliography{krr,procs}

\clearpage
\appendix
\section{Proofs}

We use the following reduct-based characterization of strong equivalence~\cite{turner03a}.

\begin{lemma}\label{prp:strong:equivalence}
Sets \(\mathcal{H}_1\) and \(\mathcal{H}_2\) of infinitary formulas are strongly equivalent
iff
\(\Reduct{\mathcal{H}_1}{I}\) and \(\Reduct{\mathcal{H}_2}{I}\) are classically equivalent for all two-valued interpretations~\(I\).
\end{lemma}

\begin{proof}
Let \(I\) and \(J\) be two-valued interpretations, and \(\mathcal{H}\) be an infinitary formula.
Clearly, \(I \models \Reduct{\mathcal{H}}{J}\) iff \(I \cap J \models \Reduct{\mathcal{H}}{J}\).
Thus, we only need to consider interpretations such that \(I \subseteq J\).
By Lemma~1 due to~\citeN{halipeva17a}, we have that \(I \models \Reduct{\mathcal{H}}{J}\) iff \((I,J)\) is an HT-model of \(\mathcal{H}\).
The proposition holds because by Theorem~3 Item~(iii) due to~\citeN{halipeva17a},
we have that \(\mathcal{H}_1\) and \(\mathcal{H}_2\) are strongly equivalent iff they have the same HT models.
\end{proof}

\begin{proof}[Proof of \cref{prp:formula-program}]
Let \(F = {\{ \Body{r} \rightarrow \Head{r} \mid r \in P \}}^\wedge\).

First, we consider the case \(I\models F\).
By Lemma~1 due to~\citeN{truszczynski12a},
this implies that \(\Reduct{P}{I}\) and \(\Reduct{F}{I}\) are classically equivalent and
thus have the same minimal models.\footnote{To be precise, Lemma~1 by~\citeN{truszczynski12a} is stated for a set of formulas, which can be understood as an infinitary conjunction.}
Thus, \(I\) is a stable model of \(P\) iff \(I\) is a stable model of \(F\).

Second, we consider the case \(I\not\models F\) and show that \(I\) is neither a stable model of \(F\) nor \(P\).
Proposition~1 by~\citeN{truszczynski12a} states that \(I\) is a model of \(F\) iff \(I\) is a model of \(\Reduct{F}{I}\).
Thus, \(I\) is not a stable model of \(F\).
Furthermore, because \(I \not\models F\), there is a rule \(r \in P\) such that \(I \not\models \Body{r} \rightarrow \Head{r}\).
Consequently, we have \(I \models \Body{r}\) and \(I \not\models \Head{r}\).
Using the above proposition again, we get \(I \models \Reduct{\Body{r}}{I}\).
Because \(I \models \Reduct{\Body{r}}{I}\) and \(I \not\models \Head{r}\), we get \(I \not\models r^I\) and in turn \(I \not\models \Reduct{P}{I}\).
Thus, \(I\) is not a stable model of \(P\) either.
\end{proof}

\begin{lemma}\label{lem:monotone-positive}
Let \(F\) be a formula, and \(I\) and \(J\) be interpretations.

If \(F\) is positive and \(I \subseteq J\), then \(I \models F\) implies \(J \models F\).
\end{lemma}

\begin{proof}
This property can be shown by induction over the rank of the formula.
\end{proof}

The following two propositions shed some light on the two types of reducts.

\begin{lemma}\label{lem:ferraris-basic}
Let \(F\) be a formula, and \(I\) and \(J\) be interpretations.

Then,
\begin{alphaenum}
\item if \(F\) is positive then \(\Reduct{F}{I}\) is positive,\label{lem:ferraris-basic:a}
\item \(I \models F\) iff \(I \models \Reduct{F}{I}\),\label{lem:ferraris-basic:b}
\item if \(F\) is strictly positive and \(I \subseteq J\) then \(I \models F\) iff \(I \models \Reduct{F}{J}\).\label{lem:ferraris-basic:d}
\end{alphaenum}
\end{lemma}

\begin{proof}
\begin{proofprop}[\ref{lem:ferraris-basic:a}]
Because the reduct only replaces subformulas by \(\bot\), the resulting formula is still positive.
\end{proofprop}

\begin{proofprop}[\ref{lem:ferraris-basic:b}]
Corresponds to Proposition~1 by~\citeN{truszczynski12a}.
\end{proofprop}

\begin{proofprop}[\ref{lem:ferraris-basic:d}]
This property can be shown by induction over the rank of the formula.\qedhere
\end{proofprop}
\end{proof}

\begin{lemma}\label{lem:foid-basic}
Let \(F\) be a formula, and \(I\), \(J\), and \(X\) be interpretations.

Then,
\begin{alphaenum}
\item \(\IDReductP{F}{I}\) is positive,\label{lem:foid-basic:a}
\item \(I \models F\) iff \(I \models \IDReductP{F}{I}\),\label{lem:foid-basic:b}
\item if \(F\) is positive then \(F = \IDReductP{F}{I}\), and\label{lem:foid-basic:c}
\item if \(I \subseteq J\) then \(X \models \IDReductP{F}{J}\) implies \(X \models \IDReductP{F}{I}\).\label{lem:foid-basic:d}
\end{alphaenum}
\end{lemma}

\begin{proof}
\begin{proofprop}[\ref{lem:foid-basic:a}]
Because the \FOID-reduct replaces all negative occurrences of atoms, the resulting formula is positive.
\end{proofprop}

\begin{proofprop}[\ref{lem:foid-basic:b}]
This property holds
because when the reduct replaces an atom \(a\), it is replaced by either \(\top\) or \(\bot\) depending on whether \(I \models a\) or \(I \not\models a\).
This does not change the satisfaction of the subformula w.r.t.\ \(I\).
\end{proofprop}

\begin{proofprop}[\ref{lem:foid-basic:c}]
Because a positive formula does not contain negative occurrences of atoms, it is not changed by the \FOID-reduct.
\end{proofprop}

\begin{proofprop}[\ref{lem:foid-basic:d}]
We prove by induction over the rank of formula \(F\) that
\begin{align}
X \models \IDReductP{F}{J} & \ImpliesT X \models \IDReductP{F}{I} \text{\ and}\label{foid-basic:e:p}\\
X \models \IDReductN{F}{I} & \ImpliesT X \models \IDReductN{F}{J}.\label{foid-basic:e:n}
\end{align}

\begin{proofpart}[Base]
We consider the case that \(F\) is a formula of rank \(0\).
\begin{align*}
\text{\(F\) is a formula of rank \(0\)} & \ImpliesT \text{\(F\) is an atom.}
\intertext{\indent First, we show~\eqref{foid-basic:e:p}.
We assume \(X \models \IDReductP{F}{J}\):}
\text{\(F\) is an atom} & \ImpliesT \IDReductP{F}{I} = \IDReductP{F}{J} = F. \\
\IDReductP{F}{I} = \IDReductP{F}{J} \AndT X \models \IDReductP{F}{J} & \ImpliesT X \models \IDReductP{F}{I}.
\intertext{\indent Second, we show~\eqref{foid-basic:e:n}.
We assume \(X \models \IDReductN{F}{I}\):}
\text{\(F\) is an atom} \AndT X \models \IDReductN{F}{I} & \ImpliesT \IDReductN{F}{I} = \top. \\
\IDReductN{F}{I} = \top & \ImpliesT F \in I. \\
F \in I \AndT I \subseteq J & \ImpliesT F \in J. \\
\text{\(F\) is an atom} \AndT F \in J & \ImpliesT \IDReductN{F}{J} = \top.\\
\IDReductN{F}{J} = \top & \ImpliesT X \models \IDReductN{F}{J}.
\end{align*}
\end{proofpart}

\begin{proofpart}[Hypothesis]
We assume that~\eqref{foid-basic:e:p} and~\eqref{foid-basic:e:n}
hold for formulas \(F\) of ranks smaller than \(i\).
\end{proofpart}

\begin{proofpart}[Step]
We only show~\eqref{foid-basic:e:p} because~\eqref{foid-basic:e:n} can be shown in a similar way.
We consider formulas \(F\) of rank \(i\).

First, we consider the case that \(F\) is a conjunction of form \(\FormulasH^\wedge\).
\begin{align*}
X \models \IDReductP{\FormulasH^\wedge}{J} & \ImpliesT X \models \IDReductP{G}{J} \ForAllT G \in \FormulasH.\\
X \models \IDReductP{G}{J} & \ImpliesT X \models \IDReductP{G}{I} \ByT{hypothesis}.\\
X \models \IDReductP{G}{I} & \ImpliesT X \models \IDReductP{\FormulasH^\wedge}{I}.
\intertext{The case for disjunctions can be proven in a similar way.\endgraf Last, we consider the case that \(F\) is an implication of form \(G \rightarrow H\).
Observe that}
\IDReductP{F}{I} & = \IDReductN{G}{I} \rightarrow \IDReductP{H}{I} \text{\ and}\\
\IDReductP{F}{J} & = \IDReductN{G}{J} \rightarrow \IDReductP{H}{J}.
\intertext{First, we consider the case \(X \not\models \IDReductN{G}{J}\):}
X \not\models \IDReductN{G}{J} & \ImpliesT X \not\models \IDReductN{G}{I} \ByT{hypothesis}.\\
X \not\models \IDReductN{G}{I} & \ImpliesT X \models \IDReductP{F}{I}.\\
\intertext{Second, we consider the case \(X \models \IDReductP{H}{J}\):}
X \models \IDReductP{H}{J} & \ImpliesT X \models \IDReductP{H}{I} \ByT{hypothesis}.\\
X \models \IDReductP{H}{I} & \ImpliesT X \models \IDReductP{F}{I}.\qedhere
\end{align*}
\end{proofpart}
\end{proofprop}
\end{proof}

\begin{proof}[Proof of \cref{lem:well-founded:consistent}]
This lemma follows from Proposition~14 by~\citeN{dematr00a} observing that the well-founded operator is a monotone symmetric operator.
The proposition is actually a bit more general stating that the operator maps any consistent four-valued interpretation to a consistent four-valued interpretation.
\end{proof}

\begin{lemma}\label{lem:operator:less}
Let \(\OpSym\) and \(\OppSym\) be monotone operators over complete lattice~\((L,\leq)\) with \(\OppApp{x} \leq \OpApp{x}\) for each \(x \in L\).

Then, we get \(x' \leq x\) where \(x'\) and \(x\) are the least fixed points of \(\OppSym\) and \(\OpSym\), respectively.
\end{lemma}

\begin{proof}
Let \(y\) be a prefixed point of \(\OpSym\).
We have \(\OpApp{y} \leq y\).
Because \(\OppApp{y} \leq \OpApp{y}\), we get \(\OppApp{y} \leq y\).
So each prefixed point of \(\OpSym\) is also a prefixed point of \(\OppSym\).

Let \(S'\) and \(S\) be the set of all prefixed points of \(\OppSym\) and \(\OpSym\), respectively.
We obtain \(S \subseteq S'\).
By \Nref{thm:knaster-tarski}{a}, we get that \(x'\) is the greatest lower bound of \(S'\).
Observe that \(x'\) is a lower bound for \(S\).
By construction of \(S\), we have \(x \in S\).
Hence, we get \(x' \leq x\).
\end{proof}

\begin{lemma}\label{lem:program:positive:subset}
Let \(P\) and \(P'\) be \(\FormulasF\)-programs and \(I\) be an interpretation.

Then, \(P' \subseteq P\) implies \(\Stable{P'}{I} \subseteq \Stable{P}{I}\).
\end{lemma}

\begin{proof}
This lemma is a direct consequence of \cref{lem:operator:less}
observing that the one-step provability operator derives fewer consequences for~\(P'\).
\end{proof}

\begin{lemma}\label{lem:stable:possible}
Let \(P\) be an \(\FormulasR\)-program,
\(I\) be a two-valued interpretation, and
\(J = \Stable{P}{I}\).

Then, \(X\) is a stable model of \(P\), \(I \subseteq X\), and \(I \subseteq J\) implies \(X \subseteq J\).
\end{lemma}

\begin{proof}

Because \(X\) is a stable model of \(P\), it is the only minimal model of \(\Reduct{P}{X}\).
Furthermore, we have that \(J\) is a model of \(\IDReductP{P}{I}\).
To show that \(X \subseteq J\), we show that \(J\) is also a model of \(\Reduct{P}{X}\).
For this, it is enough to show that for each rule \(r \in P\) we have \(J \not\models \IDReductP{\Body{r}}{I}\) implies \(J \not\models \Reduct{\Body{r}}{X}\).
We prove inductively over the rank of the formula \(F = \Body{r}\) that \(J \not\models \IDReductP{F}{I}\) implies \(J \not\models \Reduct{F}{X}\).

\begin{proofpart}[Base]
We consider the case that \(F\) is a formula of rank \(0\).

If \(X \not\models F\), we get \(J \not\models \Reduct{F}{X}\) because \(\Reduct{F}{X} = \bot\).
Thus, we only have to consider the case \(X \models F\):
\begin{align*}
\text{\(F\) is a formula of rank \(0\)} & \ImpliesT \text{\(F\) is an atom.}\\
\text{\(F\) is an atom} & \ImpliesT \text{\(\IDReductP{F}{I} = F\).} \\
X \models F \AndT \text{\(F\) is an atom} & \ImpliesT \Reduct{F}{X} = F. \\
\IDReductP{F}{I} = F \AndT \Reduct{F}{X} = F & \ImpliesT \IDReductP{F}{I} = \Reduct{F}{X}. \\
J \not\models \IDReductP{F}{I} \AndT \IDReductP{F}{I} = \Reduct{F}{X} & \ImpliesT J \not\models \Reduct{F}{X}.
\end{align*}
\end{proofpart}

\begin{proofpart}[Hypothesis]
We assume that \(J \not\models \IDReductP{F}{I}\) implies \(J \not\models \Reduct{F}{X}\) holds for formulas \(F\) of ranks smaller than \(i\).
\end{proofpart}

\begin{proofpart}[Step]
We consider the case that \(F\) is a formula of rank \(i\).

As in the base case, we only have to consider the case \(X \models F\).
Furthermore, we have to distinguish the cases that \(F\) is a conjunction, disjunction, or implication.

We first consider the case that \(F\) is a conjunction of form \(\mathcal{F}^\wedge\):
\begin{align*}
X \models F & \ImpliesT \Reduct{F}{X} = \{{\Reduct{G}{X} \mid G \in \mathcal{F}\}}^\wedge. \\
J \not\models \IDReductP{F}{I} \AndT \IDReductP{F}{I} = {\{ \IDReductP{G}{I} \mid G \in \mathcal{F} \}}^\wedge & \ImpliesT J \not\models \IDReductP{G}{I} \ForSomeT G \in \mathcal{F}. \\
G \in \mathcal{F} \AndT \text{\(F\) has rank \(i\)} & \ImpliesT \text{\(G\) has rank less than \(i\).} \\
J \not\models \IDReductP{G}{I} \AndT \text{\(G\) has rank less than \(i\)} & \ImpliesT J \not\models \Reduct{G}{X} \ByT{hypothesis}. \\
J \not\models \Reduct{G}{X} \AndT \Reduct{F}{X} = \{{\Reduct{G}{X} \mid G \in F\}}^\wedge & \ImpliesT J \not\models \Reduct{F}{X}.
\end{align*}

The case that \(F\) is a disjunction can be shown in a similar way to the case that \(F\) is a conjunction.

Last, we consider the case that \(F\) is an implication of form \(G \rightarrow H\).
Observe that \(G\) is positive because \(F\) has no occurrences of implications in its antecedent
and, furthermore, given that \(F\) is a formula of rank \(i\), \(H\) is a formula of rank less than \(i\).

We show \(I \models G\):
\begin{align*}
J \not\models \IDReductP{F}{I} \AndT \IDReductP{F}{I} = \IDReductN{G}{I} \rightarrow \IDReductP{H}{I} & \ImpliesT J \models \IDReductN{G}{I}\\
J \models \IDReductN{G}{I} \AndT \text{\(G\) is positive} & \ImpliesT I \models G \text{\ because\ } \IDReductN{G}{I} \equiv \top.
\intertext{\indent We show \(J \models \Reduct{G}{X}\):}
\text{\(G\) is positive} \CT I \subseteq X \CandT I \models G & \ImpliesT X \models G \ByT{\cref{lem:monotone-positive}.} \\
X \models F \CT X \models G \CandT F = G \rightarrow H & \ImpliesT X \models H \text{.}\\
\text{\(G\) is positive} \CT I \subseteq X \CandT I \models G & \ImpliesT I \models \Reduct{G}{X} \ByT{\Nref{lem:ferraris-basic}{d}.} \\
\text{\(G\) is positive} & \ImpliesT \text{\(\Reduct{G}{X}\) is positive} \ByT{\Nref{lem:ferraris-basic}{a}.} \\
\Reduct{G}{X} \text{\ is positive} \CT I \subseteq J \CandT I \models \Reduct{G}{X} & \ImpliesT J \models \Reduct{G}{X} \ByT{\cref{lem:monotone-positive}.} \\
\intertext{\indent We show \(J \not\models \Reduct{H}{X}\):}
I \models G \AndT \IDReductP{F}{I} = \IDReductN{G}{I} \rightarrow \IDReductP{H}{I} & \ImpliesT \IDReductP{F}{I} \equiv \IDReductP{H}{I} \text{\ because\ } \IDReductN{G}{I} \equiv \top. \\
\IDReductP{F}{I} \equiv \IDReductP{H}{I} \AndT J \not\models \IDReductP{F}{I} & \ImpliesT J \not\models \IDReductP{H}{I}. \\
J \not\models \IDReductP{H}{I} \text{\ and\ \(H\) has rank less than \(i\)} & \ImpliesT J \not\models \Reduct{H}{X} \text{\ by hypothesis.}
\end{align*}

Because \(X \models F\), we have \(\Reduct{F}{X} = \Reduct{G}{X} \rightarrow \Reduct{H}{X}\).
Using \(J \models \Reduct{G}{X}\) and \(J \not\models \Reduct{H}{X}\), we get \(J \not\models \Reduct{F}{X}\).\qedhere
\end{proofpart}
\end{proof}

\begin{proof}[Proof of \cref{thm:stable:well-founded}]
Let \(X\) be a stable model of \(P\).

We prove by transfinite induction over the sequence of postfixed points leading to the well-founded model:
\begin{align*}
(\Ii{I}{0}, \Ii{J}{0}) & = (\emptyset, \Signature),\\
(\Ii{I}{\alpha+1}, \Ii{J}{\alpha+1}) & = \Well{P}{\Ii{I}{\alpha},\Ii{J}{\alpha}} \text{\ for ordinals \(\alpha\)}, \AndT\\
(\Ii{I}{\beta},\Ii{J}{\beta}) & = (\bigcup_{\alpha < \beta} \Ii{I}{\alpha}, \bigcap_{\alpha < \beta} \Ii{J}{\alpha}) \text{\ for limit ordinals \(\beta\).}
\end{align*}
We have that \(\alpha < \beta\) implies \((\Ii{I}{\alpha},\Ii{J}{\alpha}) \leq_p (\Ii{I}{\beta},\Ii{J}{\beta})\) for ordinals \(\alpha\) and \(\beta\),
\(\Ii{I}{\alpha} \subseteq \Ii{J}{\alpha}\) for ordinals \(\alpha\), and
there is a least ordinal \(\alpha\) such that \((I,J) = (\Ii{I}{\alpha},\Ii{J}{\alpha})\).

\begin{proofpart}[Base]
We have \(\Ii{I}{0} \subseteq X \subseteq \Ii{J}{0}\).
\end{proofpart}

\begin{proofpart}[Hypothesis]
We assume \(\Ii{I}{\beta} \subseteq X \subseteq \Ii{J}{\beta}\) for all ordinals \(\beta < \alpha\).
\end{proofpart}

\begin{proofpart}[Step]
If \(\alpha = \beta + 1\) is a successor ordinal we have
\begin{align*}
(\Ii{I}{\alpha}, \Ii{J}{\alpha})
& = \Well{P}{\Ii{I}{\beta},\Ii{J}{\beta}}\\
& = (\Stable{P}{\Ii{J}{\beta}}, \Stable{P}{\Ii{I}{\beta}}).
\end{align*}
By the induction hypothesis we have \(\Ii{I}{\beta} \subseteq X \subseteq \Ii{J}{\beta}\).

First, we show \(\Ii{I}{\alpha} \subseteq X\):
\begin{align*}
\text{\(X\) is a (stable) model} & \ImpliesT \Stable{P}{X} \subseteq X.\\
X \subseteq \Ii{J}{\beta} & \ImpliesT \Stable{P}{\Ii{J}{\beta}} \subseteq \Stable{P}{X}.\\
\Stable{P}{X} \subseteq X \AndT \Stable{P}{\Ii{J}{\beta}} \subseteq \Stable{P}{X} & \ImpliesT \Stable{P}{\Ii{J}{\beta}} \subseteq X.\\
\Ii{I}{\alpha} = \Stable{P}{\Ii{J}{\beta}} \AndT \Stable{P}{\Ii{J}{\beta}} \subseteq X & \ImpliesT \Ii{I}{\alpha} \subseteq X.
\end{align*}

Second, we show \(X \subseteq \Ii{J}{\alpha}\):
\begin{align*}
\beta < \alpha & \ImpliesT (\Ii{I}{\beta},\Ii{J}{\beta}) \leq_p (\Ii{I}{\alpha},\Ii{J}{\alpha})\\
(\Ii{I}{\beta},\Ii{J}{\beta}) \leq_p (\Ii{I}{\alpha},\Ii{J}{\alpha}) \AndT \Ii{I}{\alpha} \subseteq \Ii{J}{\alpha} & \ImpliesT \Ii{I}{\beta} \subseteq \Ii{J}{\alpha}\\
\text{\shortstack[l]{\(X\) is a stable model, \(\Ii{I}{\beta} \subseteq X\), \\\(\Ii{J}{\alpha} = \Stable{P}{\Ii{I}{\beta}}\), and \(\Ii{I}{\beta} \subseteq \Ii{J}{\alpha}\)}} & \ImpliesT \text{\(X \subseteq \Ii{J}{\alpha}\) by \cref{lem:stable:possible}}.
\end{align*}
We have shown \(\Ii{I}{\alpha} \subseteq X \subseteq \Ii{J}{\alpha}\) for successor ordinals.

If \(\alpha\) is a limit ordinal we have
\begin{align*}
(\Ii{I}{\alpha}, \Ii{J}{\alpha}) & = (\bigcup_{\beta < \alpha} \Ii{I}{\beta},\bigcap_{\beta < \alpha} \Ii{J}{\beta}).
\end{align*}

Let \(x \in \Ii{I}{\alpha}\).
There must be an ordinal \(\beta < \alpha\) such that \(x \in \Ii{I}{\beta}\).
Since \(\Ii{I}{\beta} \subseteq X\) by the hypothesis, we have \(x \in X\).
Thus, \(\Ii{I}{\alpha} \subseteq X\).

Let \(x \in X\).
For each ordinal \(\beta < \alpha\) we have \(x \in \Ii{J}{\beta}\) because \(X \subseteq \Ii{J}{\beta}\) by the hypothesis.
Thus, we get \(x \in \Ii{J}{\alpha}\).
It follows that \(X \subseteq \Ii{J}{\alpha}\).

We have shown \(\Ii{I}{\alpha} \subseteq X \subseteq \Ii{J}{\alpha}\) for limit ordinals.\qedhere
\end{proofpart}
\end{proof}

\begin{lemma}\label{lem:ground:rule:derive}
Let \(P\) be an \(\FormulasF\)-program and \((I,J)\) be a four-valued interpretation.

Then, we have \(\Head{\Simp{P}{I,J}} = \Step{\IDReductP{P}{I}}{J}\).
\end{lemma}

\begin{proof}
The program \(\Simp{P}{I,J}\) contains all rules \(r \in P\) such that \(J \models \IDReductP{\Body{r}}{I}\).
This are exactly the rules whose heads are gathered by the \(T\) operator.
\end{proof}

\begin{lemma}\label{prp:simplification:well-founded}
Let \(P\) be an \(\FormulasF\)-program and \((I,J)\) be the well-founded model of \(P\).

Then, we have
\begin{alphaenum}
\item \(\Stable{\Simp{P}{I,J}}{I'} = J\) for all \(I' \subseteq I\), and\label{prp:simplification:well-founded:a}
\item \(\Stable{\Simp{P}{I,J}}{J'} = \Stable{P}{J'}\) for all \(J \subseteq J'\).\label{prp:simplification:well-founded:b}
\end{alphaenum}
\end{lemma}

\begin{proof}

Throughout the proof we use
\begin{align*}
\Stable{P}{J} & = I, \\
\Stable{P}{I} & = J,\\
\Simp{P}{I,J} & \subseteq P\text{, and}\\
I & \subseteq J \text{\ because the well-founded model is consistent.}
\end{align*}

\begin{proofprop}[\ref{prp:simplification:well-founded:a}]
We show \(J = \Stable{\Simp{P}{I,J}}{I}\).
Let \(\hat{J} = \Stable{\Simp{P}{I,J}}{I}\) and \(r \in P \setminus \Simp{P}{I,J}\):
\begin{align*}
\Simp{P}{I,J} \subseteq P & \AndT \\
\hat{J} = \Stable{\Simp{P}{I,J}}{I} & \AndT \\
J = \Stable{P}{I} & \ImpliesT \hat{J} \subseteq J\text{\ by \cref{lem:program:positive:subset}.}\\
r \notin \Simp{P}{I,J} & \ImpliesT J \not\models \IDReductP{\Body{r}}{I}.\\
\hat{J} \subseteq J \AndT J \not\models \IDReductP{\Body{r}}{I} & \ImpliesT \hat{J} \not\models \IDReductP{\Body{r}}{I}\text{\ by \cref{lem:monotone-positive}.}\\
\hat{J} = \Stable{\Simp{P}{I,J}}{I} \AndT \hat{J} \not\models \IDReductP{\Body{r}}{I} & \ImpliesT \hat{J} \models \IDReductP{P}{I}.\\
\hat{J} \models \IDReductP{P}{I} \AndT J = \Stable{P}{I} & \ImpliesT J \subseteq \hat{J}.\\
J \subseteq \hat{J} \AndT \hat{J} \subseteq J & \ImpliesT J = \hat{J}.
\end{align*}
Thus, we get that \(\Stable{\Simp{P}{I,J}}{I} = J\).

With this we can continue to prove \(\Stable{\Simp{P}{I,J}}{I'} = J\).
Let \(r \in \Simp{P}{I,J}\):
\begin{align*}
r \in \Simp{P}{I,J} & \ImpliesT J \models \IDReductP{\Body{r}}{I}.\\
r \in \Simp{P}{I,J} \AndT \Simp{P}{I,J} \subseteq P & \ImpliesT r \in P.\\
J \models \IDReductP{\Body{r}}{I} \CT r \in P, \AndT \Stable{P}{I} = J & \ImpliesT \Head{r} \in J.\\
J \models \IDReductP{\Body{r}}{I} \AndT I' \subseteq I & \ImpliesT J \models \IDReductP{\Body{r}}{I'}\text{\ by  \Nref{lem:foid-basic}{d}.}\\
\Head{r} \in J \AndT J \models \IDReductP{\Body{r}}{I'} & \ImpliesT \Stable{\Simp{P}{I,J}}{I'} \subseteq J.\\
I' \subseteq I \AndT J = \Stable{\Simp{P}{I,J}}{I} & \ImpliesT J \subseteq \Stable{\Simp{P}{I,J}}{I'}.
\end{align*}
Thus, we get \(\Stable{\Simp{P}{I,J}}{I'} = J\).
\end{proofprop}

\begin{proofprop}[\ref{prp:simplification:well-founded:b}]
Let \(I' = \Stable{\Simp{P}{I,J}}{J'}\) and \(r \in P \setminus \Simp{P}{I,J}\):
\begin{align*}
r \notin \Simp{P}{I,J} & \ImpliesT J \not\models \IDReductP{\Body{r}}{I}.\\
I \subseteq J \CT J \subseteq J' \CandT J \not\models \IDReductP{\Body{r}}{I} & \ImpliesT J \not\models \IDReductP{\Body{r}}{J'}.\text{\ by  \Nref{lem:foid-basic}{d}.}.\\
I' \subseteq I \CT I \subseteq J \CandT J \not\models \IDReductP{\Body{r}}{J'} & \ImpliesT I' \not\models \IDReductP{\Body{r}}{J'}.\text{\ by  \cref{lem:monotone-positive}.}\\
I' = \Stable{\Simp{P}{I,J}}{J'} \AndT I' \not\models \IDReductP{\Body{r}}{J'} & \ImpliesT \Stable{P}{J'} \subseteq \Stable{\Simp{P}{I,J}}{J'}.\\
\Simp{P}{I,J} \subseteq P & \ImpliesT \Stable{\Simp{P}{I,J}}{J'} \subseteq \Stable{P}{J'} \text{\ by \cref{lem:program:positive:subset}.}
\end{align*}
Thus, we get \(\Stable{\Simp{P}{I,J}}{J'} = \Stable{P}{J'}\).\qedhere
\end{proofprop}
\end{proof}

\begin{proof}[Proof of \cref{thm:well-founded:simplification}]
By \cref{prp:simplification:well-founded}, we have \((I,J) = \Well{\Simp{P}{I,J}}{I,J}\).
Furthermore, we let \((\widehat{I},\widehat{J}) = \WellM{\Simp{P}{I,J}}\):
\begin{align*}
(\widehat{I},\widehat{J}) = \WellM{\Simp{P}{I,J}} \AndT (I,J) = \Well{\Simp{P}{I,J}}{I,J} & \ImpliesT (\widehat{I},\widehat{J}) \leq_p (I,J).\\
& \text{\ by \Nref{thm:knaster-tarski}{c}.}\\
\widehat{I} \subseteq I & \ImpliesT \Stable{\Simp{P}{I,J}}{\widehat{I}} = \Stable{\Simp{P}{I,J}}{I}\\
& \text{\ by \Nref{prp:simplification:well-founded}{a}.} \\
\hat{J} = \Stable{\Simp{P}{I,J}}{\widehat{I}} = \Stable{\Simp{P}{I,J}}{I} = J& \ImpliesT \widehat{J} = J. \\
\widehat{I} = \Stable{\Simp{P}{I,J}}{\widehat{J}} \CT \Stable{\Simp{P}{I,J}}{J} = I \CandT \widehat{J} = J & \ImpliesT \widehat{I} = I.
\end{align*}
We obtain \((I,J) = (\widehat{I}, \widehat{J})\).
\end{proof}

\begin{proof}[Proof of \cref{thm:stable:simplification}]
We first show that all rule bodies removed by the simplification are falsified by \(X\).
Let \(r \in P \setminus \Simp{P}{I,J}\) and assume \(X \models {\Body{r}}\):
\begin{align*}
X \models \Body{r} & \ImpliesT X \models \IDReductP{\Body{r}}{X} \ByT{\Nref{lem:foid-basic}{b}}.\\
X \models \IDReductP{\Body{r}}{X} \AndT I \subseteq X & \ImpliesT X \models \IDReductP{\Body{r}}{I} \ByT{\Nref{lem:foid-basic}{d}}.\\
X \models \IDReductP{\Body{r}}{I} \AndT X \subseteq J & \ImpliesT J \models \IDReductP{\Body{r}}{I} \ByT{\cref{lem:monotone-positive}}.
\intertext{This is a contradiction and, thus, \(X \not\models \Body{r}\).
We use the following consequence in the proof below:}
X \not\models \Body{r} & \ImpliesT \ReductC{P \setminus \Simp{P}{I,J}}{X} \equiv \emptyset.
\end{align*}

To show the theorem, we show that \(\Reduct{P}{X}\) and \(\ReductC{\Simp{P}{I,J}}{X}\) have the same minimal models.
Clearly, we have \(\Reduct{P}{X} = \ReductC{\Simp{P}{I,J}}{X} \cup \ReductC{P \setminus \Simp{P}{I,J}}{X}\).
Using this and \(\ReductC{P \setminus \Simp{P}{I,J}}{X} \equiv \emptyset\),
we obtain that \(\Reduct{P}{X}\) and \(\ReductC{\Simp{P}{I,J}}{X}\) have the same minimal models.
\end{proof}

\begin{proof}[Proof of \cref{cor:stable:simplification}]
The result follows from \cref{thm:stable:simplification,thm:well-founded:simplification,thm:stable:well-founded}.
\end{proof}

\begin{proof}[Proof of \cref{thm:well-founded:simplification:superset}]
By \cref{lem:program:positive:subset},
we have \(\Stable{\Simp{P}{I,J}}{X} \subseteq \Stable{Q}{X} \subseteq \Stable{P}{X}\) for any two-valued interpretation~\(X\).
Thus, by \cref{thm:well-founded:simplification}, we get \((I,J) = \Well{Q}{I,J}\).

Let \((\widehat{I}, \widehat{J})\) be a prefixed point of \(\WellO{Q}\) with \((\widehat{I}, \widehat{J}) \leq_p (I,J)\).
We have \((\Stable{Q}{\widehat{J}}, \Stable{Q}{\widehat{I}}) \leq_p (\widehat{I}, \widehat{J}) \leq_p (I,J)\).
\begin{align*}
J \subseteq \widehat{J}
& \ImpliesT \Stable{\Simp{P}{I,J}}{\widehat{J}} = \Stable{Q}{\widehat{J}} = \Stable{P}{\widehat{J}}\\
& \ByT{\Nref{prp:simplification:well-founded}{b}}.\\
\Stable{Q}{\widehat{J}} = \Stable{P}{\widehat{J}} \AndT \Stable{Q}{\widehat{J}} \subseteq \widehat{I}
& \ImpliesT \Stable{P}{\widehat{J}} \subseteq \widehat{I}.\\
\widehat{J} \subseteq \Stable{Q}{\widehat{I}} \AndT \Stable{Q}{\widehat{I}} \subseteq \Stable{P}{\widehat{I}}
& \ImpliesT \widehat{J} \subseteq \Stable{P}{\widehat{I}}.\\
\Stable{P}{\widehat{J}} \subseteq \widehat{I} \AndT \widehat{J} \subseteq \Stable{P}{\widehat{I}}
& \ImpliesT \Well{P}{\widehat{I}, \widehat{J}} \leq_p (\widehat{I}, \widehat{J}).\\
\Well{P}{\widehat{I}, \widehat{J}} \leq_p (\widehat{I}, \widehat{J})
& \ImpliesT (I,J) \leq_p (\widehat{I}, \widehat{J})\\
& \ByT{\Nref{thm:knaster-tarski}{a}}.\\
(I,J) \leq_p (\widehat{I}, \widehat{J}) \AndT (\widehat{I}, \widehat{J}) \leq_p (I,J)
& \ImpliesT (I,J) = (\widehat{I}, \widehat{J}).
\end{align*}
By \Nref{thm:knaster-tarski}{a}, we obtain that \(\WellM{Q}=(I,J)\).
\end{proof}

\begin{proof}[Proof of \cref{cor:stable:simplification:superset}]
Observe that \(\Simp{P}{I,J} = \Simp{Q}{I,J}\).
With this, the corollary follows from \cref{cor:stable:simplification,thm:well-founded:simplification:superset}.
\end{proof}

Alternatively,
the incorporation of context atoms can also be seen as a form of partial evaluation applied to the underlying program.

\begin{definition}
Let \(\Ir{I}\) be a two-valued interpretation.

We define the \emph{partial evaluation} of an \(\FormulasF\)-formula w.r.t.\ \(\Ir{I}\) as follows:
\begin{align*}
\IDEvalP{a}{\Ir{I}}                   & = \top \text{\ if } a \in \Ir{I}                           &
\IDEvalN{a}{\Ir{I}}                   & = a                                                        \\
\IDEvalP{a}{\Ir{I}}                   & = a \text{\ if } a \notin \Ir{I}                           &
                                      &                                                            \\
\IDEvalP{\FormulasH^{\wedge}}{\Ir{I}} & = {\{ \IDEvalP{F}{\Ir{I}} \mid F \in \FormulasH \}}^\wedge &
\IDEvalN{\FormulasH^{\wedge}}{\Ir{I}} & = {\{ \IDEvalN{F}{\Ir{I}} \mid F \in \FormulasH \}}^\wedge \\
\IDEvalP{\FormulasH^{\vee}}{\Ir{I}}   & = {\{ \IDEvalP{F}{\Ir{I}} \mid F \in \FormulasH \}}^\vee   &
\IDEvalN{\FormulasH^{\vee}}{\Ir{I}}   & = {\{ \IDEvalN{F}{\Ir{I}} \mid F \in \FormulasH \}}^\vee   \\
\IDEvalPC{F \rightarrow G}{\Ir{I}}    & = \IDEvalN{F}{\Ir{I}} \rightarrow \IDEvalP{G}{\Ir{I}}      &
\IDEvalNC{F \rightarrow G}{\Ir{I}}    & = \IDEvalP{F}{\Ir{I}} \rightarrow \IDEvalN{G}{\Ir{I}}
\end{align*}
where \(a\) is an atom, \(\FormulasH\) a set of formulas, and \(F\) and \(G\) are formulas.
\end{definition}

The partial evaluation of an \(\FormulasF\)-program~\(P\) w.r.t.\ a two-valued interpretation~\(\Ir{I}\) is
\(\IDEvalP{P}{\Ir{I}} = \{ \IDEvalP{r}{\Ir{I}} \mid r \in P \}\) where
\(\IDEvalP{r}{\Ir{I}} = \Head{h} \leftarrow \IDEvalP{\Body{r}}{\Ir{I}}\).
Accordingly,
the partial evaluation of rules boils down to replacing satisfied
positive occurrences of atoms in rule bodies by \(\top\).

We observe the following relationship between the relative one-step operators and partial evaluations.

\begin{observation}\label{obs:step:relative:evaluate}
Let \(P\) be a positive \(\FormulasF\)-program and \(\Ir{I}\) be a two-valued interpretation.

Then, we have for any two-valued interpretation \(I\) that
\begin{align*}
\StepR{P}{\Ir{I}}{I} &= \Step{\IDEvalP{P}{\Ir{I}}}{I}.
\end{align*}
\end{observation}

Note that \({\IDReductP{\IDEvalP{P}{\Ir{I}}}{J}} = {\IDEvalP{\IDReductP{P}{J}}{\Ir{I}}}\).

\begin{proof}[Proof of \cref{prop:rel-step-monotone}]
Clearly, \(\IDEvalP{P}{\Ir{I}}\) is positive whenever \(P\) is positive.
Using \cref{obs:step:relative:evaluate},
we obtain that \(\StepRO{P}{\Ir{I}}\) is monotone.

The second property directly follows from the monotonicity of the one-step provability operator.
\end{proof}

\begin{lemma}\label{lem:step-op-eval}
Let \(P\) be an \(\FormulasF\)-program and \(\Ir{I}\) be a two-valued interpretation.

For any two-valued interpretation~\(J\),
we get
\begin{align*}
\StableR{P}{\Ir{I}}{J} &= \LeastM{\IDReductP{\IDEvalP{P}{\Ir{I}}}{J}}.
\end{align*}
\end{lemma}

\begin{proof}
This lemma immediately follows from \cref{obs:step:relative:evaluate}.
\end{proof}

\begin{proof}[Proof of \cref{prp:stable:relative:monotonicity}]
Both properties can be shown by inspecting the reduced programs.

\begin{proofprop}[\(J' \subseteq J\) implies \(\StableR{P}{\Ir{I}}{J} \subseteq \StableR{P}{\Ir{I}}{J'}\)]
Observe that we can use \cref{lem:step-op-eval} to equivalently write \(\StableR{P}{\Ir{I}}{J} = \Stable{\IDEvalP{P}{\Ir{I}}}{J}\) and \(\StableR{P}{\Ir{I}}{J'} = \Stable{\IDEvalP{P}{\Ir{I}}}{J'}\).
With this and \cref{lem:stable:operator:antimonotone}, we see that the relative stable operator is antimonotone just as the stable operator.
\end{proofprop}

\begin{proofprop}[\(\Ipr{I} \subseteq \Ir{I}\) implies \(\StableR{P}{\Ipr{I}}{J} \subseteq \StableR{P}{\Ir{I}}{J}\)]
Observe that \(\StableR{P}{\Ir{I}}{J}\) is equal to the least fixed point of \(\StepRO{\IDReductP{P}{J}}{\Ir{I}}\)
and \(\StableR{P}{\Ipr{I}}{J}\) is equal to the least fixed point of \(\StepRO{\IDReductP{P}{J}}{\Ipr{I}}\).
Furthermore, observe that \(\StepR{\IDReductP{P}{J}}{\Ipr{I}}{X} \subseteq \StepR{\IDReductP{P}{J}}{\Ir{I}}{X}\) for any two-valued interpretation \(X\) because \(\Ipr{I} \subseteq \Ir{I}\) and the underlying \(\StepSym\) operator is monotone.
With this and \cref{lem:operator:less}, we have shown the property.\qedhere
\end{proofprop}
\end{proof}

\begin{observation}\label{obs:stable:relative:properties}
Let \(P\) be an \(\FormulasF\)-program, and \(\Ir{I}\) and \(J\) be two-valued interpretations.

We get the following properties:
\begin{alphaenum}
\item \(\StableR{P}{\emptyset}{J} = \Stable{P}{J}\),\label{obs:stable:relative:properties:a}
\item \(\StableR{P}{\Ir{I}}{J} \subseteq \Head{P}\), and\label{obs:stable:relative:properties:b}
\item \(\StableR{P}{\Ir{I}}{J} = \StableR{P}{\Ir{I} \cap \IDPos{\Body{P}}}{J \cap \IDNeg{\Body{P}}}\).\label{obs:stable:relative:properties:c}
\end{alphaenum}
\end{observation}

\begin{proof}[Proof of \cref{prp:well-founded:relative:monotonicity}]
Both properties can be shown by using the monotonicity of the underlying relative stable operator:
\begin{proofprop}[\((I',J') \leq_p (I,J)\) implies \(\WellR{P}{\Ir{I},\Ir{J}}{I',J'} \leq_p \WellR{P}{\Ir{I},\Ir{J}}{I,J} \)]
Given that \(\StableRO{P}{\Ir{I}}\) is antimonotone and \(J' \cup \Ir{J} \subseteq J \cup \Ir{J}\),
we have \(\StableR{P}{\Ir{I}}{J \cup \Ir{J}} \subseteq \StableR{P}{\Ir{I}}{J' \cup \Ir{J}}\).
Analogously, we can show \(\StableR{P}{\Ir{J}}{I' \cup \Ir{I}} \subseteq \StableR{P}{\Ir{J}}{I \cup \Ir{I}}\).
We get \( (\StableR{P}{\Ir{I}}{J \cup \Ir{J}}, \StableR{P}{\Ir{J}}{I \cup \Ir{I}}) \leq_p (\StableR{P}{\Ir{I}}{J' \cup \Ir{J}}, \StableR{P}{\Ir{J}}{I' \cup \Ir{I}})\).

Hence, \(\WellRO{P}{\Ir{I},\Ir{J}}\) is monotone.
\end{proofprop}
\begin{proofprop}[\((\Ipr{I},\Ipr{J}) \leq_p (\Ir{I},\Ir{J})\) implies \(\WellR{P}{\Ipr{I},\Ipr{J}}{I,J} \leq_p \WellR{P}{\Ir{I},\Ir{J}}{I,J}\)]
We have to show \((\StableR{P}{\Ipr{I}}{J \cup \Ipr{J}}, \StableR{P}{\Ipr{J}}{I \cup \Ipr{I}}) \leq_p (\StableR{P}{\Ir{I}}{J \cup \Ir{J}},\allowbreak \StableR{P}{\Ir{J}}{I \cup \Ir{I}})\).

Given that \(\Ipr{I} \subseteq \Ir{I}\) and \(J \cup \Ir{J} \subseteq J \cup \Ipr{J}\), we obtain \(\StableR{P}{\Ipr{I}}{J \cup \Ipr{J}} \subseteq \StableR{P}{\Ir{I}}{J \cup \Ir{J}}\) using \cref{prp:stable:relative:monotonicity}.
The same argument can be used for the possible atoms of the four-valued interpretations.
Given that \(\Ir{J} \subseteq \Ipr{J}\) and \(I \cup \Ipr{I} \subseteq I \cup \Ir{I}\), we obtain \(\StableR{P}{\Ir{J}}{I \cup \Ir{I}} \subseteq \StableR{P}{\Ipr{J}}{I \cup \Ipr{I}}\) using \cref{prp:stable:relative:monotonicity}.

Hence, we have shown \(\WellR{P}{\Ipr{I},\Ipr{J}}{I,J} \leq_p \WellR{P}{\Ir{I},\Ir{J}}{I,J}\).\qedhere
\end{proofprop}
\end{proof}

\begin{observation}\label{obs:evaluation:properties}
Let \(P\) be an \(\FormulasF\)-program, and \(I\), \(I'\) and \(\Ir{I}\) be two-valued interpretations.

We get the following properties:
\begin{alphaenum}
\item \(I \models P\) and \(\Ir{I} \subseteq I\) implies \(I \models \IDEvalP{P}{\Ir{I}}\),\label{obs:evaluation:properties:a}
\item \(I \models \IDEvalP{P}{\Ir{I}}\) and \(I' \cap \Body{P}^+ \subseteq \Ir{I}\) implies \(I \cup I' \models \IDEvalP{P}{\Ir{I}}\), and\label{obs:evaluation:properties:b}
\item \(I \models \IDEvalP{P}{\Ir{I}}\) implies \(I \models P\).\label{obs:evaluation:properties:c}
\end{alphaenum}
\end{observation}

\begin{lemma}\label{prp:splitting:stable}
Let \(\Fb{P}\) and \(\Ft{P}\) be \(\FormulasF\)-programs,
\(\Ir{I}\) and \(J\) be two-valued interpretations,
\(I = \StableR{\Fb{P} \cup \Ft{P}}{\Ir{I}}{J}\),
\(\Ie{I} = I \cap (\IDPos{\Body{\Fb{P}}} \cap \Head{\Ft{P}})\),
\(\Ib{I} = \StableR{\Fb{P}}{\Ir{I} \cup \Ie{I}}{J}\), and
\({\It{I}} = \StableR{\Ft{P}}{\Ir{I} \cup \Ib{I}}{J}\).

Then, we have \(I = \Ib{I} \cup {\It{I}}\).
\end{lemma}

\begin{proof}
Let \(\Iwt{I} = \Ib{I} \cup {\It{I}}\).
Furthermore, we use the following programs:
\begin{align*}
\Fwhb{P} & = \IDEvalP{\IDReductP{\Fb{P}}{J}}{\Ir{I}} &
\Fwtb{P} & = \IDEvalP{\IDReductP{\Fb{P}}{J}}{\Ir{I} \cup \Ie{I}} = \IDEvalP{\Fwhb{P}}{\Ie{I}} \\
\Fwht{P} & = \IDEvalP{\IDReductP{\Ft{P}}{J}}{\Ir{I}} &
\Fwtt{P} & = \IDEvalP{\IDReductP{\Ft{P}}{J}}{\Ir{I} \cup \Ib{I}} = \IDEvalP{\Fwht{P}}{\Ib{I}}
\end{align*}
Observe that
\begin{align*}
I & = \StableR{\Fb{P} \cup \Ft{P}}{\Ir{I}}{J} = \LeastM{\Fwhb{P} \cup \Fwht{P}}, \\
\Ib{I} & = \StableR{\Fb{P}}{\Ir{I} \cup \Ie{I}}{J} = \LeastM{\Fwtb{P}} \text{, and}\\
\It{I} & = \StableR{\Ft{P}}{\Ir{I} \cup \Ib{I}}{J} = \LeastM{\Fwtt{P}}.
\end{align*}
To show that \(\Iwt{I} \subseteq I\), we show that \(I\) is a model of both \(\Fwtb{P}\) and \(\Fwtt{P}\).
To show that \(I \subseteq \Iwt{I}\), we show that \(\Iwt{I}\) is a model of both \(\Fwhb{P}\) and \(\Fwht{P}\).

\begin{proofprop}[\(I \models \Fwtb{P}\)]
\begin{align*}
I = \LeastM{\Fwhb{P} \cup \Fwht{P}} & \ImpliesT I \models \Fwhb{P}.\\
I \models \Fwhb{P} \AndT \Ie{I} \subseteq I & \ImpliesT I \models \Fwtb{P}\\
& \ByT{\Nref{obs:evaluation:properties}{a}}.
\end{align*}
\end{proofprop}

\begin{proofprop}[\(I \models \Fwtt{P}\)]
\begin{align*}
I = \LeastM{\Fwhb{P} \cup \Fwht{P}} & \ImpliesT I \models \Fwht{P}.\\
I \models \Fwtb{P} \AndT \Ib{I} = \LeastM{\Fwtb{P}} & \ImpliesT \Ib{I} \subseteq I.\\
I \models \Fwht{P} \AndT \Ib{I} \subseteq I & \ImpliesT I \models \Fwtt{P}.\\
& \ByT{\Nref{obs:evaluation:properties}{a}}.
\end{align*}
\end{proofprop}

\begin{proofprop}[\(\Iwt{I} \models \Fwhb{P}\)]
Let \(E = \IDPos{\Body{\Fb{P}}} \cap \Head{\Ft{P}}\):
\begin{align*}
\Iwt{I} \subseteq I \AndT \Ie{I} = I \cap E & \ImpliesT \It{I} \cap E \subseteq \Ie{I}. \\
\It{I} = \LeastM{\Fwtt{P}} & \ImpliesT \It{I} \subseteq \Head{\Fwtt{P}}.\\
\It{I} \cap E \subseteq \Ie{I} \AndT \It{I} \subseteq \Head{\Fwtt{P}} & \ImpliesT \It{I} \cap \IDPos{\Body{\Fb{P}}} \subseteq \Ie{I}. \\
\It{I} \cap \IDPos{\Body{\Fb{P}}} \subseteq \Ie{I} & \ImpliesT \It{I} \cap \IDPos{\Body{\Fwhb{P}}} \subseteq \Ie{I}. \\
\It{I} \cap \IDPos{\Body{\Fwhb{P}}} \subseteq \Ie{I} \AndT \Ib{I} = \LeastM{\Fwtb{P}} & \ImpliesT \Iwt{I} \models \Fwtb{P}\\
& \ByT{\Nref{obs:evaluation:properties}{b}}.\\
\Iwt{I} \models \Fwtb{P} & \ImpliesT \Iwt{I} \models \Fwhb{P}\\
& \ByT{\Nref{obs:evaluation:properties}{c}}.
\end{align*}
\end{proofprop}

\begin{proofprop}[\(\Iwt{I} \models \Fwht{P}\)]
\begin{align*}
\It{I} = \LeastM{\Fwtt{P}} & \ImpliesT \Iwt{I} \models \Fwtt{P}\\
& \ByT{\Nref{obs:evaluation:properties}{b}}.\\
\Iwt{I} \models \Fwtt{P} & \ImpliesT \Iwt{I} \models \Fwht{P}\\
& \ByT{\Nref{obs:evaluation:properties}{c}}.\qedhere
\end{align*}
\end{proofprop}
\end{proof}

\begin{proof}[Proof of \cref{prp:splitting:well-founded}]\allowdisplaybreaks Let \(P = \Fb{P} \cup \Ft{P}\) and \(E = \IDPosNeg{\Body{\Fb{P}}} \cap \Head{\Ft{P}}\).
We begin by evaluating \(P\), \(\Fb{P}\) and \(\Ft{P}\) w.r.t. \((I,J)\) and obtain
\begin{align*}
(I,J)
& = \WellR{P}{\Ir{I},\Ir{J}}{I,J} \\
& = (\StableR{P}{\Ir{I}}{\Ir{J} \cup J}, \StableR{P}{\Ir{J}}{\Ir{I} \cup I}), \\
(\Iwhb{I},\Iwhb{J})
& = \WellR{\Fb{P}}{(\Ir{I},\Ir{J}) \sqcup (\Ie{I},\Ie{J})}{I,J} \\
& = (\StableR{\Fb{P}}{\Ir{I} \cup \Ie{I}}{\Ir{J} \cup \Ie{J} \cup J}, \StableR{\Fb{P}}{\Ir{J} \cup \Ie{J}}{\Ir{I} \cup \Ie{I} \cup I})\text{, and}\\
(\Iwht{I},\Iwht{J})
& = \WellR{\Ft{P}}{(\Ir{I},\Ir{J}) \sqcup (\Iwhb{I},\Iwhb{J})}{I,J} \\
& = (\StableR{\Ft{P}}{\Ir{I} \cup \Iwhb{I}}{\Ir{J} \cup \Iwhb{J} \cup J}, \StableR{\Ft{P}}{\Ir{J} \cup \Iwhb{J}}{\Ir{I} \cup \Iwhb{I} \cup I}).\\
\intertext{Using \((\Ie{I},\Ie{J}) \sqsubseteq (I,J)\), we get}
(\Iwhb{I},\Iwhb{J})
& = (\StableR{\Fb{P}}{\Ir{I} \cup \Ie{I}}{\Ir{J} \cup J}, \StableR{\Fb{P}}{\Ir{J} \cup \Ie{J}}{\Ir{I} \cup I}). \\
\intertext{By \cref{prp:splitting:stable} and \nref{obs:stable:relative:properties}{c}, we get}
(\Iwhb{I},\Iwhb{J})
& \sqsubseteq (I,J),\\
(\Iwht{I},\Iwht{J})
& = (\StableR{\Ft{P}}{\Ir{I} \cup \Ie{I}}{\Ir{J} \cup J}, \StableR{\Ft{P}}{\Ir{J} \cup \Ie{J}}{\Ir{I} \cup I})\text{, and} \\
(I,J)
& = (\Iwhb{I},\Iwhb{J}) \sqcup (\Iwht{I},\Iwht{J}).
\end{align*}
We first show \((\Ib{I},\Ib{J}) = (\Iwhb{I},\Iwhb{J})\) and then \(({\It{I}}, {\It{J}}) = (\Iwht{I},\Iwht{J})\).

\begin{proofprop}[\((\Ib{I},\Ib{J}) \leq_p (\Iwhb{I},\Iwhb{J})\)]

\begin{align*}
(\Iwhb{I},\Iwhb{J}) \sqsubseteq (I,J) & \AndT \\
(\Ie{I}, \Ie{J}) = (I,J) \sqcap E & \ImpliesT (\Iwhb{I}, \Iwhb{J}) \sqcup (\Ie{I}, \Ie{J}) \sqsubseteq (I,J).\\
(I,J) = (\Iwhb{I},\Iwhb{J}) \sqcup (\Iwht{I},\Iwht{J}) & \ImpliesT (\Iwht{I},\Iwht{J}) \sqsubseteq (I,J).\\
(\Iwht{I},\Iwht{J}) \sqsubseteq \Head{\Ft{P}} & \ImpliesT (\Iwht{I},\Iwht{J}) \sqcap \IDPosNeg{\Body{\Fb{P}}} \sqsubseteq (\Iwht{I},\Iwht{J}) \sqcap E.\\
\rlap{\((\Iwht{I},\Iwht{J}) \sqcap \IDPosNeg{\Body{\Fb{P}}} \sqsubseteq (\Iwht{I},\Iwht{J}) \sqcap E \AndT\)}
\phantom{(\Iwht{I},\Iwht{J}) \sqcap \IDPosNeg{\Body{\Fb{P}}} \sqsubseteq (\Ie{I}, \Ie{J})} \\
(\Iwht{I},\Iwht{J}) \sqsubseteq (I,J) & \ImpliesT (\Iwht{I},\Iwht{J}) \sqcap \IDPosNeg{\Body{\Fb{P}}} \sqsubseteq (\Ie{I}, \Ie{J}).\\
(\Iwht{I},\Iwht{J}) \sqcap \IDPosNeg{\Body{\Fb{P}}} \sqsubseteq (\Ie{I}, \Ie{J}) & \AndT \\
(I,J) = (\Iwhb{I},\Iwhb{J}) \sqcup (\Iwht{I},\Iwht{J}) & \ImpliesT (I,J) \sqcap \IDPosNeg{\Body{\Fb{P}}} \sqsubseteq (\Iwhb{I},\Iwhb{J}) \sqcup (\Ie{I},\Ie{J}).
\end{align*}

With the above, we use \Nref{obs:stable:relative:properties}{c} to show that \((\Iwhb{I},\Iwhb{J})\) is a fixed point of \(\WellRO{\Fb{P}}{(\Ir{I},\Ir{J}) \sqcup (\Ir{J},\Ie{J})}\):
\begin{align*}
(\Iwhb{I},\Iwhb{J})
& = \WellR{\Fb{P}}{(\Ir{I},\Ir{J}) \sqcup (\Ie{I},\Ie{J})}{I,J} \\
& = \WellR{\Fb{P}}{(\Ir{I},\Ir{J}) \sqcup (\Ie{I},\Ie{J})}{(I,J) \sqcap \IDPosNeg{\Body{\Fb{P}}}} \\
& = \WellR{\Fb{P}}{(\Ir{I},\Ir{J}) \sqcup (\Ie{I},\Ie{J})}{(\Iwhb{I},\Iwhb{J}) \sqcup (\Ie{I},\Ie{J})} \\
& = \WellR{\Fb{P}}{(\Ir{I},\Ir{J}) \sqcup (\Ie{I},\Ie{J})}{\Iwhb{I},\Iwhb{J}}
\end{align*}

Thus, by \Nref{thm:knaster-tarski}{c}, \((\Ib{I},\Ib{J}) \leq_p (\Iwhb{I},\Iwhb{J})\).
\end{proofprop}

\begin{proofprop}[\((\Ib{I},\Ib{J}) = (\Iwhb{I},\Iwhb{J})\)]
To show the property, let
\begin{align*}
(\Iwt{I},\Iwt{J})
& = (\Ib{I}, \Ib{J}) \sqcup (\Ie{I},\Ie{J}) \sqcup (\Iwht{I},\Iwht{J}),\\
(\Iwte{I},\Iwte{J})
& = \WellR{P}{\Ir{I},\Ir{J}}{\Iwt{I},\Iwt{J}} \sqcap E,\\
(\Iwtb{I},\Iwtb{J})
& = (\StableR{\Fb{P}}{\Ir{I} \cup \Iwte{I}}{\Ir{J} \cup \Iwt{J}}, \StableR{\Fb{P}}{\Ir{J} \cup \Iwte{J}}{\Ir{I} \cup \Iwt{I}}),\\
(\Iwtt{I},\Iwtt{J})
& = (\StableR{\Ft{P}}{\Ir{I} \cup \Iwtb{I}}{\Ir{J} \cup \Iwt{J}}, \StableR{\Ft{P}}{\Ir{J} \cup \Iwtb{J}}{\Ir{I} \cup \Iwt{I}}),\text{\ and}\\
\WellR{P}{\Ir{I},\Ir{J}}{\Iwt{I},\Iwt{J}} & = (\Iwtb{I},\Iwtb{J}) \sqcup (\Iwtt{I},\Iwtt{J}) \text{\ by \cref{prp:splitting:stable}.}
\end{align*}

We get:
\begin{align*}
(\Ib{I},\Ib{J}) \leq_p (\Iwhb{I},\Iwhb{J}) & \ImpliesT (\Iwt{I},\Iwt{J}) \leq_p (I,J).\\
(\Iwt{I},\Iwt{J}) \leq_p (I,J) & \ImpliesT \WellR{P}{\Ir{I},\Ir{J}}{\Iwt{I},\Iwt{J}} \leq_p (I,J).\\
\WellR{P}{\Ir{I},\Ir{J}}{\Iwt{I},\Iwt{J}} \leq_p (I,J) & \ImpliesT (\Iwte{I},\Iwte{J}) \leq_p (\Ie{I},\Ie{J}).\\
(\Iwht{I},\Iwht{J}) \sqcap \IDPosNeg{\Body{\Fb{P}}} \sqsubseteq (\Ie{I}, \Ie{J}) & \ImpliesT (\Iwt{I},\Iwt{J}) \sqcap \IDPosNeg{\Body{\Fb{P}}} \sqsubseteq (\Iwhb{I},\Iwhb{J}) \sqcup (\Ie{I},\Ie{J}).\\
(\Iwt{I},\Iwt{J}) \sqcap \IDPosNeg{\Body{\Fb{P}}} \sqsubseteq (\Iwhb{I},\Iwhb{J}) \sqcup (\Ie{I},\Ie{J}) & \ImpliesT \Iwtb{I} = \StableR{\Fb{P}}{\Ir{I} \cup \Iwte{I}}{\Ir{J} \cup \Ie{J} \cup \Ib{J}}\\
& \ImpliesAndT \Iwtb{J} = \StableR{\Fb{P}}{\Ir{J} \cup \Iwte{J}}{\Ir{I} \cup \Ie{I} \cup \Ib{I}}.\\
& \ByT{\Nref{obs:stable:relative:properties}{c}}.\\
\Iwtb{I} = \StableR{\Fb{P}}{\Ir{I} \cup \Iwte{I}}{\Ir{J} \cup \Ie{J} \cup \Ib{J}} & \AndT \\
\Iwtb{J} = \StableR{\Fb{P}}{\Ir{J} \cup \Iwte{J}}{\Ir{I} \cup \Ie{I} \cup \Ib{I}} & \AndT \\
\Ib{I} = \StableR{\Fb{P}}{\Ir{I} \cup \Ie{I}}{\Ir{J} \cup \Ie{J} \cup \Ib{J}} & \AndT \\
\Ib{J} = \StableR{\Fb{P}}{\Ir{J} \cup \Ie{J}}{\Ir{I} \cup \Ie{I} \cup \Ib{I}} & \AndT \\
(\Iwte{I},\Iwte{J}) \leq_p (\Ie{I},\Ie{J}) & \ImpliesT (\Iwtb{I},\Iwtb{J}) \leq_p (\Ib{I},\Ib{J})\\
& \ByT{\cref{prp:well-founded:relative:monotonicity}}.\\
(\Iwtb{I},\Iwtb{J}) \leq_p (\Ib{I},\Ib{J}) & \AndT \\
(\Ib{I},\Ib{J}) \leq_p (\Iwhb{I},\Iwhb{J}) & \ImpliesT (\Iwtb{I},\Iwtb{J}) \leq_p (\Iwhb{I},\Iwhb{J}). \\
\Iwht{I} = \StableR{\Ft{P}}{\Ir{I} \cup \Iwhb{I}}{\Ir{J} \cup J} & \AndT \\
\Iwht{J} = \StableR{\Ft{P}}{\Ir{J} \cup \Iwhb{J}}{\Ir{I} \cup I} & \AndT \\
\Iwtt{I} = \StableR{\Ft{P}}{\Ir{I} \cup \Iwtb{I}}{\Ir{J} \cup \Iwt{J}} & \AndT \\
\Iwtt{J} = \StableR{\Ft{P}}{\Ir{J} \cup \Iwtb{J}}{\Ir{I} \cup \Iwt{I}} & \AndT \\
(\Iwtb{I},\Iwtb{J}) \leq_p (\Iwhb{I},\Iwhb{J}) & \AndT \\
(\Iwt{I},\Iwt{J}) \leq_p (I,J) & \ImpliesT  (\Iwtt{I},\Iwtt{J}) \leq_p (\Iwht{I},\Iwht{J}) \\
& \ByT{\cref{prp:stable:relative:monotonicity}}.\\
\WellR{P}{\Ir{I},\Ir{J}}{\Iwt{I},\Iwt{J}} = (\Iwtb{I},\Iwtb{J}) \sqcup (\Iwtt{I},\Iwtt{J}) & \AndT\\
(\Iwtb{I},\Iwtb{J}) \leq_p (\Ib{I},\Ib{J}) & \AndT \\
(\Iwtt{I},\Iwtt{J}) \leq_p (\Iwht{I},\Iwht{J}) & \ImpliesT
\WellR{P}{\Ir{I},\Ir{J}}{\Iwt{I},\Iwt{J}} \leq_p (\Ib{I},\Ib{J}) \sqcup (\Iwht{I},\Iwht{J}).\\
(\Iwte{I},\Iwte{J}) \leq_p (\Ie{I},\Ie{J}) & \AndT\\
(\Iwte{I},\Iwte{J}) \sqsubseteq \WellR{P}{\Ir{I},\Ir{J}}{\Iwt{I},\Iwt{J}} & \AndT\\
\WellR{P}{\Ir{I},\Ir{J}}{\Iwt{I},\Iwt{J}} \leq_p (\Ib{I},\Ib{J}) \sqcup (\Iwht{I},\Iwht{J}) & \AndT\\
(\Iwt{I},\Iwt{J}) = (\Ib{I}, \Ib{J}) \sqcup (\Ie{I},\Ie{J}) \sqcup (\Iwht{I},\Iwht{J}) & \ImpliesT
\WellR{P}{\Ir{I},\Ir{J}}{\Iwt{I},\Iwt{J}} \leq_p (\Iwt{I},\Iwt{J}).\\
\WellRM{P}{\Ir{I},\Ir{J}} = (I,J) & \AndT \\
\WellR{P}{\Ir{I},\Ir{J}}{\Iwt{I},\Iwt{J}} \leq_p (\Iwt{I},\Iwt{J}) & \ImpliesT
(I,J) \leq_p (\Iwt{I},\Iwt{J}) \\
& \ByT{\Nref{thm:knaster-tarski}{a}}.\\
(\Iwt{I},\Iwt{J}) \leq_p (I,J) \AndT (I,J) \leq_p (\Iwt{I},\Iwt{J}) & \ImpliesT
(I,J) = (\Iwt{I},\Iwt{J}).\\
(\Iwt{I},\Iwt{J}) = (\Ib{I}, \Ib{J}) \sqcup (\Ie{I},\Ie{J}) \sqcup (\Iwht{I},\Iwht{J}) & \ImpliesT (\Ib{I}, \Ib{J}) \sqcup (\Ie{I}, \Ie{J}) \sqsubseteq (\Iwt{I},\Iwt{J}).\\
(\Iwt{I},\Iwt{J}) = (\Ib{I}, \Ib{J}) \sqcup (\Ie{I},\Ie{J}) \sqcup (\Iwht{I},\Iwht{J}) & \AndT \\
(\Iwht{I},\Iwht{J}) \sqcap \IDPosNeg{\Body{\Fb{P}}} \sqsubseteq (\Ie{I}, \Ie{J}) & \ImpliesT (\Iwt{I},\Iwt{J}) \sqcap \IDPosNeg{\Body{\Fb{P}}} \sqsubseteq (\Ib{I},\Ib{J}) \sqcup (\Ie{I},\Ie{J}).\\
(I,J) = (\Iwt{I},\Iwt{J}) & \AndT \\
(\Ib{I}, \Ib{J}) \sqcup (\Ie{I}, \Ie{J}) \sqsubseteq (\Iwt{I},\Iwt{J}) & \AndT\\
(\Iwt{I},\Iwt{J}) \sqcap \IDPosNeg{\Body{\Fb{P}}} \sqsubseteq (\Ib{I},\Ib{J}) \sqcup (\Ie{I},\Ie{J}) & \AndT \\
(\Ib{I},\Ib{J}) = \WellR{\Fb{P}}{(\Ir{I},\Ir{J}) \sqcup (\Ie{I},\Ie{J})}{\Ib{I},\Ib{J}} & \AndT\\
(\Iwhb{I},\Iwhb{J}) = \WellR{\Fb{P}}{(\Ir{I},\Ir{J}) \sqcup (\Ie{I},\Ie{J})}{I,J} & \ImpliesT (\Ib{I}, \Ib{J}) = (\Iwhb{I},\Iwhb{J})\\
& \ByT{\Nref{obs:stable:relative:properties}{c}}.
\end{align*}
\end{proofprop}

\begin{proofprop}[\((\Iwht{I},\Iwht{J}) = ({\It{I}},{\It{J}})\)]
Observe that the lemma can be applied with \(\Fb{P}\) and \(\Ft{P}\) exchanged.
Let
\begin{align*}
\Iwt{E}             & = \IDPosNeg{\Body{\Ft{P}}} \cap \Head{\Fb{P}}\text{,}\\
(\Iwte{I},\Iwte{J}) & = (I,J) \cap \Iwt{E} \text{,}\\
(\Iwtt{I},\Iwtt{J}) & = \WellRM{\Ft{P}}{(\Ir{I},\Ir{J}) \sqcup (\Iwte{I},\Iwte{J})}\text{, and}\\
(\Iwtb{I},\Iwtb{J}) & = \WellR{\Fb{P}}{(\Ir{I},\Ir{J}) \sqcup (\Iwtt{I},\Iwtt{J})}{I,J}.\\
\intertext{Using the properties shown so far, we obtain}
(I,J) & = (\Iwtt{I},\Iwtt{J}) \sqcup (\Iwtb{I},\Iwtb{J}).
\end{align*}
With this we get:
\begin{align*}
(\Ib{I},\Ib{J}) = (\Iwhb{I},\Iwhb{J}) & \AndT\\
(I,J) = (\Iwhb{I},\Iwhb{J}) \sqcup (\Iwht{I},\Iwht{J}) & \AndT \\
(\Ib{I},\Ib{J}) \sqsubseteq \Head{\Fb{P}} & \AndT\\
(\Iwte{I}, \Iwte{J}) = (I,J) \sqcap \Iwt{E} & \ImpliesT
(\Ib{I},\Ib{J}) \sqcap \IDPosNeg{\Body{\Ft{P}}} \sqsubseteq (\Iwte{I},\Iwte{J}).\\
(I,J) = (\Iwtt{I},\Iwtt{J}) \sqcup (\Iwtb{I},\Iwtb{J}) & \AndT \\
(\Iwtb{I},\Iwtb{J}) \sqsubseteq \Head{\Fb{P}} & \AndT\\
(\Iwte{I}, \Iwte{J}) = (I,J) \sqcap \Iwt{E} & \ImpliesT
(\Iwtb{I},\Iwtb{J}) \sqcap \IDPosNeg{\Body{\Ft{P}}} \sqsubseteq (\Iwte{I}, \Iwte{J}).\\
(\Iwtb{I},\Iwtb{J}) \sqcap \IDPosNeg{\Body{\Ft{P}}} \sqsubseteq (\Iwte{I}, \Iwte{J}) & \AndT \\
(I,J) = (\Iwtt{I},\Iwtt{J}) \sqcup (\Iwtb{I},\Iwtb{J}) & \ImpliesT
(I,J) \sqcap \IDPosNeg{\Body{\Fb{P}}} \sqsubseteq (\Iwtt{I},\Iwtt{J}) \sqcup (\Iwte{I},\Iwte{J}).\\
(\Ib{I},\Ib{J}) \sqcap \IDPosNeg{\Body{\Ft{P}}} \sqsubseteq (\Iwte{I},\Iwte{J}) & \AndT \\
(I,J) \sqcap \IDPosNeg{\Body{\Fb{P}}} \sqsubseteq (\Iwtt{I},\Iwtt{J}) \sqcup (\Iwte{I},\Iwte{J}) & \AndT \\
(\Iwtt{I},\Iwtt{J}) = \WellR{\Ft{P}}{(\Ir{I},\Ir{J}) \sqcup (\Iwte{I},\Iwte{J})}{\Iwtt{I},\Iwtt{J}} & \AndT \\
(\Iwht{I},\Iwht{J}) = \WellR{\Ft{P}}{(\Ir{I},\Ir{J}) \sqcup (\Ib{I},\Ib{J})}{I,J} & \ImpliesT
(\Iwtt{I},\Iwtt{J}) = (\Iwht{I},\Iwht{J})\\
& \ByT{\Nref{obs:stable:relative:properties}{c}}.\\
(\Ib{I},\Ib{J}) \sqcap \IDPosNeg{\Body{\Ft{P}}} \sqsubseteq (\Iwte{I},\Iwte{J}) & \AndT\\
(\Iwtt{I},\Iwtt{J}) = \WellRM{\Ft{P}}{(\Ir{I},\Ir{J}) \sqcup (\Iwte{I},\Iwte{J})} & \AndT \\
(\It{I},\It{J}) = \WellRM{\Ft{P}}{(\Ir{I},\Ir{J}) \sqcup (\Ib{I},\Ib{J})} & \ImpliesT
(\Iwtt{I},\Iwtt{J}) = (\It{I},\It{J})\\
& \ByT{\Nref{obs:stable:relative:properties}{c}}.
\end{align*}
Thus, we get \((\Iwht{I},\Iwht{J}) = ({\It{I}},{\It{J}})\).\qedhere
\end{proofprop}
\end{proof}

\begin{proof}[Proof of \cref{thm:sequence:well-founded}]
The theorem can be shown by transfinite induction over the sequence indices.
We do not give the full induction proof here but focus on the key idea.
Let \((\Ipi{I}{i}, \Ipi{J}{i})\) be the intermediate interpretations as in \cref{eq:well-founded:sequence:intermediate}
when computing the well-founded model of the sequence.
Furthermore, let
\begin{align*}
(\Ii{I}{i}, \Ii{J}{i}) & = \WellRM{\Fi{P}{i}}{(\Iri{I}{i},\Iri{J}{i}) \sqcup (\Iei{I}{i},\Iei{J}{i})}
\intertext{be the intermediate interpretations
where \((\Iri{I}{i}, \Iri{J}{i})\) is the union of the intermediate interpretations as in \cref{eq:well-founded:sequence:relative} and}
(\Iei{I}{i}, \Iei{J}{i}) & = (I,J) \cap E_i
\end{align*}
with \(E_i\) as in \cref{eq:well-founded:sequence:external}.

Observe that with \cref{prp:splitting:well-founded}, we have
\(\WellM{\bigcup_{i \in \Index}\Fi{P}{i}} = \bigcup_{i \in \Index}(\Ii{I}{i}, \Ii{J}{i})\).
By \cref{prp:well-founded:relative:monotonicity}, we have \((\Ipi{I}{i}, \Ipi{J}{i}) \leq_p (\Ii{I}{i}, \Ii{J}{i})\)
and, thus, we obtain that \(\WellM{{(\Fi{P}{i})}_{i \in \Index}} \leq_p \WellM{\bigcup_{i \in \Index}\Fi{P}{i}}\).
\end{proof}

\begin{proof}[Proof of \cref{thm:sequence:simplification}]
By \cref{thm:sequence:well-founded}, we have
\begin{align*}
\bigsqcup_{i\in\Index}(\Ii{I}{i},\Ii{J}{i}) & \leq_p (I,J). \\
\intertext{We get \(\bigcup_{i \in \Index}\Ii{I}{i} \subseteq I\) and, thus,}
\bigcup_{i \leq k}\Ii{I}{i}
& = \Iri{I}{k} \cup \Ii{I}{k} \subseteq I.\\
\intertext{Using \(J \subseteq \bigcup_{i \in \Index}\Ii{J}{i}\) and \(\Ii{J}{i} \subseteq \Head{\Fi{P}{i}}\), we get}
J \cap \IDPosNeg{\Body{\Fi{P}{k}}}
& \subseteq (\bigcup_{i \leq k}\Ii{J}{i} \cup \bigcup_{k < i}\Head{\Fi{P}{i}}) \cap \IDPosNeg{\Body{\Fi{P}{k}}}\\
& \subseteq (\bigcup_{i \leq k}\Ii{J}{i} \cup \Ii{E}{k}) \cap \IDPosNeg{\Body{\Fi{P}{k}}}\\
& \subseteq (\Iri{J}{k} \cup \Ii{J}{k} \cup \Ii{E}{k}) \cap \IDPosNeg{\Body{\Fi{P}{k}}}.
\end{align*}
Using both results, we obtain
\begin{align*}
((\Iri{I}{k},\Iri{J}{k}) \sqcup (\emptyset,\Ii{E}{k}) \sqcup (\Ii{I}{k},\Ii{J}{k})) \sqcap \IDPosNeg{\Body{\Fi{P}{k}}} & \leq_p (I,J) \sqcap \IDPosNeg{\Body{\Fi{P}{k}}}.
\end{align*}
Because the body literals determine the simplification, we get
\begin{align*}
\Simp{\Fi{P}{k}}{I,J} & \subseteq \Simp{\Fi{P}{k}}{(\Iri{I}{k},\Iri{J}{k}) \sqcup (\emptyset,\Ii{E}{k}) \sqcup (\Ii{I}{k},\Ii{J}{k})}.
\qedhere
\end{align*}
\end{proof}

\begin{lemma}\label{cor:sequence:simplification:superset}
Let \({(\Fi{P}{i})}_{i \in \Index}\) be a sequence of \(\FormulasR\)-programs,
and \((I,J)\) be the well-founded model of \(\bigcup_{i \in \Index}{\Fi{P}{i}}\).

Then, \(\bigcup_{i \in \Index}{\Fi{P}{i}}\) and \(\bigcup_{i \in \Index}{\Fi{Q}{i}}\)
with \(\Simp{\Fi{P}{i}}{I,J} \subseteq \Fi{Q}{i} \subseteq \Fi{P}{i}\)
have the same well-founded and stable models.
\end{lemma}

\begin{proof}
This lemma is a direct consequence of \cref{thm:well-founded:simplification:superset,cor:stable:simplification:superset,thm:sequence:simplification}.
\end{proof}

\begin{proof}[Proof of \cref{cor:sequence:simplification:stable}]
This corollary is a direct consequence of \cref{thm:sequence:simplification,cor:sequence:simplification:superset}.
\end{proof}

\begin{proof}[Proof of \cref{cor:sequence:well-founded:stratified}]
This can be proven in the same way as \cref{thm:sequence:well-founded} but note that because \(\Ii{E}{i}\) is empty,
we get \((\Ipi{I}{i}, \Ipi{J}{i}) = (\Ii{I}{i}, \Ii{J}{i})\).
\end{proof}

\begin{proof}[Proof of \cref{cor:sequence:simplification:stratified}]
This can be proven in the same way as \cref{thm:sequence:simplification} but note that because \(\Ii{E}{i}\) is empty,
all \(\leq_p\) and most \(\subseteq\) relations can be replaced with equivalences.
\end{proof}

Whenever head atoms do not interfere with negative body literals,
the relative well founded-model of a program can be calculated with just two applications of the relative stable operator.

\begin{lemma}\label{lem:well-founded:stratified}
Let \(P\) be an \FormulasF-program
such that \(\IDNeg{\Body{P}} \cap \Head{P} = \emptyset\)
and \((\Ir{I},\Ir{J})\) be a four-valued interpretation.

Then, \(\WellRM{P}{\Ir{I},\Ir{J}} = (\StableR{P}{\Ir{I}}{\Ir{J}}, \StableR{P}{\Ir{J}}{\Ir{I}})\).
\end{lemma}

\begin{proof}
Let \((I,J) = \WellRM{P}{\Ir{I},\Ir{J}}\).

We have \(J = \StableR{P}{\Ir{J}}{\Ir{I} \cup I}\).
By \Nref{obs:stable:relative:properties}{b}, we get \(J \subseteq \Head{P}\).
With this and \(\IDNeg{\Body{P}} \cap \Head{P} = \emptyset\), we get \({\Body{P}}^- \cap J = \emptyset\).
Thus, \(\StableR{P}{\Ir{I}}{\Ir{J} \cup J} = \StableR{P}{\Ir{I}}{\Ir{J}}\) by \Nref{obs:stable:relative:properties}{c}.

The same arguments apply to show \(\StableR{P}{\Ir{J}}{\Ir{I} \cup I} = \StableR{P}{\Ir{J}}{\Ir{I}}\).
\end{proof}

Any sequence as in \cref{cor:sequence:well-founded:stratified} in which
each \(P_i\) additionally satisfies the precondition of \cref{lem:well-founded:stratified}
has a total well-founded model.
Furthermore, the well-founded model of such a sequence can be calculated
with just two (independent) applications of the relative stable operator per program \(P_i\) in the sequence.

\begin{proof}[Proof of \cref{prp:translation:equivalence}]
We use \cref{prp:strong:equivalence} to show that both formulas are strongly equivalent.
\begin{proofprop}[\(I \models \Reduct{\TransA{a}}{J}\) implies \(I \models \Reduct{\Trans{a}}{J}\) for arbitrary interpretations \(I\)]
The formulas \(\TransA{a}\) and \(\Trans{a}\) only differ in the consequents of their implications.
Observe that the consequents in \(\TransA{a}\) are stronger than the ones in \(\Trans{a}\).
Thus, it follows that \(\TransA{a}\) is stronger than \(\Trans{a}\).
Furthermore, observe that the same holds for their reducts.
\end{proofprop}
\begin{proofprop}[\(I \not\models \Reduct{\TransA{a}}{J}\) implies \(I \not\models \Reduct{\Trans{a}}{J}\) for arbitrary interpretations \(I\)]
Let \(\ElemG\) be the set of all instance of the aggregate elements of \(a\).
Because \(I \not\models \Reduct{\TransA{a}}{J}\),
there must be a set \(\ElemJ \subseteq \ElemG\) such that
\(\ElemJ \notjustifies a\), \(I \models \ReductC{\Trans{\ElemJ}^\wedge}{J}\), and \(I \not\models \ReductC{\TransAD{a}{\ElemJ}^\vee}{J}\).
With this, we construct the set
\begin{align*}
\widehat{\ElemJ} = \ElemJ \cup \{ e \in \ElemG \setminus \ElemJ \mid I \models \Reduct{\Trans{e}}{J} \}.
\end{align*}
\begin{align*}
\text{The construction of \(\widehat{\ElemJ}\)} & \AndT\\
\ElemJ \notjustifies a & \AndT\\
I \not\models \ReductC{\TransAD{a}{\ElemJ}^\vee}{J}
& \ImpliesT \widehat{\ElemJ} \notjustifies a \AndT I \not\models \ReductC{\TransD{a}{\widehat{\ElemJ}}^\vee}{J}.\\
\text{The construction of \(\widehat{\ElemJ}\)} & \AndT\\
I \models \ReductC{\Trans{\ElemJ}^\wedge}{J}
& \ImpliesT I \models \ReductC{\Trans{\widehat{\ElemJ}}^\wedge}{J}.\\
\widehat{\ElemJ} \notjustifies a & \AndT \\
I \not\models \ReductC{\TransD{a}{\widehat{\ElemJ}}^\vee}{J} & \AndT \\
I \models \ReductC{\Trans{\widehat{\ElemJ}}^\wedge}{J}
& \ImpliesT I \not\models \Reduct{\Trans{a}}{J}.\qedhere
\end{align*}
\end{proofprop}
\end{proof}

\begin{proof}[Proof of \cref{prp:translation:monotonicity}]
Let \(\ElemG\) be the set of ground instances of the aggregate elements of~\(a\).
Furthermore, observe that a monotone aggregate \(a\) is either constantly true or not justified by the empty set.

In case that \(\TransA{a} \equiv \top\), we get \(\IDReductP{\TransA{a}}{I} \equiv \top\) and the lemma holds.

Next, we consider the case that the empty set does not justify the aggregate.
Observe that \(\TransAD{a}{\emptyset}\) is stronger than \(\TransAD{a}{\ElemJ}\) for any \(\ElemJ \subseteq \ElemG\).
And, we have that \(\TransA{a}\) contains the implication \(\top \rightarrow \TransAD{a}{\emptyset}\).
Because of this, we have \(\TransA{a} \equiv \TransAD{a}{\emptyset}\).
Furthermore, all consequents in \(\TransA{a}\) are positive formulas and, thus, not modified by the reduct.
Thus, the reduct \(\IDReductPC{\top \rightarrow \TransAD{a}{\emptyset}}{I}\) is equal to \(\top \rightarrow \TransAD{a}{\emptyset}\).
And as before, it is stronger than all other implications in \(\IDReductP{\TransA{a}}{I}\).
Hence, we get \(\IDReductP{\TransA{a}}{I} \equiv \TransAD{a}{\emptyset}\).
\end{proof}

\begin{proof}[Proof of \cref{prp:translation:restricted:properties}]
Remember that the translation \(\TransA{a}\) is a conjunction of implications.
The antecedents of the implications are conjunctions of aggregate elements and the consequents are disjunctions of conjunctions of aggregate elements.

\begin{proofprop}[\ref{prp:translation:restricted:properties:a}]
If the conjunction in an antecedent contains an element not in \(J\), then the conjunction is not satisfied by \(X\) and the implication does not affect the satisfiability of \(\TransA{a}\).
If a conjunction in a consequent contains an element not in \(J\), then \(X\) does not satisfy the conjunction and the conjunction does not affect the satisfiability of the encompassing disjunction.
Observe that both cases correspond exactly to those subformulas omitted in~\(\TransAD{J}{a}\).
\end{proofprop}

The remaining two properties follow for similar reasons.
\end{proof}

The next observation summarizes how dependencies transfer from non-ground aggregate programs to the corresponding ground \(\FormulasR\)-programs.

\begin{observation}\label{obs:depend}
Let \(\Fi{P}{1}\) and \(\Fi{P}{2}\) be aggregate programs, and \(G_1 = \TransA{P_1}\) and \(G_2 = \TransA{P_2}\).

Then,
\begin{alphaenum}
\item \(\Fi{P}{1}\) does not depend on \(\Fi{P}{2}\) implies \(\IDPosNeg{\Body{\Fi{G}{1}}} \cap \Head{\Fi{G}{2}} = \emptyset\),\label{obs:depend:a}
\item \(\Fi{P}{1}\) does not positively depend on \(\Fi{P}{2}\) implies \(\IDPos{\Body{\Fi{G}{1}}} \cap \Head{\Fi{G}{2}} = \emptyset\),\label{obs:depend:b}
\item \(\Fi{P}{1}\) does not negatively depend on \(\Fi{P}{2}\) implies \(\IDNeg{\Body{\Fi{G}{1}}} \cap \Head{\Fi{G}{2}} = \emptyset\).\label{obs:depend:c}
\end{alphaenum}
\end{observation}

The next two lemmas pin down important properties of instantiation sequences.
First of all,
there are no external atoms in the components of instantiation sequences.

\begin{lemma}\label{lem:sequence:external}
Let \(P\) be an aggregate program and \(\Fsqi{P}{\Index}{i}\) be an instantiation sequence for \(P\).

Then, for the sequence \(\Fsqi{G}{\Index}{i}\) with \(\Fi{G}{i} = \TransA{\Fi{P}{i}}\),
we have \(\Ii{E}{i}=\emptyset\) for each \(i \in \Index\) where \(\Ii{E}{i}\) is defined as in \cref{eq:well-founded:sequence:external}.
\end{lemma}

\begin{proof}
This lemma is a direct consequence of \Nref{obs:depend}{a} and the anti-symmetry
of the dependency relation between components.
\end{proof}

\begin{proof}[Proof of \cref{thm:instantiation-sequence:well-founded}]
This theorem is a direct consequence of \cref{lem:sequence:external,cor:sequence:well-founded:stratified}.
\end{proof}

Moreover, for each stratified component in an instantiation sequence, we obtain a total well-founded model.

\begin{lemma}\label{lem:sequence:well-founded:stratified}
Let \(P\) be an aggregate program and \(\Fsqi{P}{\Index}{i}\) be an instantiation sequence for \(P\).

Then, for the sequence \(\Fsqi{G}{\Index}{i}\) with \(\Fi{G}{i} = \TransA{\Fi{P}{i}}\),
we have \(\Ii{I}{i}=\Ii{J}{i}=\StableR{\Fi{G}{i}}{\Iri{I}{i}}{\Iri{I}{i}}\)
for each stratified component \(\Fi{P}{i}\) where
\((\Iri{I}{i},\Iri{J}{i})\) and \((\Ii{I}{i},\Ii{J}{i})\) are defined as in \cref{eq:well-founded:sequence:relative,eq:well-founded:sequence:intermediate}
in the construction of the well-founded model of \(\Fsqi{G}{\Index}{i}\) in \cref{def:well-founded:sequence}.
\end{lemma}

\begin{proof}
In the following, we use \(\Ii{E}{i}\) and \((\Iri{I}{i},\Iri{J}{i})\) for the sequence \(\Fsqi{G}{\Index}{i}\) as defined in \cref{eq:well-founded:sequence:external,eq:well-founded:sequence:relative}.
Note that, by \cref{lem:sequence:external}, we have \(\Ii{E}{i} = \emptyset\).

We prove by induction.

\begin{proofpart}[Base]
Let \(\Fi{P}{i}\) be a stratified component that does not depend on any other component.
Because \(\Fi{P}{i}\) does not depend on any other component, we have \(\bigcup_{j < i} \Head{\Fi{G}{j}} \cap \IDPosNeg{\Body{\Fi{G}{i}}} = \emptyset\).
Thus, by \Nref{obs:stable:relative:properties}{b}, we get \(\Iri{I}{i} \cap \IDPosNeg{\Body{\Fi{G}{i}}} = \Iri{J}{i} \cap \IDPosNeg{\Body{\Fi{G}{i}}} = \emptyset\).
By \Nref{obs:stable:relative:properties}{c}, we get \((\Ii{I}{i},\Ii{J}{i})=\WellRM{\Fi{G}{i}}{\Iri{I}{i},\Iri{J}{i}} = \WellRM{\Fi{G}{i}}{\Iri{I}{i},\Iri{I}{i}}\).
Because \(\Fi{P}{i}\) is stratified, we have \({\Body{\Fi{G}{i}}}^- \cap \Head{\Fi{G}{i}} = \emptyset\).
We then use \cref{lem:well-founded:stratified} to obtain \(\Ii{I}{i} = \Ii{J}{i} = \StableR{\Fi{G}{i}}{\Iri{I}{i}}{\Iri{I}{i}}\).
\end{proofpart}

\begin{proofpart}[Hypothesis]
We assume that the theorem holds for any component \(\Fi{P}{j}\) with \(j < i\).
\end{proofpart}

\begin{proofpart}[Step]
Let \(\Fi{P}{i}\) be a stratified component.
For any \(j < i\), component \(\Fi{P}{i}\) either depends on \(\Fi{P}{j}\) or not.
If \(\Fi{P}{i}\) depends on \(\Fi{P}{j}\), then \(\Fi{P}{j}\) is stratified and we get \(\Ii{I}{j} = \Ii{J}{j}\) by the induction hypothesis.
If \(\Fi{P}{i}\) does not depend on \(\Fi{P}{j}\), then \(\Ii{I}{j} \cap \IDPosNeg{\Body{\Fi{G}{i}}} = \Ii{J}{j} \cap \IDPosNeg{\Body{\Fi{G}{i}}} = \emptyset\).
By \Nref{obs:stable:relative:properties}{c}, we get \((\Ii{I}{i},\Ii{J}{i})=\WellRM{\Fi{G}{i}}{\Iri{I}{i},\Iri{J}{i}} = \WellRM{\Fi{G}{i}}{\Iri{I}{i},\Iri{I}{i}}\).
Just as in the base case, by \cref{lem:well-founded:stratified}, we get \(\Ii{I}{i} = \Ii{J}{i} = \StableR{\Fi{G}{i}}{\Iri{I}{i}}{\Iri{I}{i}}\).\qedhere
\end{proofpart}
\end{proof}

\begin{lemma}\label{lem:sequence:refined:external}
Let \(P\) be an aggregate program and \(\Fsqic{P}{\IndexAlt}{i,j}\) be a refined instantiation sequence for \(P\).

Then, for the sequence \(\Fsqic{G}{\IndexAlt}{i,j}\) with \(\Fi{G}{i,j} = \TransA{\Fi{P}{i,j}}\),
we have \(\Ii{E}{i,j} \cap \IDPos{\Body{\Fi{G}{i,j}}} = \emptyset\) for each \((i,j) \in \IndexAlt\)
where \(\Ii{E}{i,j}\) is defined as in \cref{eq:well-founded:sequence:external}.
\end{lemma}

\begin{proof}
The same arguments as in the proof of \cref{lem:sequence:external} can be used but using \Nref{obs:depend}{b} instead.
\end{proof}

\begin{proof}[Proof of \cref{lem:approximate:well-founded}]
Let \(G = \TransA{P}\), \(G' = \TransA{P'}\) with \(P'\) as in \cref{def:approximate}, \((I,J) = \ApproxRM{P}{E}{\Ir{I},\Ir{J}}\), and~\((I',J') = \WellRM{G}{\Ir{I},\Ir{J} \cup \Ir{E}}\).

We first show \(I \subseteq I'\), or equivalently
\begin{align*}
\StableR{G'}{\Ir{I}}{\Ir{J}}       & \subseteq \StableR{G}{\Ir{I}}{\Ir{J} \cup \Ir{E} \cup J'}.\\
\intertext{Because \(G' \subseteq G\), we get}
\StableR{G'}{\Ir{I}}{\Ir{J} \cup \Ir{E} \cup J'} & \subseteq \StableR{G}{\Ir{I}}{\Ir{J} \cup \Ir{E} \cup J'}.\\
\intertext{Because \(\Predicate{\Body{P'}^-} \cap E = \emptyset\), all rules~\(r \in G'\) satisfy \(\Body{r}^- \cap \Ir{E} \neq \emptyset\) and
we obtain}
\StableR{G'}{\Ir{I}}{\Ir{J} \cup J'} & \subseteq \StableR{G'}{\Ir{I}}{\Ir{J} \cup \Ir{E} \cup J'}.\\
\intertext{Because \(\Predicate{\Head{P}} \cap \Predicate{\Body{P}^-} \subseteq E\) and \(\Predicate{\Body{P'}^-} \cap \PredE = \emptyset\), all rules \(r \in G'\) satisfy \(\Body{r}^- \cap J' = \emptyset\) and
we obtain}
\StableR{G'}{\Ir{I}}{\Ir{J}} & = \StableR{G'}{\Ir{I}}{\Ir{J} \cup J'}\\
                       & \subseteq \StableR{G}{\Ir{I}}{\Ir{J} \cup \Ir{E} \cup J'}.
\intertext{To show \(J' \subseteq J\), we use \(I \subseteq I'\) and \cref{prp:stable:relative:monotonicity}:}
\StableR{G}{\Ir{J}}{\Ir{I} \cup I'} & \subseteq \StableR{G}{\Ir{J}}{\Ir{I} \cup I}.\qedhere
\end{align*}
\end{proof}

\begin{proof}[Proof of \cref{thm:approximate:sequence:well-founded}]
We begin by showing \(\ApproxM{\Fsqi{P}{\IndexAlt}{j}} \leq_p \WellM{\TransA{P}}\) and then show \(\ApproxM{\Fsqi{P}{\Index}{i}} \leq_p \ApproxM{\Fsqi{P}{\IndexAlt}{j}}\).
\begin{proofpart}[Property \((\ApproxM{\Fsqi{P}{\IndexAlt}{j}} \leq_p \WellM{\TransA{P}})\)]
Let \(\Ii{\PredE}{j}\), \((\Iri{I}{j},\Iri{J}{j})\), and \((\Ii{I}{j},\Ii{J}{j})\)
be defined as in \cref{def:approximate:sequence:external,def:approximate:sequence:relative,def:approximate:sequence:intermediate}
for the sequence \(\Fsqi{P}{\IndexAlt}{j}\).
Similarly, let \(\Ipi{E}{j}\), \((\Irpi{I}{j},\Irpi{J}{j})\), and \((\Ipi{I}{j},\Ipi{J}{j})\)
be defined as in \cref{eq:well-founded:sequence:intermediate,eq:well-founded:sequence:relative,eq:well-founded:sequence:external}
for the sequence \(\Fsqi{G}{\IndexAlt}{j}\) with \(\Fi{G}{j} = \TransA{\Fi{P}{j}}\).
Furthermore, let \(\Iri{E}{j}\) be the set of all ground atoms over atoms in \(\Ii{\PredE}{j}\).

We first show \(\Ipi{E}{j} \subseteq \Iri{E}{j}\) for each \(j \in \IndexAlt\) by showing that \(\Ipi{E}{j} \subseteq \Iri{E}{j}\).
By \cref{lem:sequence:refined:external}, only negative body literals have to be taken into account:
\begin{align*}
\Ipi{E}{j} & = \IDPosNeg{\Body{\Fi{G}{j}}} \cap \bigcup_{j < k} \Head{\Fi{G}{k}}\\
           & = \IDNeg{\Body{\Fi{G}{j}}} \cap \bigcup_{j < k} \Head{\Fi{G}{k}}.
\intertext{Observe that \(\Predicate{\IDNeg{\Body{\Fi{G}{j}}} \subseteq \Predicate{\IDNeg{\Body{\Fi{P}{j}}}}}\) and \(\Predicate{\Head{\Fi{G}{j}}} \subseteq \Predicate{\Head{\Fi{P}{j}}}\).
Thus, we get}
\Predicate{\Ipi{E}{j}}
& = \Predicate{\IDNeg{\Body{\Fi{G}{j}}} \cap \bigcup_{j < k} \Head{\Fi{G}{k}}}\\
& \subseteq \Predicate{\IDNeg{\Body{\Fi{P}{j}}}} \cap \Predicate{\bigcup_{j < k} \Head{\Fi{P}{k}}}\\
& \subseteq \Predicate{\IDNeg{\Body{\Fi{P}{j}}}} \cap \Predicate{\bigcup_{j \leq k} \Head{\Fi{P}{k}}}\\
& = \Ii{\PredE}{j}.
\end{align*}
It follows that \(\Ipi{E}{j} \subseteq \Iri{E}{j}\).

By \cref{thm:sequence:well-founded},
we have \(\bigsqcup_{j\in\IndexAlt} (\Ipi{I}{j},\Ipi{J}{j}) \leq_p \WellM{\TransA{G}}\).
To show the theorem, we show \((\Ii{I}{j},\Ii{J}{j}) \leq_p (\Ipi{I}{j},\Ipi{J}{j})\).
We omit the full induction proof and focus on the key idea:
Using \cref{lem:approximate:well-founded}, whose precondition holds by construction of \(\Ii{\PredE}{j}\), and \cref{prp:well-founded:relative:monotonicity}, we get
\begin{align*}
\ApproxRM{\Fi{P}{j}}{\Ii{\PredE}{j}}{\Iri{I}{j},\Iri{J}{j}}
 & \leq_p \WellRM{\Fi{G}{j}}{(\Iri{I}{j},\Iri{J}{j}) \sqcup (\emptyset,\Iri{E}{j})}\\
 & \leq_p \WellRM{\Fi{G}{j}}{(\Irpi{I}{j},\Irpi{J}{j}) \sqcup (\emptyset,\Ipi{E}{j})}.
\end{align*}
\end{proofpart}
\begin{proofpart}[Property \((\ApproxM{\Fsqi{P}{\Index}{i}} \leq_p \ApproxM{\Fsqic{P}{\IndexAlt}{i,j}})\)]
We omit a full induction proof for this property because it would be very technical.
Instead, we focus on the key idea why the approximate model of a refined instantiation sequence is at least as precise as the one of an instantiation sequence.

Let \(\Ii{\PredE}{i}\) and \(\Ii{\PredE}{i,j}\) be defined as in~\eqref{def:approximate:sequence:external} for the instantiation and refined instantiation sequence, respectively.
Clearly, we have \(\Ii{\PredE}{i,j} \subseteq \Ii{\PredE}{i}\) for each \((i,j) \in \IndexAlt\).
Observe, that (due to rule dependencies and \nref{obs:stable:relative:properties}{c})
calculating the approximate model of the refined sequence,
using \(\Ii{\PredE}{i}\) instead of \(\Ii{\PredE}{i,j}\) in~\eqref{def:approximate:sequence:intermediate},
would result in the same approximate model as for the instantiation sequence.
With this, the property simply follows from the monotonicity of the stable operator.\qedhere
\end{proofpart}
\end{proof}

\begin{proof}[Proof of \cref{thm:approximate:sequence:simplification}]
Let \(\Ii{\PredE}{i}\), \((\Iri{I}{i},\Iri{J}{i})\), and \((\Ii{I}{i},\Ii{J}{i})\)
be defined as in \cref{def:approximate:sequence:external,def:approximate:sequence:relative,def:approximate:sequence:intermediate}
for the sequence \(\Fsqi{P}{\Index}{i}\).
Similarly, let \(\Ipi{E}{i}\), \((\Irpi{I}{i},\Irpi{J}{i})\), and \((\Ipi{I}{i},\Ipi{J}{i})\)
be defined as in \cref{eq:well-founded:sequence:intermediate,eq:well-founded:sequence:relative,eq:well-founded:sequence:external}
for the sequence \(\Fsqi{G}{\Index}{i}\) with \(\Fi{G}{i} = \TransA{\Fi{P}{i}}\).
Furthermore, we assume w.l.o.g.\ that \(\Index = \{1,\dots,n\}\).

We have already seen in the proof of \cref{thm:approximate:sequence:well-founded}
that the atoms \(\Ipi{E}{i}\) are a subset of the ground atoms over predicates \(\Ii{\PredE}{i}\) and that \((\Ii{I}{i},\Ii{J}{i}) \leq_p (\Ipi{I}{i},\Ipi{J}{i})\).
Observing that ground atoms over predicates \(\Ii{\PredE}{i}\) can only appear negatively in rule bodies,
we obtain that \(\Simp{\Fi{G}{i}}{(\Irpi{I}{i},\Irpi{J}{i}) \cup (\Ipi{I}{i},\Ipi{J}{i}) \cup (\emptyset,\Ipi{E}{i})} = \Simp{\Fi{G}{i}}{(\Irpi{I}{i},\Irpi{J}{i}) \cup (\Ipi{I}{i},\Ipi{J}{i})}\).
By \cref{thm:sequence:simplification,cor:sequence:simplification:superset}, we obtain that \(\bigcup_{i \in \Index}\Simp{\Fi{G}{i}}{(\Iri{I}{i},\Iri{J}{i}) \sqcup (\Ii{I}{i},\Ii{J}{i})}\) and \(\TransA{P}\) have the same well-founded and stable models.
To shorten the notation, we let
\begin{align*}
\Fi{F}{i} & = \Simp{\TransA{\Fi{P}{i}}}{(\Iri{I}{i},\Iri{J}{i}) \sqcup (\Ii{I}{i},\Ii{J}{i})}, &
\Fi{H}{i} & = \Simp{\TransR{\Fi{P}{i}}{\Iri{J}{i} \cup \Ii{J}{i}}}{(\Iri{I}{i},\Iri{J}{i}) \sqcup (\Ii{I}{i},\Ii{J}{i})},\\
F & = \bigcup_{i \in \Index}\Fi{F}{i},\AndT &
H & = \bigcup_{i \in \Index}\Fi{H}{i}.
\end{align*}
With this, it remains to show that programs \(F\) and \(H\) have the same well-founded and stable models.

We let \(J = \bigcup_{i \in \Index}\Ii{J}{i}\).
Furthermore, we let \(\TransA{a}\) be a subformula in \(\Fi{F}{i}\) and \(\TransR{a}{\Iri{J}{i} \cup \Ii{J}{i}}\) be a subformula in \(\Fi{H}{i}\)
where both subformulas originate from the translation of the closed aggregate \(a\).
(We see below that existence of one implies the existence of the other because both formulas are identical in their context.)

Because an aggregate always depends positively on the predicates occurring in its elements,
the intersection between \(\bigcup_{i < k}\Head{\Fi{F}{k}} = \bigcup_{i < k}\Ii{J}{k}\) and the atoms occurring in \(\TransA{a}\) is empty.
Thus the two formulas \(\TransR{a}{\Iri{J}{i} \cup \Ii{J}{i}}\) and \(\TransR{a}{J}\) are identical.
Observe that each stable model of either \(F\) and \(H\) is a subset of \(J\).
By \cref{prp:translation:restricted:properties},
satisfiability of the aggregates formulas as well as their reducts is the same for subsets of \(J\).
Thus, both formulas have the same stable models.
Similarly, the well-founded model of both formulas must be more-precise than \((\emptyset,J)\).
By \cref{prp:translation:restricted:properties},
satisfiability of the aggregate formulas as well as their \FOID-reducts is the same.
Thus, both formulas have the same well-founded model.
\end{proof}

\begin{proof}[Proof of \cref{thm:approximate:sequence:well-founded:stratified}]
Clearly, we have \(\Fi{\PredE}{i}=\emptyset\) if all components are stratified.
With this, the theorem follows from \cref{lem:sequence:well-founded:stratified}.
\end{proof}

We can characterize the result of \cref{fun:ground:rule} as follows.

\newcommand\GRS[2]{\ensuremath{G_{#1,#2}}}

\begin{lemma}\label{prp:ground:rule}
Let \(r\) be a safe normal rule, \((I,J)\) be a finite four-valued interpretation, \(f \in \{ \True, \False\}\), and \(J'\) be a finite two-valued interpretation.

Then, a call to \(\GroundRule{r}{I}{J}{J'}{\iota}{\Body{r}}{f}\) returns the finite set of instances~\(g\) of \(r\) satisfying
\begin{align}
J \models \IDReductP{\Trans{\Body{g}}^\wedge}{I} \text{\ and } (f = \True \text{\ or\ } {\Body{g}}^+ \nsubseteq J').\label{eq:ground:rule:property}
\end{align}
\end{lemma}

\begin{proof}
Observe that the algorithm does not modify \(f\), \(r\), \((I,J)\), and \(J'\).
To shorten the notation below, let \(\GRS{\sigma}{L} = \GroundRule{r}{I}{J}{J'}{\sigma}{L}{f}\).

Calling \(\GRS{\iota}{\Body{r}}\), the algorithm maintains the following invariants in subsequent calls~\(\GRS{\sigma}{L}\):
\begin{align*}
& {(\Body{r} \setminus L)}\sigma^+ \subseteq J,\tag{I1}\label{eq:ground:rule:invariant:positive}\\
& {(\Body{r} \setminus L)}\sigma^- \cap I = \emptyset, \AndT\tag{I2}\label{eq:ground:rule:invariant:negative}\\
& \text{each comparison in \((\Body{r} \setminus L)\sigma\) holds.}\tag{I3}\label{eq:ground:rule:invariant:comparison}
\end{align*}
We only prove the first invariant because the latter two can be shown in a similar way.
We prove by induction.

\begin{proofpart}[Base]
For the call~\(\GRS{\iota}{\Body{r}}\), the invariant holds because the set difference \(\Body{r} \setminus L\) is empty for \(L=\Body{r}\).
\end{proofpart}

\begin{proofpart}[Hypothesis]
We assume the invariant holds for call~\(\GRS{\sigma}{L}\) and show that it is maintained in subsequent calls.
\end{proofpart}

\begin{proofpart}[Step]
Observe that there are only further calls if \(L\) is non-empty.
In \cref{fun:ground:rule:select}, a body literal~\(l\) is selected from~\(L\).
Observe that it is always possible to select such a literal.
In case that there are positive literals in \(L\), we can select any one of them.
In case that there are no positive literals in \(L\),
\(\sigma\) replaces all variables in the positive body of \(r\).
Because \(r\) is safe,
all literals in \(L\sigma\) are ground and we can select any one of them.

In case that \(l\) is a positive literal,
all substitutions \(\sigma'\), obtained by calling \(\Matches{l}{I,J}{\sigma}\) in the following line, ensure
\begin{align*}
l\sigma' & \in J.\\
\intertext{Furthermore, \(\sigma\) is more general than \(\sigma'\).
Thus, we have}
{(\Body{r} \setminus L)\sigma'}^+ & = {(\Body{r} \setminus L)\sigma}^+ \\
                                  & \subseteq J.\\
\intertext{In \cref{fun:ground:rule:recurse}, the algorithm calls \(\GRS{\sigma'}{L'}\) with \(L' = L\setminus \{l\}\).
We obtain}
{(\Body{r} \setminus L')\sigma'}^+ & = {(\Body{r} \setminus L)\sigma'}^+ \cup \{ l\sigma' \} \\
                                   & \subseteq J.
\end{align*}

In case that \(l\) is a comparison or negative literal,
we get \({(\Body{r} \setminus L)}^+ = {(\Body{r} \setminus L \setminus \{l\})}^+\).
Furthermore, the substitution \(\sigma\) is either not changed or is discarded altogether.
Thus, the invariant is maintained in subsequent calls to \KwGroundRule.
\end{proofpart}

We prove by induction over subsets \(L\) of \(\Body{r}\) with corresponding substitution~\(\sigma\) satisfying invariants~\eqref{eq:ground:rule:invariant:positive}--\eqref{eq:ground:rule:invariant:comparison}
that \(G_{L,\sigma}\) is finite and that \(g \in G_{L,\sigma}\) iff \(g\) is a ground instance of \(r\sigma\) that satisfies~\eqref{eq:ground:rule:property}.

\begin{proofpart}[Base]
We show the base case for \(L = \emptyset\).
Using invariant~\eqref{eq:ground:rule:invariant:positive}, we only have to consider substitutions \(\sigma\) with \({\Body{r}}^+ \sigma \subseteq J\).
Because \(r\) is safe and \(\sigma\) replaces all variables in its positive body,
\(\sigma\) also replaces all variables in its head and negative body.
Thus, \(r\sigma\) is ground and the remainder of the algorithm just filters the set \(\{r\sigma\}\)
while the invariants~\eqref{eq:ground:rule:invariant:positive}--\eqref{eq:ground:rule:invariant:comparison} ensure that \(J \models \IDReductP{\Trans{\Body{r\sigma}}^\wedge}{I}\).
The condition in \cref{fun:ground:rule:next} cannot apply because \(L=\emptyset\).
The condition in \cref{fun:ground:rule:seen} discards rules \(r\sigma\) not satisfying \(f=\True\) or \({\Body{r\sigma}}^+ \nsubseteq J'\).
\end{proofpart}

\begin{proofpart}[Hypothesis]
We show that the property holds for \(L \neq \emptyset\)
assuming that it holds for subsets \(L' \subset L\) with corresponding substitutions \(\sigma'\).
\end{proofpart}

\begin{proofpart}[Step]
Because \(L \neq \emptyset\) we only have to consider the case in \cref{fun:ground:rule:next}.

First, the algorithm selects an element \(l \in L\).
We have already seen that it is always possible to select such an element.
Let \(L' = L \setminus \{l\}\).
The algorithm then loops over the set
\begin{align*}
\Sigma & = \Matches{l}{I,J}{\sigma}
\intertext{and, in \crefrange{fun:ground:rule:loop:begin}{fun:ground:rule:loop:end}, computes the union}
\GRS{\sigma}{L} &= \bigcup_{\sigma' \in \Sigma} \GRS{\sigma'}{L'}.
\end{align*}

First, we show that the set \(\GRS{\sigma}{L}\) is finite.
In case \(l\) is not a positive literal, the set \(\Sigma\) has at most one element.
In case \(l\) is a positive literal,
observe that there is a one-to-one correspondence between \(\Sigma\) and the set \(\{ l\sigma' \mid \sigma' \in \Sigma \}\).
We obtain that \(\Sigma\) is finite because \(\{ l\sigma' \mid \sigma' \in \Sigma\} \subseteq J\) and \(J\) is finite.
Furthermore, using the induction hypothesis, each set \(\GRS{\sigma'}{L'}\) in the union \(\GRS{\sigma}{L}\) is finite.
Hence, the set \(\GRS{\sigma}{L}\) returned by the algorithm is finite.

Second, we show \(g \in \GRS{\sigma}{L}\) implies that \(g\) is a ground instance of \(r\sigma\) satisfying~\eqref{eq:ground:rule:property}.
We have that \(g\) is a member of some \(\GRS{\sigma'}{L'}\).
By the induction hypothesis, \(g\) is a ground instance of \(r\sigma'\) satisfying~\eqref{eq:ground:rule:property}.
Observe that \(g\) is also a ground instance of \(r\sigma\) because \(\sigma\) is more general than \(\sigma'\).

Third, we show that each ground instance~\(g\) of \(r\sigma\) satisfying~\eqref{eq:ground:rule:property} is also contained in \(\GRS{\sigma}{L}\).
Because \(g\) is a ground instance of \(r\sigma\),
there is a substitution \(\theta\) more specific than \(\sigma\) such that \(g = r\theta\).
In case that the selected literal \(l \in L\) is a positive literal,
we have \(l\theta \in J\).
Then, there is also a substitution~\(\theta'\) such that \(\theta' \in \Match{l\sigma}{l\theta}\).
Let \(\sigma' = \sigma \circ \theta'\).
By \cref{def:matches}, we have \(\sigma' \in \Sigma\).
It follows that \(g \in \GRS{\sigma}{L}\) because \(g \in \GRS{\sigma'}{L'}\) by the induction hypothesis and \(\GRS{\sigma'}{L'} \subseteq \GRS{\sigma}{L}\).
In the case that \(l\) is not a positive literal, we have \(\sigma \in \Sigma\) and can apply a similar argument.
\end{proofpart}

Hence, we have shown that the proposition holds for \(\GRS{\iota}{\Body{r}}\).
\end{proof}

In terms of the program simplification in \cref{def:simplification},
the first condition in \cref{prp:ground:rule} amounts to checking whether
\(\Head{g} \leftarrow \Trans{\Body{g}}^\wedge \) is in \(\Simp{\Trans{P}}{I,J}\),
which is the simplification of the (ground) \(\FormulasR\)-program \(\Trans{P}\) preserving all stable models between \(I\) and  \(J\).
The two last conditions are meant to avoid duplicates from a previous invocation.
Since \(r\) is a normal rule, translation \TransSym\ is sufficient.

\begin{proof}[Proof of \cref{lem:ground:rule:split}]
The first property directly follows from \Cref{prp:ground:rule} and the definition of \(\GrdSimp{\{r\}}{I,J}\).

It remains to show the second property.
Let \(G\) be the set of all ground instances of \(r\)
and
\begin{align*}
G^{X,Y}_f & = \{ g \in G \mid Y \models \IDReductP{\Trans{\Body{g}}^\wedge}{I}, (f=\True \text{\ or\ } {\Body{g}}^+ \nsubseteq X) \}.
\intertext{By \cref{prp:ground:rule} and the first property, we can reformulate the second property of the proposition as \(G^{\emptyset,J}_{\True} = G^{\emptyset,J'}_{\True} \cup G^{J',J}_{\False}\).
We have}
G^{\emptyset,J}_{\True}  & = \{ g \in G \mid J \models \IDReductP{\Trans{\Body{g}}^\wedge}{I} \}\\
G^{\emptyset,J'}_{\True} & = \{ g \in G \mid J' \models \IDReductP{\Trans{\Body{g}}^\wedge}{I} \}\text{, and}\\
G^{J',J}_{\False}        & = \{ g \in G \mid J \models \IDReductP{\Trans{\Body{g}}^\wedge}{I}, {\Body{g}}^+ \nsubseteq J' \}.
\intertext{Observe that, given \(J' \subseteq J\), we can equivalently write \(G^{\emptyset,J'}_{\True}\) as}
G^{\emptyset,J'}_{\True} & = \{ g \in G \mid J \models \IDReductP{\Trans{\Body{g}}^\wedge}{I}, {\Body{g}}^+ \subseteq J' \}.\\
\intertext{Because \({\Body{g}}^+ \subseteq J'\) and \({\Body{g}}^+ \nsubseteq J'\) cancel each other, we get}
G_{\emptyset,J}  & = G_{\emptyset,J'} \cup G_{J',J}.\qedhere
\end{align*}
\end{proof}

\begin{proof}[Proof of \cref{lem:propagate}]
We first show \cref{lem:propagate:a} and then~\ref{lem:propagate:b}.

\begin{proofprop}[\ref{lem:propagate:a}]
For a rule \(r \in P\), we use \(\Fa{r}\) to refer to the corresponding rule with replaced aggregate occurrences in \(\Fa{P}\).
Similarly, for a ground instance \(g\) of \(r\), we use \(\Fa{g}\) to to refer to the corresponding instance of \(\Fa{r}\).
Observe that \(\Simp{\TransR{P}{J}}{I,J} = \TransR{\GrdSimp{P}{I,J}}{J}\).
We show that \(g \in \GrdSimp{P}{I,J}\) iff \(\Fa{g} \in \GrdSimp{\Fa{P}}{I,J \cup \Ia{J}}\).
In the following, because the rule bodies of \(g\) and \(\Fa{g}\) only differ regarding aggregates and their replacement atoms, we only consider rules with aggregates in their bodies.

\begin{proofpart}[Case \(g \in \GrdSimp{P}{I,J}\)]
Let \(r\) be a rule in \(P\) containing aggregate \(a\),
\(\alpha\) be the replacement atom of form~\eqref{eq:rewrite:aggregate:atom} for \(a\), and
\(\sigma\) be a ground substitution such that \(r\sigma = g\).
We show that for each aggregate~\(a\sigma \in \Body{g}\), we have
\(\AggrEmpty{r}{a}{\Fv{G}}{\sigma} \cup G \neq \emptyset\) and \(J \models \IDReductP{\TransAD{G}{a\sigma}}{I}\)
with \(G = \AggrElem{r}{a}{\Fe{G}}{\sigma}\)
and in turn \(\Ia{J} \models \alpha\sigma\).
Because \(J \models \IDReductP{\TransA{\Body{g}}^\wedge}{I}\),
we get \(J \models \IDReductP{\TransAD{J}{a\sigma}}{I}\).
It remains to show that \(\AggrEmpty{r}{a}{\Fv{G}}{\sigma} \cup G \neq \emptyset\) and \(\TransAD{J}{a\sigma} = \TransAD{G}{a\sigma}\).
Observe that \(\TransAD{J}{a\sigma} = \TransAD{G}{a\sigma}\)
because the set \(G\) obtained from rules in \(\Fe{G}\) contains all instances of elements of \(a\sigma\)
whose conditions are satisfied by \(J\)
while the remaining literals of these rules are contained in the body of \(g\).
Furthermore,
observe that if no aggregate element is satisfied by \(J\),
we get \(\AggrEmpty{r}{a}{\Fv{G}}{\sigma} \neq \emptyset\) because the corresponding ground instance of~\cref{eq:rewrite:empty:rule} is satisfied.
\end{proofpart}

\begin{proofpart}[Case \(\Fa{g} \in \GrdSimp{\Fa{P}}{I,J \cup \Ia{J}}\)]
Let \(\Fa{r}\) be a rule in \(\Fa{P}\) containing replacement atom \(\alpha\) of form~\eqref{eq:rewrite:aggregate:atom} for aggregate~\(a\) and
\(\sigma\) be a ground substitution such that \(\Fa{r}\sigma = \Fa{g}\).
Because \(J \cup \Ia{J} \models \IDReductP{\Trans{\Body{\Fa{g}}}^\wedge}{I}\), we have \(\alpha\sigma \in \Ia{J}\).
Thus, we get that \(J \models \IDReductP{\TransAD{G}{a\sigma}}{I}\) with \(G = \AggrElem{r}{a}{\Fe{G}}{\sigma}\).
We have already seen in the previous case that \(\TransAD{J}{a\sigma} = \TransAD{G}{a\sigma}\).
Thus, \(J \models \IDReductP{\TransAD{J}{a\sigma}}{I}\).
Observing that \(g = r\sigma\) and \(a\sigma \in \Body{g}\), we get \(g \in \GrdSimp{P}{I,J}\).
\end{proofpart}
\end{proofprop}

\begin{proofprop}[\ref{lem:propagate:b}]
This property follows from \cref{lem:propagate:a,prp:translation:restricted:properties,lem:ground:rule:derive}.\qedhere
\end{proofprop}
\end{proof}

\begin{proof}[Proof of \cref{prp:ground:component}]
We prove \cref{prp:ground:component:a,prp:ground:component:b} by showing that the function calculates the stable model by iteratively calling the \(T\) operator until a fixed-point is reached.
\begin{proofprop}[\ref{prp:ground:component:a} and~\ref{prp:ground:component:b}]
At each iteration~\(i\) of the loop starting with 1,
let \(\Iai{J}{i}\) be the value of \(\Propagate{P}{\Fv{G}}{\Fe{G}}{I}{J}\) in \cref{fun:ground:component:propagate},
\(\Fvi{G}{i}\), \(\Fei{G}{i}\), and \(\Fai{G}{i}\) be the values on the right-hand-side of the assignments in \cref{fun:ground:component:empty,fun:ground:component:elements,fun:ground:component:aggregate}, and
\(\Ii{J}{i} = \Head{\Fai{G}{i}}\).
Furthermore, let \(\Ii{J}{0} = \emptyset\).

By \cref{lem:ground:rule:split}, we get
\begin{align*}
\Fvi{G}{i} & = \GrdSimp{\Fv{P}}{\Ir{I},\Ir{J} \cup \Ii{J}{i-1}},\\
\Fei{G}{i} & = \GrdSimp{\Fe{P}}{\Ir{I},\Ir{J} \cup \Ii{J}{i-1}},\\
\Iai{J}{i} & = \Propagate{P}{\Fvi{G}{i}}{\Fei{G}{i}}{I}{J}\CandT\\
\Fai{G}{i} & = \GrdSimp{\Fa{P}}{\Ir{I},\Ir{J} \cup \Iai{J}{i} \cup \Ii{J}{i-1}}.
\intertext{Using \nref{lem:propagate}{b} and
observing the one-to-one correspondence between \(\Fai{G}{i}\) and \(\Simp{\TransA{P}}{\Ir{I},\Ir{J} \cup \Iri{J}{i}}\), we get}
J_i & = \Head{\Fai{G}{i}} \\
    & = \Step{\IDReductP{\TransA{P}}{\Ir{I}}}{\Ir{J} \cup \Ii{J}{i-1}}.
\end{align*}

Observe that, if the loop exits, then the algorithm computes the fixed point of \(\StepRO{\IDReductP{\TransA{P}}{\Ir{I}}}{\Ir{J}}\), i.e.,
\(J = \StableR{\TransA{P}}{\Ir{I}}{\Ir{J}}\).
Furthermore, observe that this fixed point calculation terminates whenever \(\StableR{\TransA{P}}{\Ir{I}}{\Ir{J}}\) is finite.
Finally, we obtain \(\GroundComponent{P}{\Ir{I}}{\Ir{J}} = \Simp{\TransR{P}{\Ir{J} \cup J}}{\Ir{I}, \Ir{J} \cup J}\) using \nref{lem:propagate}{a}.
\end{proofprop}
\begin{proofprop}[\ref{prp:ground:component:c}]
We have seen above that the interpretation~\(J\) is a fixed point of \(\StepRO{\IDReductP{\TransA{P}}{\Ir{I}}}{\Ir{J}}\).
By \nref{lem:propagate}{b} and observing that function \KwAssemble\ only modifies rule bodies,
we get \(\Head{\GroundComponent{P}{\Ir{I}}{\Ir{J}}} = J\).\qedhere
\end{proofprop}
\end{proof}

\begin{proof}[Proof of \cref{thm:ground:program}]
Since the program is finite, its instantiation sequences are finite, too.
We assume w.l.o.g.\ that \(\Index = \{1,\dots,n\}\) for some \(n \geq 0\).
We let \(\Fri{F}{i}\) and \(\Fri{G}{i}\) be the values of variables \(F\) and \(G\) at iteration \(i\) at the beginning of the loop in \crefrange{fun:ground:program:loop:begin}{fun:ground:program:loop:end},
and \(\Fi{F}{i}\) and \(\Fi{G}{i}\) be the results of the calls to \(\KwGroundComponent\) in \cref{fun:ground:program:true,fun:ground:program:possible} at iteration \(i\).

By \cref{prp:ground:component},
we get that \crefrange{fun:ground:program:rewrite}{fun:ground:program:possible} correspond to an application of the approximate model operator
as given in \cref{def:approximate}.
For each iteration \(i\), we get
\begin{align*}
(\Fri{F}{i}, \Fri{G}{i}) & = \bigsqcup_{j < i} (\Fi{F}{i}, \Fi{G}{i}),\\
(\Iri{I}{i}, \Iri{J}{i}) & = (\Head{\Fri{F}{i}}, \Head{\Fri{G}{i}}),\\
(\Ii{I}{i}, \Ii{J}{i})   & = (\Head{\Fi{F}{i}}, \Head{\Fi{G}{i}})\CandT\\
\Fi{G}{i}                & = \Simp{\TransR{\Fi{P}{i}}{\Iri{J}{i} \cup \Ii{J}{i}}}{(\Iri{I}{i},\Iri{J}{i}) \sqcup (\Ii{I}{i},\Ii{J}{i})}
\intertext{whenever \((\Ii{I}{i},\Ii{J}{i})\) is finite.
In case that each \((\Ii{I}{i},\Ii{J}{i})\) is finite, the algorithm returns
in \cref{fun:ground:program:return} the program}
\Fri{G}{n} \cup \Fi{G}{n} & = \bigcup_{i \in \Index} \Fi{G}{i} \\
                          & = \bigcup_{i \in \Index}\Simp{\TransR{\Fi{P}{i}}{\Iri{J}{i} \cup \Ii{J}{i}}}{(\Iri{I}{i},\Iri{J}{i}) \sqcup (\Ii{I}{i},\Ii{J}{i})}.
\end{align*}
Thus, the algorithm terminates iff each call to \KwGroundComponent\ is finite,
which is exactly the case when \(\ApproxM{{(\Fi{P}{i})}_{i\in\Index}}\) is finite.
\end{proof}

\begin{proof}[Proof of \cref{col:main-result}]
This is a direct consequence of \cref{thm:approximate:sequence:simplification,thm:ground:program}.
\end{proof}

\begin{proof}[Proof of \cref{prp:translation:antimonotonicity}]
Let \(\ElemG\) be the set of ground instances of the aggregate elements of \(a\)
and \(\ElemJ \subseteq \ElemG\) be a set such that \(\ElemJ \notjustifies a\).

Due to the antimonotonicity of the aggregate, we get \(\TransAD{a}{\ElemJ} = \bot\).
Thus, the reduct is constant because all consequents in \(\TransA{a}\) as well as \(\IDReductP{\TransA{a}}{I}\) are equal to \(\bot\) and the antecedents in \(\TransA{a}\) are completely evaluated by the reduct.
Hence, the lemma follows by \nref{lem:foid-basic}{b}.\qedhere
\end{proof}

\begin{proof}[Proof of \cref{prp:sum-aggregate-pm}]
We only show \cref{prp:sum-aggregate-pm:a} because the proof of \cref{prp:sum-aggregate-pm:b} is symmetric.

Let \(a = \SumA{\ElemF} \succ b\) and \(a_+ = \SumPA{\ElemF} \succ b'\).
Given an arbitrary two-valued interpretation \(J\), we consider the following two cases:

\begin{proofcase}[\(J \not\models \IDReductP{\TransA{a}}{I}\)]
There is a set \(\ElemJ \subseteq \ElemG\) such that \(\ElemJ \notjustifies a\), \(I \models \Trans{\ElemJ}^\wedge\), and \(J \not\models \TransAD{a}{\ElemJ}^\vee\).
Let \(\widehat{\ElemJ} = \ElemJ \cup \{ e \in \ElemG \mid I \models \Trans{e}, \Weight{\Tuple{e}} < 0 \}\).

Clearly, \(\widehat{\ElemJ} \notjustifies a\) and \(I \models \Trans{\widehat{\ElemJ}}^\wedge\).
Furthermore, \(J \not\models \TransAD{a}{\widehat{\ElemJ}}^\vee\)
because we constructed \(\widehat{\ElemJ}\) so that \(\TransAD{a}{\widehat{\ElemJ}}^\vee\) is stronger than \(\TransAD{a}{\ElemJ}^\vee\)
because more elements with negative weights have to be taken into considerations.

Next, observe that \(\widehat{\ElemJ} \notjustifies a_+\) holds
because we have \(\SumP(\Head{\widehat{\ElemJ}}) = \SumP(\Head{\ElemJ})\) and \(\SumM(\Head{\widehat{\ElemJ}}) = \SumM(T)\),
which corresponds to the value subtracted from the bound of~\(a_+\).
To show that \(J \not\models \TransAD{a_+}{\widehat{\ElemJ}}^\vee\),
we show \(\TransAD{a_+}{\widehat{\ElemJ}}^\vee\) is stronger than \(\TransAD{a}{\widehat{\ElemJ}}^\vee\).
Let \(\ElemC \subseteq \ElemG \setminus \widehat{\ElemJ}\) be a set of elements such that \(\widehat{\ElemJ} \cup \ElemC \justifies a_+\).
Because the justification of \(a_+\) is independent of elements with negative weights,
each clause in \(\TransAD{a_+}{\widehat{\ElemJ}}^\vee\) involving an element with a negative weight is subsumed by another clause without that element.
Thus, we only consider sets~\(\ElemC\) containing elements with positive weights.
Observe that \(\widehat{\ElemJ} \cup \ElemC \justifies a\) holds
because we have \(\Sum(\Head{\widehat{\ElemJ} \cup \ElemC}) = \SumP(\Head{\widehat{\ElemJ} \cup \ElemC}) + \SumM(T)\).
Hence, we get \(J \not\models \TransAD{a_+}{\widehat{\ElemJ}}^\vee\).
\end{proofcase}

\begin{proofcase}[\(J \not\models \IDReductP{\TransA{a_+}}{I}\)]
There is a set \(\ElemJ \subseteq \ElemG\) such that \(\ElemJ \notjustifies a_+\), \(I \models \Trans{\ElemJ}^\wedge\), and \(J \not\models \TransAD{a_+}{\ElemJ}^\vee\).
Let \(\widehat{\ElemJ} = \ElemJ \cup \{ e \in \ElemG \mid I \models \Trans{e}, \Weight{\Tuple{e}} < 0 \}\).

Observe that \(\widehat{\ElemJ} \notjustifies a_+\), \(I \models \Trans{\widehat{\ElemJ}}^\wedge\), and \(J \not\models \TransAD{a_+}{\widehat{\ElemJ}}^\vee\).
As in the previous case, we can show that \(\TransAD{a}{\widehat{\ElemJ}}^\vee\) is stronger than \(\TransAD{a_+}{\widehat{\ElemJ}}^\vee\)
because clauses in~\(\TransAD{a}{\widehat{\ElemJ}}^\vee\) involving elements with negative weights are subsumed.
Hence, we get~\(J \not\models \TransAD{a}{\widehat{\ElemJ}}^\vee\).
\qedhere
\end{proofcase}
\end{proof}

\begin{proof}[Proof of \cref{prp:aggregate-generic-implies}]
Let~\(\ElemG\) be the set of ground instances of~\(E\),
\(a_{\prec} = \Aggregate{f}{\ElemF}{\prec}{b}\) for aggregate relation~\({\prec}\),
and \(J\) be a two-valued interpretation.

\begin{proofprop}[\ref{prp:aggregate-generic-implies:a}]
We show that
\(J \models \IDReductP{\TransA{a_{<}}}{I} \vee \IDReductP{\TransA{a_{>}}}{I}\) implies
\(J \models \IDReductP{\TransA{a_{\neq}}}{I}\).

\begin{proofcase}[\(J \models \IDReductP{\TransA{a_{<}}}{I}\)]
Observe that \(\TransA{a_{\neq}}\) is conjunction of implications of form
\({\Trans{D}}^\wedge \rightarrow \TransAD{a_{\neq}}{\ElemJ}^\vee\) with \(\ElemJ \subseteq \ElemG\) and \(\ElemJ \notjustifies a_{\neq}\).
Furthermore, note that \(\ElemJ \notjustifies a_{\neq}\) implies \(\ElemJ \notjustifies a_{<}\).
Thus, \(\TransA{a_{<}}\) contains the implication \({\Trans{D}}^\wedge \rightarrow \TransAD{a_{<}}{\ElemJ}^\vee\).
Because \(J \models \IDReductP{\TransA{a_{<}}}{I}\), we get \(I \not\models {\Trans{D}}^\wedge\) or \(J \models \TransAD{a_{<}}{\ElemJ}^\vee\).
Hence, the property holds in this case because \(J \models \TransAD{a_{<}}{\ElemJ}^\vee\) implies \(J \models \TransAD{a_{\neq}}{\ElemJ}^\vee\).
\end{proofcase}

\begin{proofcase}[\(J \models \IDReductP{\TransA{a_{>}}}{I}\)]
The property can be shown analogously for this case.
\end{proofcase}
\end{proofprop}

\begin{proofprop}[\ref{prp:aggregate-generic-implies:a}]
This property can be shown in a similar way as the previous one.
We show by contraposition that
\(J \models \IDReductP{\TransA{a_{=}}}{I}\)
implies
\(J \models \IDReductP{\TransA{a_{\leq}}}{I} \wedge \IDReductP{\TransA{a_{\geq}}}{I}\).

\begin{proofcase}[\(J \not\models \IDReductP{\TransA{a_{\leq}}}{I}\)]
Observe that \(\TransA{a_{\leq}}\) is conjunction of implications of form
\({\Trans{D}}^\wedge \rightarrow \TransAD{a_{\leq}}{\ElemJ}^\vee\) with \(\ElemJ \subseteq \ElemG\) and \(\ElemJ \notjustifies a_{\leq}\).
Furthermore, note that \(\ElemJ \notjustifies a_{\leq}\) implies \(\ElemJ \notjustifies a_{=}\).
Thus, \(\TransA{a_{=}}\) contains the implication \({\Trans{D}}^\wedge \rightarrow \TransAD{a_{=}}{\ElemJ}^\vee\).
Because \(J \not\models \IDReductP{\TransA{a_{\leq}}}{I}\), we get \(I \models {\Trans{D}}^\wedge\) and \(J \not\models \TransAD{a_{<}}{\ElemJ}^\vee\)
for some \(\ElemJ \subseteq \ElemG\) with \(\ElemJ \notjustifies a_{\leq}\).
Hence, the property holds in this case because \(J \not\models \TransAD{a_{\leq}}{\ElemJ}^\vee\) implies \(J \not\models \TransAD{a_{=}}{\ElemJ}^\vee\).
\end{proofcase}
\begin{proofcase}[\(J \not\models \IDReductP{\TransA{a_{\geq}}}{I}\)]
The property can be shown analogously for this case. \qedhere
\end{proofcase}
\end{proofprop}

\end{proof}

\begin{proof}[Proof of \cref{prp:sum-aggregate-implies}]
We only consider the case that \(f\) is the \(\Sum\) function because the other ones are special cases of this function.
Furthermore, we only consider the only if directions because we have already established the other directions in~\cref{prp:aggregate-generic-implies}.

Let~\(\ElemG\) be the set of ground instances of~\(E\),
\(T_I = \Tuple{\{ g \in \ElemG \mid I \models \Cond{g} \}}\), and
\(T_J = \Tuple{\{ g \in \ElemG \mid J \models \Cond{g} \}}\).

\begin{proofprop}[\ref{prp:sum-aggregate-implies:a}]
Because \(I \subseteq J\), we get \(\SumP(T_I) \leq \SumP(T_J)\) and \(\SumM(T_J) \leq \SumM(T_I)\).
We prove by contraposition.
\begin{proofcase}[\(J \not\models \IDReductP{\TransA{a_{<}}}{I}\) and \(J \not\models \IDReductP{\TransA{a_{>}}}{I}\)]
We use \cref{prp:translation:monotonicity,prp:sum-aggregate-pm} to get the following two inequalities:
\begin{align*}
\SumM(T_J) & \geq b - \SumP(T_I)\BecauseT{\(J \not\models \IDReductP{\TransA{a_{<}}}{I}\) and}\\
\SumP(T_J) & \leq b - \SumM(T_I)\BecauseT{\(J \not\models \IDReductP{\TransA{a_{>}}}{I}\).}
\intertext{Using \(\SumM(T_J) \leq \SumM(T_I)\), we can rearrange as}
b - \SumP(T_I) & \leq \SumM(T_J)\\
& \leq{} \SumM(T_I)\\
& \leq b - \SumP(T_J).
\intertext{Using \(\SumP(T_I) \leq \SumP(T_J)\), we obtain}
\SumP(T_I) & = \SumP(T_J).
\intertext{Using \(\SumP(T_I) = \SumP(T_J)\), we get}
b - \SumP(T_I) & \leq \SumM(T_J)\\
& \leq \SumM(T_I)\\
& \leq b - \SumP(T_I)
\intertext{and, thus, obtain}
\SumM(T_I) & = \SumM(T_J)\AndT\\
b & = \Sum(T_I)\\
& = \Sum(T_J).
\end{align*}
Observe that this gives rise to an implication in \(\IDReductP{\TransA{a_{\neq}}}{I}\) that is not satisfied by \(J\).
Hence, we get \(J \not\models \IDReductP{\TransA{a_{\neq}}}{I}\).
\end{proofcase}
\end{proofprop}

\begin{proofprop}[\ref{prp:sum-aggregate-implies:b}]
Because \(J \subseteq I\), we get \(\SumP(T_J) \leq \SumP(T_I)\) and \(\SumM(T_I) \leq \SumM(T_J)\).
\begin{proofcase}[\(J \models \IDReductP{\TransA{a_{\leq}}}{I}\) and \(J \models \IDReductP{\TransA{a_{\geq}}}{I}\)]
Using \cref{prp:translation:monotonicity,prp:sum-aggregate-pm}, we get
\begin{align*}
\SumP(T_J) & \geq b - \SumM(T_I)\BecauseT{\(J \models \IDReductP{\TransA{a_{\geq}}}{I}\) and}\\
\SumM(T_J) & \leq b - \SumP(T_I)\BecauseT{\(J \models \IDReductP{\TransA{a_{\leq}}}{I}\).}
\intertext{Observe that we can proceed as in the proof of the previous property because the relation symbols are just flipped.
We obtain}
\SumM(T_I) & = \SumM(T_J)\AndT\\
b & = \Sum(T_I)\\
& = \Sum(T_J).
\end{align*}
We get \(J \models \IDReductP{\TransA{a_{=}}}{I}\)
because for any subset of tuples in \(T_I\) that do not satisfy the aggregate,
we have additional tuples in \(T_J\) that satisfy the aggregate.
\qedhere
\end{proofcase}
\end{proofprop}
\end{proof}

\begin{proof}[Proof of \cref{prp:propagate-generic-iff}]
Let \(a_{\prec} = \Aggregate{f}{\ElemF}{\prec}{b}\) for \({\prec} \in \{ {=}, {\neq} \}\).
\begin{proofprop}[\ref{prp:propagate-generic-iff:a}]
We prove by contraposition that \(J \models \IDReductP{\TransA{a_{\neq}}}{I}\) implies that
there is no set \(X \subseteq T_I\) such that \(f(X \cup T_J) = b\).
\begin{proofcase}[there is a set \(X \subseteq T_I\) such that \(f(X \cup T_J) = b\)]
Let \(\ElemJ = \{ e \in \ElemG \mid I \models \Cond{e}, \Tuple{e} \in X \cup T_J \}\).
Because \(T_J \subseteq T_I\) \(\ElemJ \notjustifies a_{\neq}\).
Furthermore, we have \(I \models \Trans{\ElemJ}^\wedge\).
Observe that \(D\) contains all elements with conditions satisfied by \(J\).
Hence, we get \(J \not\models \TransAD{a_{\neq}}{D}^\vee\) and, in turn, \(J \not\models \IDReductP{\TransA{a_{\neq}}}{I}\).
\end{proofcase}
We prove the remaining direction, again, by contraposition.
\begin{proofcase}[\(J \not\models \IDReductP{\TransA{a_{\neq}}}{I}\)]
There is a set \(D \subseteq \ElemG\) such that \(I \models \Trans{D}^\wedge\) and \(J \not\models \TransAD{a_{\neq}}{D}^\vee\).
Let \(X = \Head{D}\).
Because \(J \not\models \TransAD{a_{\neq}}{D}^\vee\), we get \(f(X \cup T_J) = b\).
\end{proofcase}
\end{proofprop}
\begin{proofprop}[\ref{prp:propagate-generic-iff:b}]
This property can be shown in a similar way as the previous one.\qedhere
\end{proofprop}
\end{proof}
 
\end{document}